\documentclass[twoside]{article}
\usepackage{PRIMEarxiv}

\usepackage[utf8]{inputenc} 
\usepackage[T1]{fontenc}    
\usepackage{hyperref}       
\usepackage{url}            
\usepackage{booktabs}       
\usepackage{amsfonts}       
\usepackage{nicefrac}       
\usepackage{microtype}      
\usepackage{lipsum}
\usepackage{fancyhdr}       
\usepackage{graphicx}       
\graphicspath{{media/}}     
\usepackage[inkscapeexe=/Applications/Inkscape.app/Contents/MacOS/inkscape,inkscapearea=page]{svg}

\usepackage{authblk}
\usepackage{amsthm}
\theoremstyle{definition}
\newtheorem{definition}{Definition}[section]

\newtheorem{theorem}{Theorem}[section]
\newtheorem{corollary}{Corollary}[theorem]

\usepackage{wrapfig}
\usepackage{tabularx}
\usepackage{booktabs}
\usepackage{multirow}
\usepackage{svg}
\usepackage{subcaption}
\usepackage{pgfplots}
\pgfplotsset{compat=1.17}
\usepackage{MnSymbol}

\usepackage{xcolor}

\definecolor{class1_br}{rgb}{1, 0.349, 0.392}
\definecolor{class2_br}{rgb}{0.086, 0.858, 0.576}
\definecolor{class3_br}{rgb}{0.937, 0.917, 0.352}
\definecolor{class4_br}{rgb}{0, 0.486, 0.745}

\usepackage{listings}       
\definecolor{codegreen}{rgb}{0,0.6,0}
\definecolor{codegray}{rgb}{0.5,0.5,0.5}
\definecolor{codepurple}{rgb}{0.58,0,0.82}
\definecolor{backcolour}{RGB}{255, 247, 237}
\lstdefinestyle{mystyle}{
    backgroundcolor=\color{backcolour},   
    commentstyle=\color{codegreen},
    keywordstyle=\color{blue},
    numberstyle=\tiny\color{codegray},
    stringstyle=\color{codepurple},
    basicstyle=\footnotesize\ttfamily,
    breakatwhitespace=false,         
    breaklines=true,                 
    captionpos=b,                    
    keepspaces=true,                 
    numbers=left,
    frame=bt,
    numbersep=5pt,                  
    showspaces=false,                
    showstringspaces=false,
    showtabs=false,                  
    tabsize=2
}
 
\usepackage{bm}
\newcommand{\matr}[1]{\bm{#1}}      
\let\conjugate\overline

\usepackage{xcolor}

\usepackage{tikz}
\usepackage{hyperref}
\hypersetup{
	colorlinks=true,
	linkcolor=black,     
	urlcolor=blue,
	citecolor=black,
}

\definecolor{lime}{HTML}{A6CE39}
\DeclareRobustCommand{\orcidicon}{%
	\begin{tikzpicture}
	\draw[lime, fill=lime] (0,0) 
	circle [radius=0.16] 
	node[white] {{\fontfamily{qag}\selectfont \tiny ID}};
	\draw[white, fill=white] (-0.0625,0.095) 
	circle [radius=0.007];
	\end{tikzpicture}
	\hspace{-2mm}
}
\foreach \x in {A, ..., Z}{
	\expandafter\xdef\csname orcid\x\endcsname{\noexpand\href{https://orcid.org/\csname orcidauthor\x\endcsname}{\noexpand\orcidicon}}
}
\newcommand{\orcid}[1]{\href{https://orcid.org/#1}{\textcolor[HTML]{A6CE39}{\orcidicon}}}

\usepackage[acronym]{glossaries}
\newacronym{psk}{PSK}{Phase-Shift Keying}
\newacronym{qam}{QAM}{Quadrature Amplitude Modulation}
\newacronym{autodiff}{autodiff}{automatic differentiation}
\newacronym{api}{API}{Application Programming Interface}
\newacronym{ux}{UX}{User eXperience}
\newacronym{cn}{CN}{Complex Normal}
\newacronym{oa}{OA}{Overall Accuracy}
\newacronym{aa}{AA}{Average Accuracy}
\newacronym{bn}{BN}{Batch Normalization}
\newacronym{mse}{MSE}{Mean Squared Error}
\newacronym{em}{EM}{ElectroMagnetic}
\newacronym{ai}{AI}{Artificial Intelligence}
\newacronym{ml}{ML}{Machine Learning}
\newacronym{sgd}{SGD}{Stochastic Gradient Descent}
\newacronym{rmsprop}{RMSprop}{Root Mean Square Propagation}
\newacronym{sar}{SAR}{Synthetic Aperture Radar}
\newacronym{polsar}{PolSAR}{Polarimetric Synthetic Aperture Radar}
\newacronym{polinsar}{PolInSAR}{Polarimetric and Interferometric Synthetic Aperture Radar}
\newacronym{cvnn}{CVNN}{Complex-Valued Neural Network}
\newacronym{rvnn}{RVNN}{Real-Valued Neural Network}

\newacronym{cv-mlp}{CV-MLP}{Complex-Valued MultiLayer Perceptron}
\newacronym{rv-mlp}{RV-MLP}{Real-Valued MultiLayer Perceptron}
\newacronym{mlp}{MLP}{MultiLayer Perceptron}

\newacronym{cnn}{CNN}{Convolutional Neural Network}
\newacronym{cv-cnn}{CV-CNN}{Complex-Valued Convolutional Neural Network}
\newacronym{rv-cnn}{RV-CNN}{Real-Valued Convolutional Neural Network}

\newacronym{rv-fcnn}{RV-FCNN}{Real-Valued Fully Convolutional Neural Network}
\newacronym{cv-fcnn}{CV-FCNN}{Complex-Valued Fully Convolutional Neural Network}
\newacronym{fcnn}{FCNN}{Fully Convolutional Neural Network}

\newacronym{gan}{GAN}{Generative Adversarial Network}

\newacronym{unet}{U-NET}{U-Network}
\newacronym{cv-unet}{CV-UNET}{Complex-Valued U-Network}

\newacronym{cv-cae}{CV-CAE}{Complex-Valued Convolutional AutoEncoder}

\newacronym{relu}{ReLU}{Rectified Linear Unit}
\newacronym{tanh}{tanh}{hyperbolic tangent}
\newacronym{1hl}{1HL}{one hidden layer}
\newacronym{2hl}{2HL}{two hidden layers}

\newcommand{\loss}{\mathcal{L}}
\newcommand{\activation}{\sigma}

\newcommand{\hiddensize}[1]{N_{#1}}

\newcommand{\beforeact}{V}
\newcommand{\weigth}{\omega}
\newcommand{\layerout}{X}

\renewcommand{\Im}{\mathrm{Im}}
\renewcommand{\Re}{\mathrm{Re}}
\newcommand{\im}{i}

\title{Theory and implementation of Complex-Valued Neural Networks
\thanks{\textit{\underline{Citation}}: 
\textbf{Authors. Theory and implementation of Complex-Valued Neural Networks. Pages.... DOI:000000/11111.}} 
}

\author{
    J. A. Barrachina$^{\star \dagger}$\orcid{0000-0002-2139-514X},
    C. Ren$^{\star}$\orcid{0000-0001-8438-4539},
    G. Vieillard$^{\dagger}$,
    C. Morisseau$^{\dagger}$,
    and J.-P. Ovarlez $^{\star \dagger}$\orcid{0000-0001-8056-4196}
}

\affil{$^{\dagger}$ DEMR, ONERA, Universit\'e Paris-Saclay, France}
\affil{$^{\star}$ SONDRA, CentraleSup\'elec,  Universit\'e Paris-Saclay,  91192, Gif-sur-Yvette, France}

\begin{document}
\maketitle

\begin{abstract}
This work explains in detail the theory behind \acrfull{cvnn}, including Wirtinger calculus, complex backpropagation and basic modules such as complex layers, complex activation functions or complex weight initialization. We also show the impact of not adapting the weight initialization correctly to the complex domain.
This work presents a strong focus on the implementation of such modules on Python using \lstinline{cvnn} toolbox.
We also perform simulations on real-valued data, casting to the complex domain by means of the Hilbert Transform and verify the potential interest of \acrshort{cvnn} even for non-complex data.
\end{abstract}

\keywords{\acrshort{cvnn} \and Deep Learning \and Complex values \and Machine Learning}

\section{Introduction}


\begin{figure}[ht]
    \includegraphics[width=\textwidth]{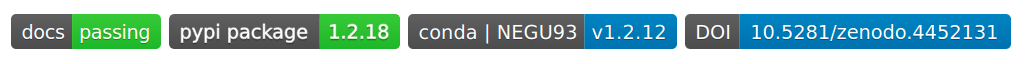}
    \label{fig:badges}
\end{figure}


\begin{figure}[!b]
    \centering
    \includegraphics[height=2cm]{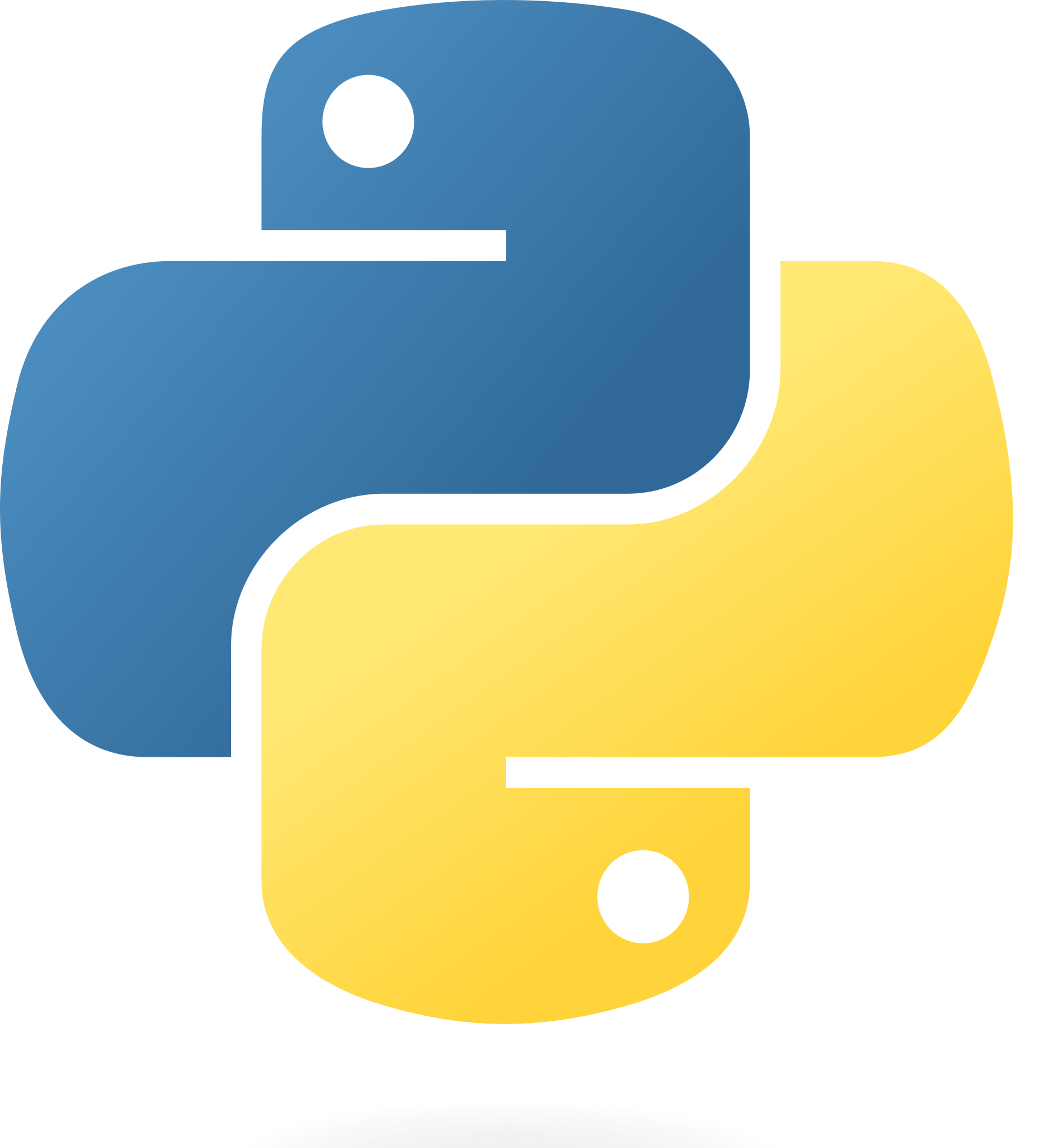}
    \includegraphics[height=2cm]{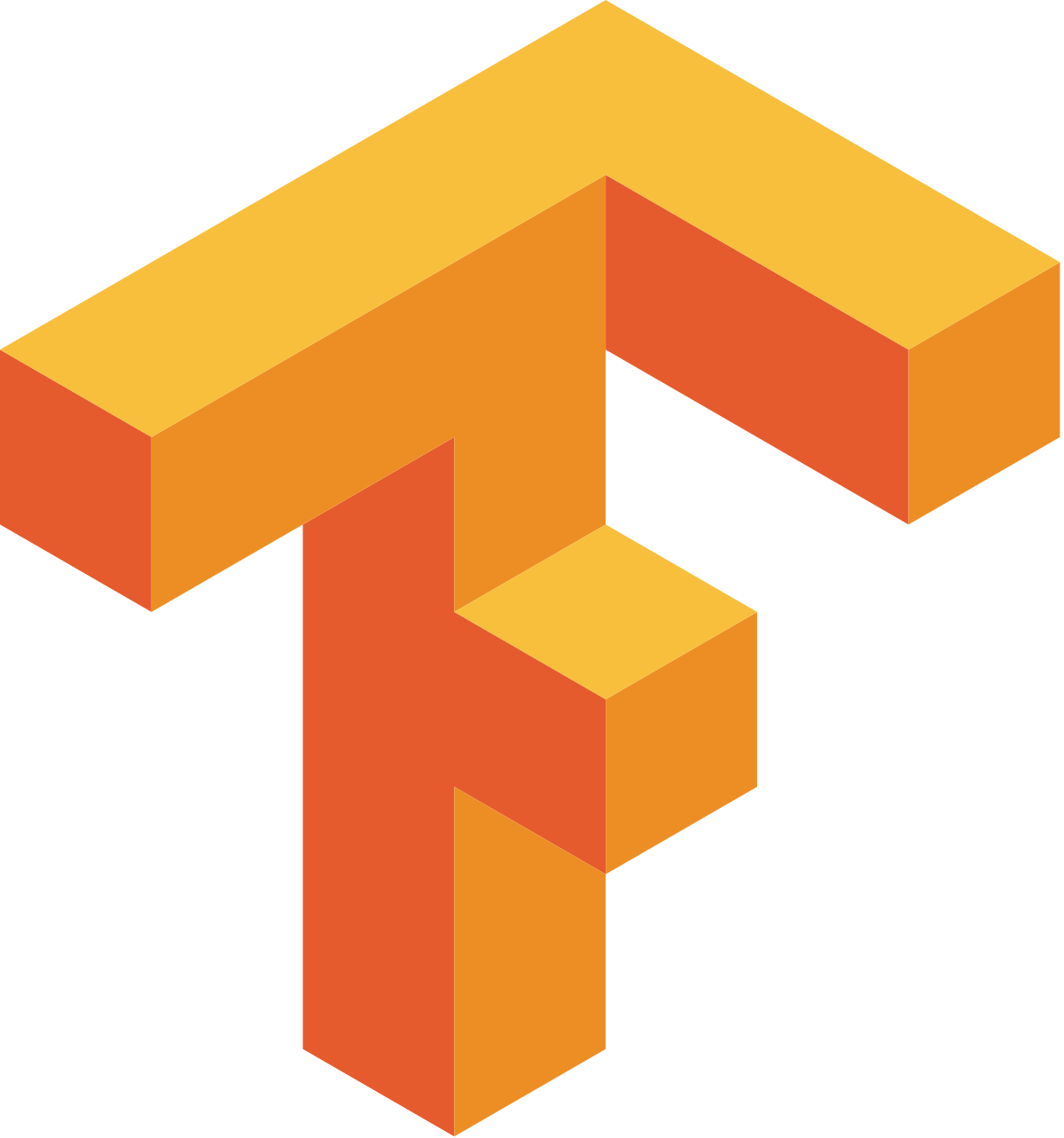}
    \includegraphics[height=2cm]{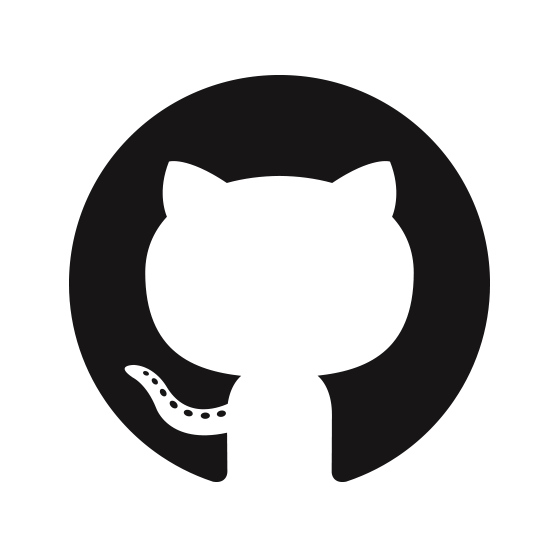}
    \includegraphics[height=2cm]{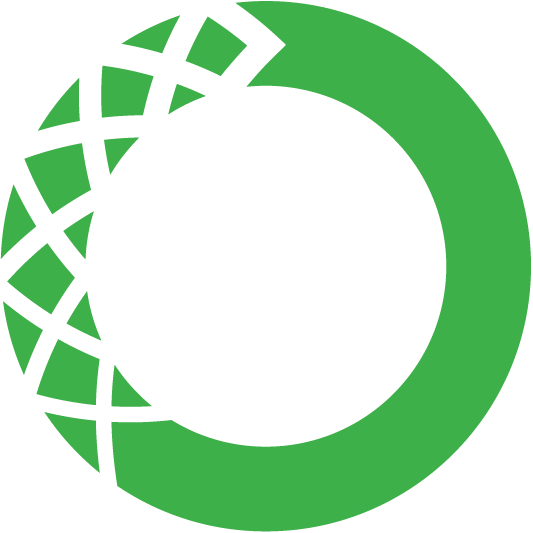}
    \includegraphics[height=2cm]{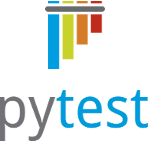}
    \includegraphics[height=2cm]{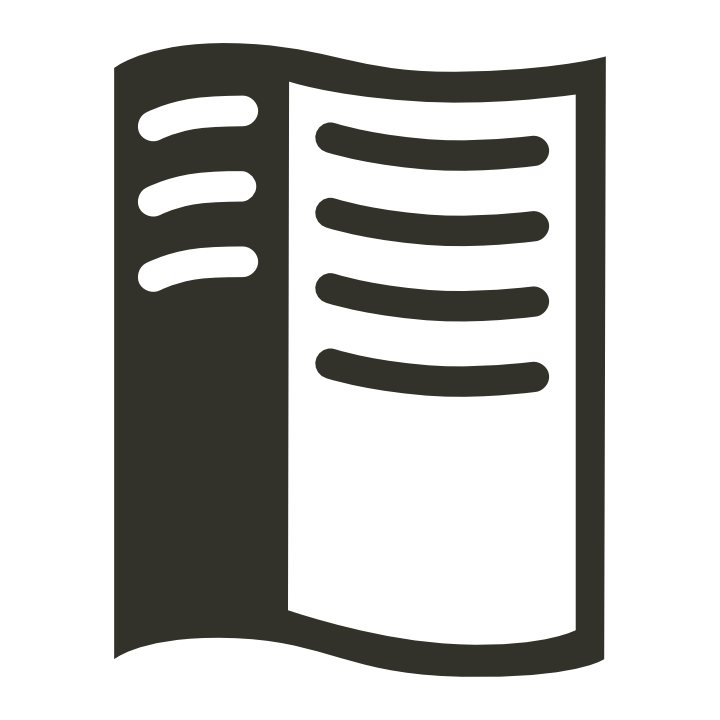}
\end{figure}

Although \acrshort{cvnn} has been investigated for particular structures of complex-valued data \cite{hirose2012generalization, hansch2009classification, hirose2012complex, hirose2009complex}, the difficulties in implementing \acrshort{cvnn} models in practice have slowed down the field from growing further \cite{monning2018evaluation}. 
Indeed, the two most popular 
\href{https://www.python.org/}{Python} 
libraries for developing deep neural networks, which are 
\href{https://pytorch.org/}{Pytorch}, from Meta (formerly named Facebook) 
and 
\href{https://www.tensorflow.org/}{Tensorflow} from Google
 do not fully support the creation of complex-valued models.

Even though \textit{Tensorflow} does not fully support the implementation of a \acrshort{cvnn}, it has one significant benefit: 
It enables the use of complex-valued data types for the \acrfull{autodiff} algorithm \cite{hoffmann2016hitchhiker} to calculate the complex gradients as defined in Appendix \ref{chap:autodif}. 
Since July 2020, \textit{PyTorch} also added this functionality as BETA with the version 1.6 release. later on, on June 2022, after the release of version 1.12, \textit{PyTorch} extended its complex functionality by adding complex convolutions (also as BETA). Although this indicates a clear intention to develop towards \acrshort{cvnn} support, there is still a lot of development to be done.


Libraries to develop \acrshort{cvnn}s do exist, the most important one of them being probably the \href{https://github.com/ChihebTrabelsi/deep_complex_networks}{code} published in \cite{trabelsi2017deep}. However, we have decided not to use this library since the latter uses \textit{Theano} as a back-end, which is no longer \href{https://groups.google.com/g/theano-users/c/7Poq8BZutbY/m/rNCIfvAEAwAJ}{maintained}. Another \href{https://github.com/JesperDramsch/keras-complex}{code} was published on \href{https://github.com/}{GitHub} that uses \textit{Tensorflow} and \textit{Keras} to implement \acrshort{cvnn} \cite{dramsch2019complex}. However, as \textit{Keras} does not support complex-valued numbers, the published code simulates complex operations using real-valued datatypes. Therefore, the user has to transform its complex-valued data into a real-valued equivalent before using this library. The same happened with \href{https://github.com/wavefrontshaping/complexPyTorch}{ComplexPyTorch} \cite{complexpytorch} until it was updated in January 2021 to support complex tensors.

\begin{figure}
\begin{subfigure}[b]{0.45\textwidth}
    \centering
    \includegraphics[width=\textwidth]{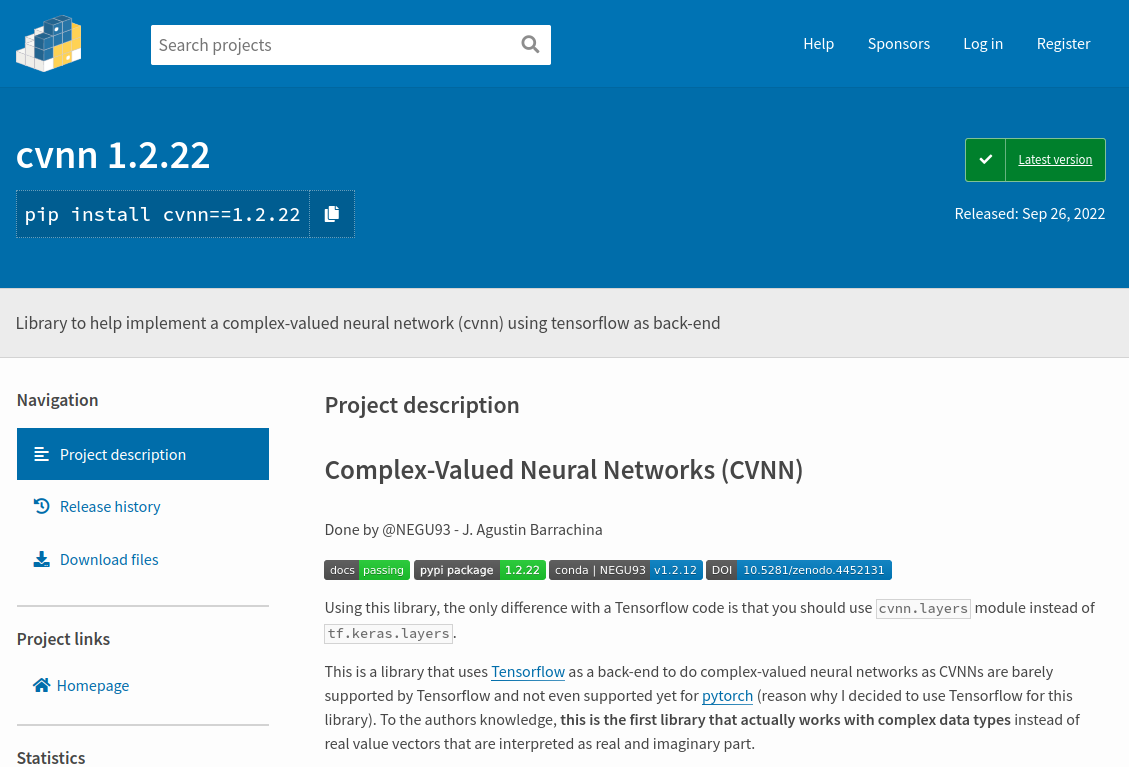}
    \caption{PIP cvnn presentation page.}
    \label{fig:pip}
\end{subfigure}
\begin{subfigure}[b]{0.54\textwidth}
    \centering
    \includegraphics[width=\textwidth]{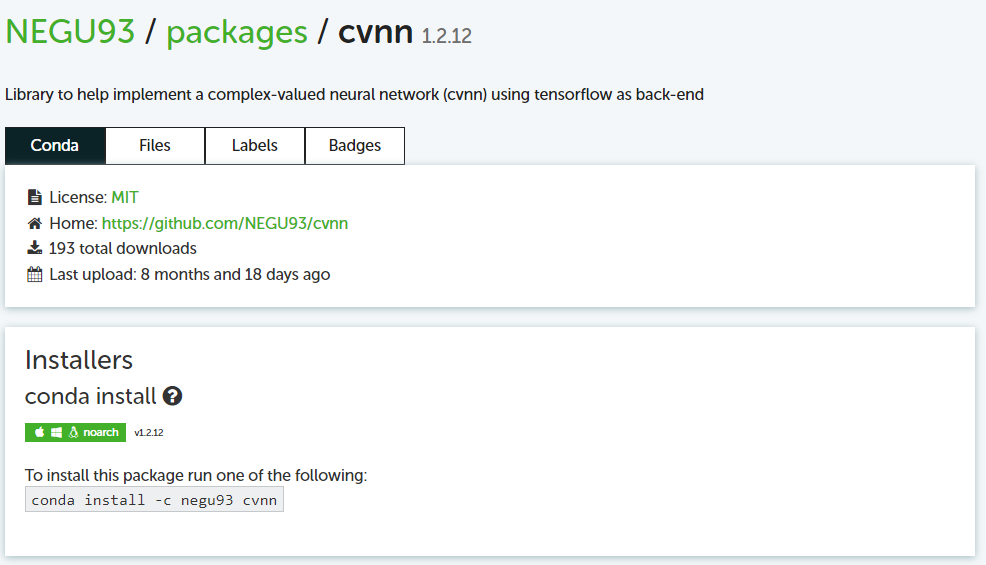}
    \caption{Anaconda cvnn presentation page.}
    \label{fig:conda}
\end{subfigure}
\end{figure}


\begin{wrapfigure}{r}{0.35\textwidth}
    \vspace{-0.5cm}
    \centering
    \includegraphics[width=0.35\textwidth]{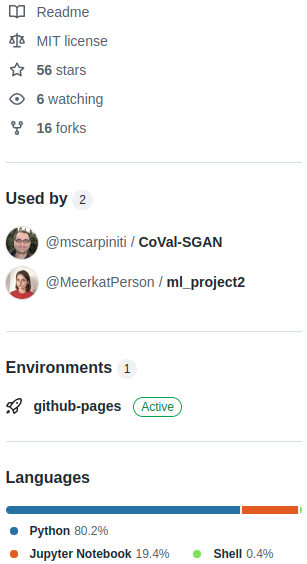}
\end{wrapfigure}
During this thesis, we created a \textit{Python} tool to deal with the implementation of \acrshort{cvnn} models using \textit{Tensorflow} as back-end. Note that the development of this library started in 2019, whereas \textit{PyTorch} support for complex numbers started in mid-2020, which is the reason why the decision to use \textit{Tensorflow} instead of \textit{Pytorch} was made. To the author's knowledge, this was the first library that natively supported complex-number data types.
The library is called CVNN and was published \cite{j_agustin_barrachina_2021_4452131} using CERN's \href{https://zenodo.org/record/4452131#.YgZ1E4yZNhE}{Zenodo} platform which received already 18 downloads.
It can be installed using both \href{https://pypi.org/project/cvnn/}{Python Index Package (PIP)} (Figure \ref{fig:pip}) and \href{https://anaconda.org/negu93/cvnn}{Anaconda}. The latter already has 193 downloads as of the $8^{\text{th}}$ October 2022, as shown in Figure \ref{fig:conda}, from which none of those downloads were ourselves.


\begin{figure}
    \centering
    \includegraphics[width=\textwidth]{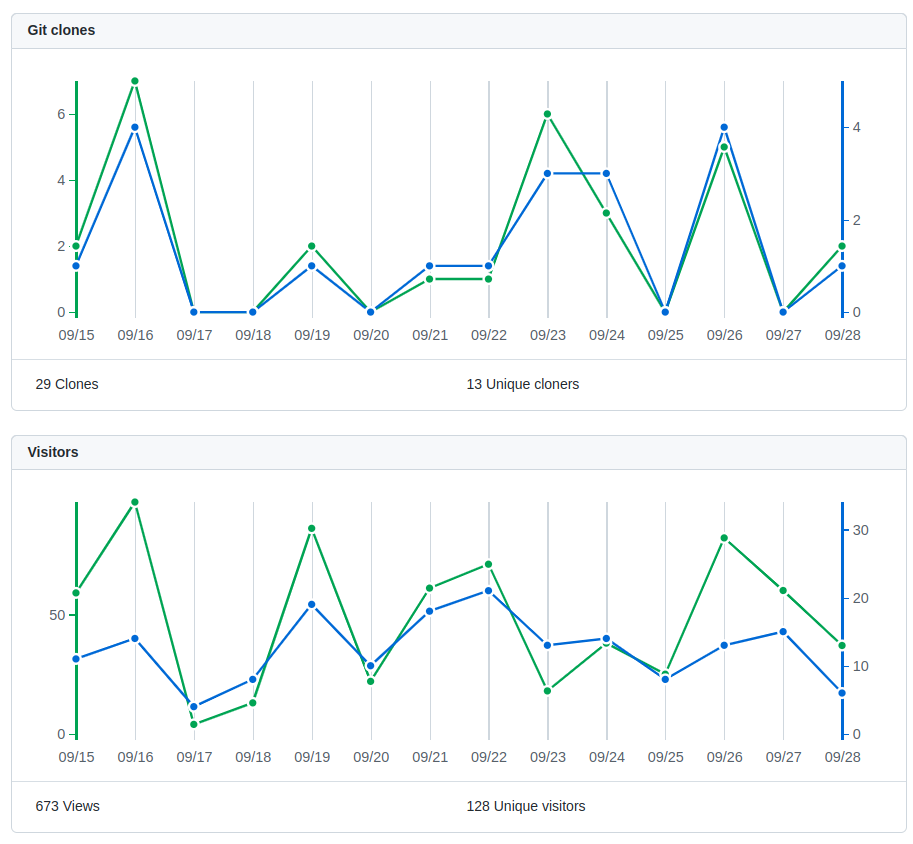}
    \caption{GitHub cvnn toolbox traffic.}
    \label{fig:gh-traffic}
\end{figure}

The library was open-sourced for the community on \href{https://github.com/NEGU93/cvnn/}{GitHub} \cite{negu2019cvnn} and has received a very positive reception from the community. 
As can be seen from Figure \ref{fig:gh-traffic}, the GitHub repository received an average of 2 clones and almost 50 visits per day in the last two weeks.
With 62 stars at the beginning of October 2022. It also has a total of 16 forks and one pull request for a new feature that has already been reviewed and accepted. Six users are actively watching every update on the code as they have activated the notifications. Finally, two users have codes in GitHub importing the library.
Thirty issues have been reported, and it was also subject to 31 private emails.
All these metrics are evidence of the impact and interest of the community in the published code.

The library was documented using \textit{reStructuredText} and uploaded for worldwide availability. The link for the full documentation (displayed in Figure \ref{fig:docs}) can be found in the following link:
\href{https://complex-valued-neural-networks.readthedocs.io/en/latest/index.html}{\path{complex-valued-neural-networks.rtfd.io}}. The documentation received a daily view which varied from a minimum of 18 views on one day to a maximum of 173 views as shown in Figure \ref{fig:docs-views}

\begin{figure}
    \centering
    \begin{subfigure}[b]{0.9\textwidth}
        \centering
        \includegraphics[width=\textwidth]{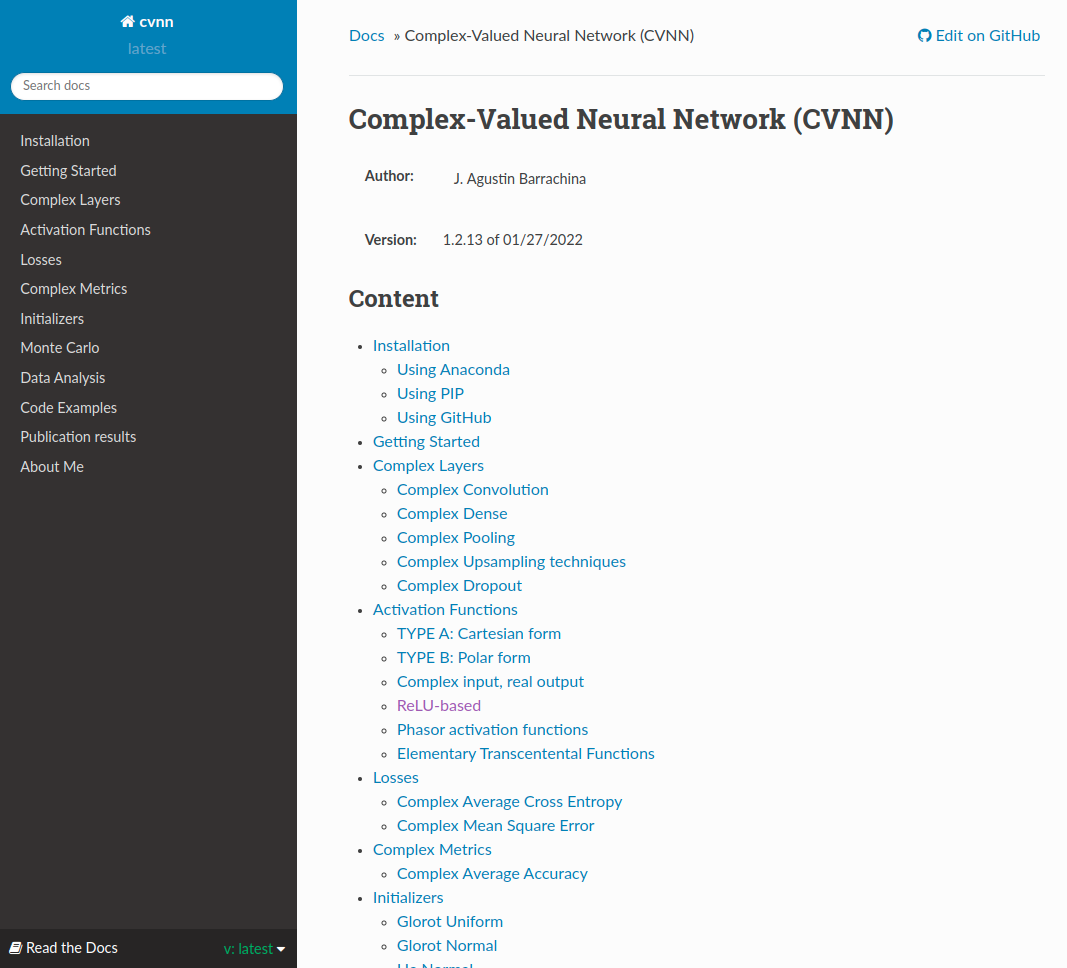}
        \caption{\textit{Read the Docs} documentation presentation page.}
        \label{fig:docs}
    \end{subfigure}
    \begin{subfigure}[b]{0.9\textwidth}
        \includegraphics[width=\textwidth]{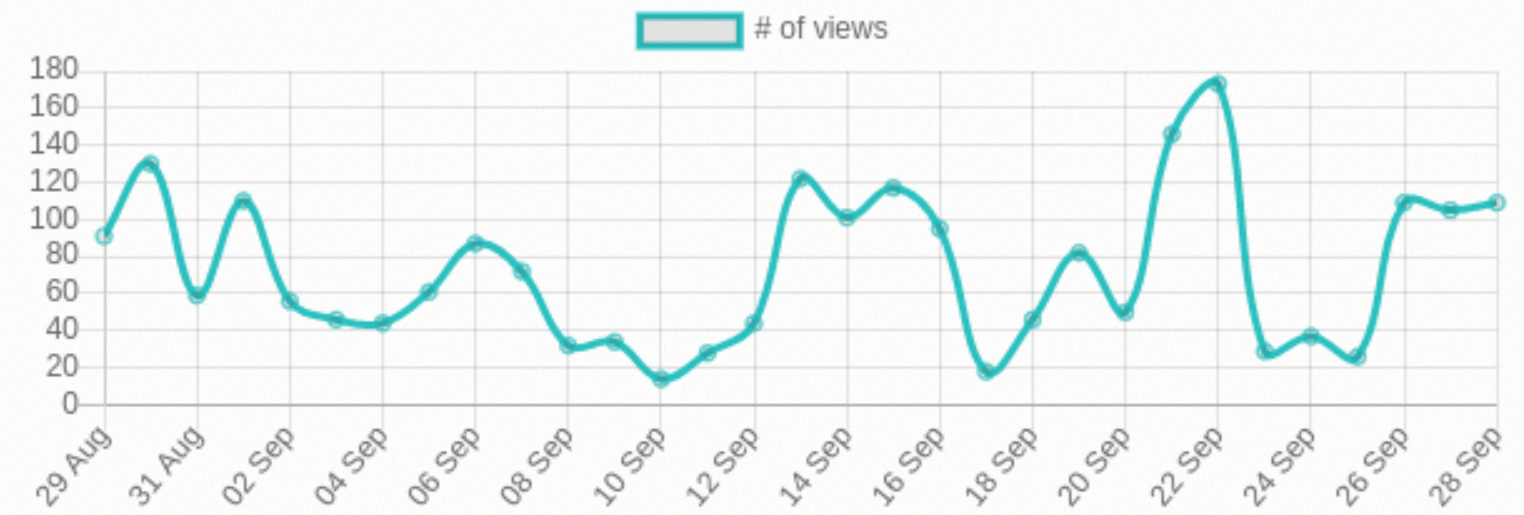}
        \caption{cvnn documentation daily total views in \textit{Read The Docs}.}
        \label{fig:docs-views}
    \end{subfigure}
    \caption{Documentation hosted by \textit{Read the Docs}.}
\end{figure}

The \textit{Python} testing framework \href{https://docs.pytest.org/en/7.0.x/}{Pytest} was used to maintain a maximum degree of quality and keep it as bug-free as possible. 
Before each new feature implementation, a test module was created to assert the expected behavior of the feature and reduce bugs to a minimum. This methodology also guaranteed feature compatibility as all designed test modules must pass in order to deploy the code.
The library allows the implementation of a \acrfull{rvnn} as well with the intention of changing as little as possible the code when using complex data types. This made it possible to perform a straight comparison against \textit{Tensorflow}'s models, which helped in the debugging. Indeed, some \textit{Pytest} modules achieved the same result that \textit{Tensorflow} when initializing the models with the same seed.

Special effort was made on the \acrfull{ux} by keeping the \acrfull{api} as similar as possible to that of \textit{Tensorflow}. The following code extract would work both for a \textit{Tensorflow} or cvnn application code:

\begin{lstlisting}[language=Python]
import numpy as np
import tensorflow as tf

# Gets the dataset, when using cvnn you normally want this to be complex 
# for example numpy arrays of dtype np.complex64
# to be done by each user
(train_images, train_labels), (test_images, test_labels) = get_dataset()        

# This function returns a tf.Model object
model = get_model()  

# Compile as any TensorFlow model
model.compile(optimizer='adam', metrics=['accuracy'], loss=tf.keras.losses.SparseCategoricalCrossentropy(from_logits=True))
model.summary()

# Train and evaluate
history = model.fit(train_images, train_labels, epochs=epochs, validation_data=(test_images, test_labels))
test_loss, test_acc = model.evaluate(test_images,  test_labels, verbose=2)
\end{lstlisting}

For creating the model, two \acrshort{api}s are available, the first one known as the sequential \acrshort{api} whose usage is something like the following code extract:

\begin{lstlisting}[language=Python]
import cvnn.layers as layers

def get_model():
    model = tf.keras.models.Sequential()
    model.add(layers.ComplexInput(input_shape=(32, 32, 3))) 
    model.add(layers.ComplexConv2D(32, (3, 3), activation='cart_relu'))
    model.add(layers.ComplexAvgPooling2D((2, 2)))
    model.add(layers.ComplexConv2D(64, (3, 3), activation='cart_relu'))
    model.add(layers.ComplexMaxPooling2D((2, 2)))
    model.add(layers.ComplexConv2D(64, (3, 3), activation='cart_relu'))
    model.add(layers.ComplexFlatten())
    model.add(layers.ComplexDense(64, activation='cart_relu'))
    model.add(layers.ComplexDense(10, activation='convert_to_real_with_abs'))   
    # An activation that casts to real must be used at the last layer. 
    # The loss function cannot minimize a complex number
    return model
\end{lstlisting}

However, some models are simply impossible to create with the sequential \acrshort{api} like a \acrshort{unet} architecture. For that, the functional \acrshort{api} must be used like this:

\begin{lstlisting}[language=Python]
import cvnn.layers as layers

def get_model():
    inputs = layers.complex_input(shape=(128, 128, 3))
    c0 = layers.ComplexConv2D(32, activation='cart_relu', kernel_size=3)(inputs)
    c1 = layers.ComplexConv2D(32, activation='cart_relu', kernel_size=3)(c0)
    c2 = layers.ComplexMaxPooling2D(pool_size=(2, 2), strides=(2, 2), padding='valid')(c1)
    t01 = layers.ComplexConv2DTranspose(5, kernel_size=2, strides=(2, 2), activation='cart_relu')(c2)
    concat01 = tf.keras.layers.concatenate([t01, c1], axis=-1)

    c3 = layers.ComplexConv2D(4, activation='cart_relu', kernel_size=3)(concat01)
    out = layers.ComplexConv2D(4, activation='cart_relu', kernel_size=3)(c3)
    return tf.keras.Model(inputs, out)
\end{lstlisting}

When using Keras loss functions, the output of the model must be real-valued as Tensorflow will not work with complex values. There are activation functions defined that receive complex values as input but output real value results to deal with this problem. 
Another solution is to use the cvnn library-defined loss functions instead.

\textit{Tensorflow} blocks the use of a complex-valued loss as the optimizer input but allows the update over complex-valued trainable parameters by means of the \textit{Wirtinger calculus} (Appendix \ref{chap:wirtinger-calculus}).
As the loss is real-valued, the optimizer can be the same as the one used for real-valued networks, so no implementation was needed in this regard. However, the other modules must be implemented. Their detail will be described in the following Sections.

\section{Mathematical Background}

The beginning of \acrshort{cvnn} was deferred mainly because of Liouville's theorem, which stated that a bounded and differentiable throughout all the complex domain is then a constant function, forcing the loss function to be non-holomorphic.
Liouville's theorem implications were considered to be a big problem around 1990 as some researchers believed indifferentiability should lead to the impossibility to obtain and/or analyze the dynamics of the \acrshort{cvnn}s \cite{hirose2013complex}.
Later Wirtinger calculus proposed a gradient definition for non-holomorphic functions and is now widely used for optimizing and training complex models.

\subsection{Liouville Theorem}\label{chap:liouville-theorem}

Liouville graduated from the École Polytechnique in 1827. After some years as an assistant at various institutions including the École Centrale Paris, he was appointed as a professor at the École Polytechnique in 1838.

\begin{definition}{}
    In complex analysis, an \textit{entire function}, also called an \textit{integral function}, is a complex-valued function that is \textit{holomorphic} at all finite points over the whole complex plane. 
\end{definition}
\begin{definition}{}
    Given a function $f$, the function is \textit{bounded} if $\exists M \in \mathbb{R^+} : |f(z)| < M$.
\end{definition}

\begin{theorem}[Cauchy integral theorem]
    Given $f$ analytic through region \textbf{D}, then the contour integral of $f(z)$ along any close path $C$ inside region \textbf{D} is zero:
    \begin{equation}
        \oint_{C} f(z) dz = 0 \, .
    \end{equation}
\end{theorem}

\begin{theorem}[Cauchy integral formula]
     Given $f$ analytic on the boundary $C$ with $z_0$ any point inside $C$, then:
     \begin{equation}
         f(z_0) = \frac{1}{2 \pi \im} \oint_{C} \frac{f(z)}{z-z_0}\,dz \, ,
     \end{equation}
     where the contour integration is taken in the counterclockwise direction.
\end{theorem}

\begin{corollary}
\begin{equation}
    f^{(n)}(z_0) = \frac{n!}{2 \pi \im} \oint_{C} \frac{f(z)}{(z-z_0)^{n+1}}\, dz\, .
\end{equation}
\label{cor:cauchy-integral-formula}
\end{corollary}

In complex analysis, Liouville's theorem states that every \textit{bounded entire function} must be constant. That is:
\begin{theorem}[Liouville's Theorem]
\begin{equation}
    f: \mathbb{C} \longrightarrow \mathbb{C} \text{ holomorphic } / \exists M \in \mathbb{R^+}: \left|f(z)\right|  < M, \forall z \in \mathbb{C} \Rightarrow f = c, c \in \mathbb{C} \, .
\end{equation}
\end{theorem}

Equivalently, non-constant \textit{holomorphic} functions on $ \mathbb{C} $ have unbounded images.

\begin{proof}
    Because every \textit{holomorphic} function is analytic (Class $C^\infty$), that is, it can be represented as a Taylor series, we can therefore write it as:
    \begin{equation}
        f(z) = \sum_{k = 0}^{\infty} a_k \,z^k\, .
        \label{eq:taylor}
    \end{equation}

    As $f$ is \textit{holomorphic} in the open region enclosed by the path and continuous on its closure. Because of \ref{cor:cauchy-integral-formula} we have:
    \begin{equation}
        a_k = \frac{f^{(k)}(0)}{k!} = \frac{1}{2 \pi i} \oint_{C_r} \frac{f(z)}{z^{k+1}dz} \, ,
    \end{equation}
    where $C_r$ is a circle of radius $r > 0$. Because $f$ is bounded:
    \begin{align}
        \left| a_k \right| \leqslant \frac{1}{2 \pi} \oint_{C_r} \frac{\left| f(z) \right| }{ \left|z\right|^{k+1}} \left|dz\right| \leqslant \frac{1}{2 \pi} \oint_{C_r} \frac{M}{ r^{k+1}} \left|dz\right| 
        &= \frac{M}{2 \pi r^{k+1}} \oint_{C_r} \left|dz\right|\, , \\  
        &= \frac{M}{2 \pi r^{k+1}} 2\pi \, ,\\ 
        &= \frac{M}{r^k}\, .
    \end{align}
    The latest derivation is also known as \textit{Cauchy's inequality}.
    
    As $r$ is any positive real number, by taking the $r \to \infty$ then $a_k \to 0$ for all $k \neq 0$. Therefore from Equation \ref{eq:taylor} we have that $f(z) = a_0$.
\end{proof}

\subsection{Wirtinger Calculus}\label{chap:wirtinger-calculus}

Wirtinger calculus, named after Wilhelm Wirtinger (1927) \cite{WIRTINGER_CALCULUS}, generalizes the notion of complex derivative, and the \textit{holomorphic} functions become a special case only. Further reading on Wirtinger calculus can be found in \cite{KREUTZ_DELGADO_WIRTINGER_CALCULUS, WILEY_APPENDIX_A}.

\begin{theorem}[Wirtinger Calculus]
    Given a complex function $f(z)$ of a complex variable $z = x + \im \, y \in \mathbb{C}, x, y \in \mathbb{R}$. 
    
    The partial derivatives with respect to $z$ and $\conjugate{z}$ respectively are defined as:
    \begin{equation}
        \begin{aligned}
            \frac{\partial f}{\partial z} \triangleq \frac{1}{2}\left( \frac{\partial f}{\partial x} - \im\,  \frac{\partial f}{\partial y} \right) \, ,\\
            \frac{\partial f}{\partial \conjugate{z}} \triangleq \frac{1}{2}\left( \frac{\partial f}{\partial x} + \im\, \frac{\partial f}{\partial y} \right) \, .
        \end{aligned}
        \label{eq:wirtinger-calculus}
    \end{equation}
\end{theorem}

These derivatives are called $\mathbb{R}$-derivative and conjugate $\mathbb{R}$-derivative, respectively. As said before, the \textit{holomorphic} case is only a special case where the function can be considered as $f(z,\conjugate{z}), \conjugate{z} = 0$.
Wirtinger calculus enables to work with \textit{non-holomorphic} functions, providing an alternative method for computing the gradient that also improves the stability of the training process.

\begin{proof}
    Defining $z = x + \im \,y$ one can also define $x(z, \conjugate{z}), y(z, \conjugate{z}) : \mathbb{C} \longrightarrow \mathbb{R}$ as follows:
    \begin{equation}
        \begin{aligned}
            x &= \frac{1}{2} \left( z + \conjugate{z} \right) \, ,\\
            y &= \frac{1}{2\, \im} \left( z - \conjugate{z} \right)\, .
        \end{aligned}
    \end{equation}
    
    Using \eqref{the:complex-chain-rule-real-im-part}:
    \begin{equation}
    \begin{aligned}
        \frac{\partial f}{\partial z} 
            & = \frac{\partial f}{\partial x} \frac{\partial x}{\partial z} +  \frac{\partial f}{\partial y} \frac{\partial y}{\partial z} \, ,\\
            & = \frac{\partial f}{\partial x} \frac{1}{2} +  \frac{\partial f}{\partial y} \left(\frac{-\im}{2}\right) \, ,\\
            & = \frac{1}{2}\left( \frac{\partial f}{\partial x} - \im \, \frac{\partial f}{\partial y} \right)\, .
    \end{aligned}
    \end{equation}
\end{proof}

\begin{corollary} \label{cor:partial-f-g-real}
    \begin{equation}
    \begin{aligned}
        \frac{\partial f}{\partial g} + \frac{\partial f}{\partial \conjugate{g}} & = 
        \frac{1}{2} \left( \frac{\partial f}{\partial g_{\Re{}}} - \im \,\frac{\partial f}{\partial g_{\Im{}}}  \right) 
        + \frac{1}{2} \left( \frac{\partial f}{\partial g_{\Re{}}} + \im \,\frac{\partial f}{\partial g_{\Im{}}}  \right) \, ,\\
        & = \frac{\partial f}{\partial g_{\Re{}}} \\
        i \left( \frac{\partial f}{\partial g} - \frac{\partial f}{\partial \conjugate{g}} \right) 
        & = 
        \frac{i}{2} \left( \frac{\partial f}{\partial g_{\Re{}}} - \im \,\frac{\partial f}{\partial g_{\Im{}}}  \right) 
        - \frac{i}{2} \left( \frac{\partial f}{\partial g_{\Re{}}} + \im \,\frac{\partial f}{\partial g_{\Im{}}}  \right) \, ,\\
        & = \frac{i}{2} \left( - 2 \im \, \frac{\partial f}{\partial g_{\Im{}}} \right) \, ,\\
        & = \frac{\partial f}{\partial g_{\Im{}}} \, ,
    \end{aligned}
    \end{equation}
    where $g_\Re{}$ and $g_\Im{}$ are the real and imaginary part of $g$ respectively.
\end{corollary}

\begin{theorem}
     Given $f : \mathbb{C} \longrightarrow \mathbb{C}$ \textit{holomorphic} with $f(x+\im y) = u(x,y)+iv(x,y)$ where $u,v: \mathbb{R} \longrightarrow \mathbb{R} $ real-differentiable functions. Then 
     $$\displaystyle\frac{\partial f}{\partial \conjugate{z}} = 0\, .$$
     \label{the:holo-df-dzc-0}
\end{theorem}


\begin{proof}
    Using Wirtinger calculus (Section \ref{chap:wirtinger-calculus}).
    \begin{equation}
        \frac{\partial f}{\partial \conjugate{z}} = 
        \frac{1}{2} \left( \frac{\partial f}{\partial x} + \im \, \frac{\partial f}{\partial y} \right) \, .
    \end{equation}
    By definition, then:
    \begin{equation}
    \begin{aligned}
        \frac{\partial f}{\partial \conjugate{.}} 
        & = \frac{1}{2} \left( \frac{\partial f}{\partial x} + \im \, \frac{\partial f}{\partial y} \right) \, ,\\
        & = \frac{1}{2} \left( 
            \left(
                \frac{\partial u}{\partial x} + \im \,  \frac{\partial v}{\partial x}
            \right) + \im \,
            \left(
                \frac{\partial u}{\partial y} + \im \,  \frac{\partial v}{\partial y}
            \right)
            \right) \, ,\\
        & = \frac{1}{2} \left( 
            \left(
                \frac{\partial u}{\partial x} -  \frac{\partial v}{\partial y}
            \right) + \im \,
            \left(
                \frac{\partial v}{\partial x} +  \frac{\partial u}{\partial y}
            \right)
            \right) \, .
    \end{aligned}
    \label{eq:partial-f-respect-conjugate-z}
    \end{equation}
    
    Because $f$ is \textit{holomorphic} then the Cauchy-Riemann equations (Theorem \ref{the:cauchy-riemann}) applies making Equation \ref{eq:partial-f-respect-conjugate-z} equal to zero.
\end{proof}


Even though so far we have always talked about general chain rule definitions. Here we will demonstrate a particularly interesting chain rule used when working with neural networks. For this optimization technique, the cost function to optimize is always real, even if its variables are not. Therefore, the following chain rule will have an application interest.

\begin{theorem}[complex chain rule with real output]
Given $f:\mathbb{C} \to \mathbb{R}$, $g:\mathbb{C} \to \mathbb{C}$ with $g(z)=r(z)+\im \,s(z)$, $z = x + \im \,y \in \mathbb{C}$:
\begin{equation}
    \frac{\partial f}{\partial z} = \frac{\partial f}{\partial r}\frac{\partial r}{\partial z} + \frac{\partial f}{\partial s}\frac{\partial s}{\partial z} \, .
\end{equation}

\begin{proof}
    For this proof, we will assume we are already working with Wirtinger Calculus for the partial derivative definition.
    \begin{equation}
    \begin{aligned}
        \frac{\partial f}{\partial z} 
            &= \frac{\partial f}{\partial g}\frac{\partial g}{\partial z} + \frac{\partial f}{\partial \conjugate{g}}\frac{\partial \conjugate{g}}{\partial z} \, ,\\
            &= \frac{1}{4} \left(\frac{\partial f}{\partial r} - \im \,\frac{\partial f}{\partial s}\right) \left(\frac{\partial g}{\partial x} - \im \,\frac{\partial g}{\partial y}\right) 
            + \frac{1}{4} \left(\frac{\partial f}{\partial r} + \im \,\frac{\partial f}{\partial s}\right) \conjugate{\left(\frac{\partial g}{\partial x} + \im \,\frac{\partial g}{\partial y}\right)} \, ,\\
            &= \frac{1}{4} \left(\frac{\partial f}{\partial r} - \im \,\frac{\partial f}{\partial s}\right)
            \left[ 
                \left(\frac{\partial r}{\partial x} + \im \,\frac{\partial s}{\partial x}\right) - \im \, \left(\frac{\partial r}{\partial y} + \im \,\frac{\partial s}{\partial y}\right)
            \right] + \cdots \\
            & \cdots + \frac{1}{4} \left(\frac{\partial f}{\partial r} + \im \,\frac{\partial f}{\partial s}\right) \conjugate{\left[ \left(\frac{\partial r}{\partial x} + \im \,\frac{\partial s}{\partial x}\right) + \im \,\left(\frac{\partial r}{\partial y} + \im \,\frac{\partial s}{\partial y}\right)\right] } \, . 
    \end{aligned}
    \end{equation}
\end{proof}

\end{theorem}

\subsection{Notation} \label{sec:notation}

A \acrshort{mlp} can be represented generically by Figure \ref{fig:mlp-net-diagram}. For that given multi-layered neural network, we define the following variables, which will be used in the following subsections and throughout our work:

\begin{itemize}
    \item{\makebox[2cm][l]{$0 \leq l \leq L$} corresponds to the layer index where $L$ is the output layer index and 0 is the input layer index.}
    \item{\makebox[2cm][l]{$1 \leq n \leq N_l$} the neuron index, where $N_l$ denotes the number of neurons of layer $l$.}
    \item{\makebox[2cm][l]{$\omega^{(l)}_{nm}$} weight of the $n^{th}$ neuron of layer $l-1$ with the $m^{th}$ neuron of layer $l$.}
    \item{\makebox[2cm][l]{$\activation$} activation function.}
    \item{\makebox[2cm][l]{$X_n^{(l)} = \activation \left( V_n^{(l)} \right)$} \qquad considered the output of layer $l$ and input of layer $l+1$, in particular, $X_n^{(L)} = y_n$. With $V_n^{(l)}$ being}
    \item $V_n^{(l)} = \displaystyle \sum_{m=1}^{N_{l-1}} \omega^{(l)}_{nm} X_m^{(l-1)}$
    \item{\makebox[2cm][l]{$e_n(d_n, y_n)$} error function. $d_n$ is the desired output for neuron $n$ of the output layer.}
    \item{\makebox[2cm][l]{$ \loss = \displaystyle \sum_n^{N_L} e_n$} cost or loss function.}
    \item{\makebox[2cm][l]{$ E = \displaystyle \frac{1}{P} \sum_p^P \loss_p$} \qquad minimum error function, with $P$ the total number of training cases or the size of the desired batch size.}
\end{itemize}

\begin{figure}[ht]
\centering
\includegraphics[width=\linewidth]{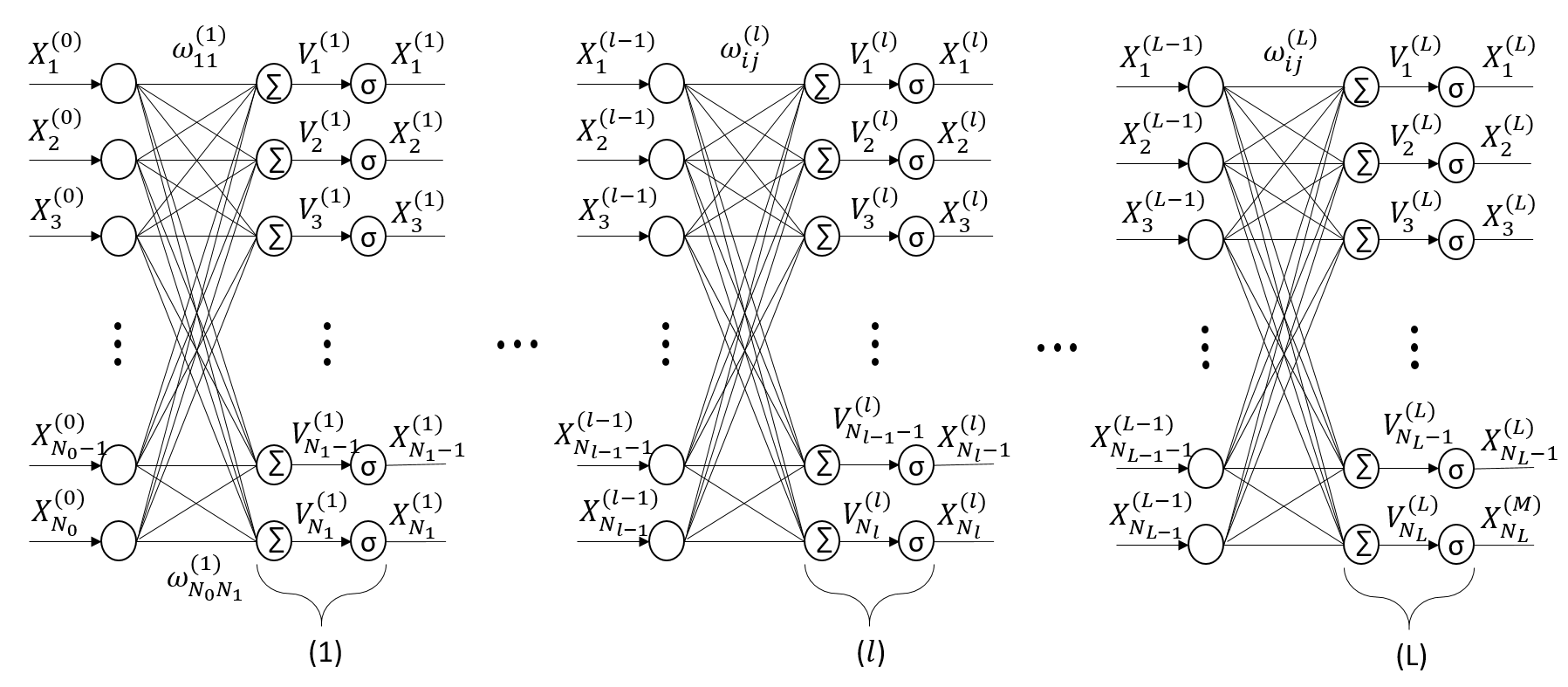}
\caption{Feed-forward Neural Network Diagram.}
\label{fig:mlp-net-diagram}
\end{figure}

\subsection{Complex-Valued Backpropagation}

The loss function remains real-valued to minimize an empirical risk during the learning process. Despite the architectural change for handling complex-valued inputs, the main challenge of \acrshort{cvnn} is the way to train such neural networks.

A problem arises when implementing the learning algorithm (commonly known as backpropagation). The parameters of the network must be optimized using the gradient or any partial-derivative-based algorithm. However, standard complex derivatives only exist for the so-called \textit{holomorphic} or \textit{analytic} functions.

Because of Liouville's theorem (discussed in Section \ref{chap:liouville-theorem}), \acrshort{cvnn}s are bound to use \textit{non-holomorphic} functions and therefore can not be derived using standard complex derivative definition. \acrshort{cvnn}s bring in \textit{non-holomorphic} functions in at least two ways \cite{amin2013learning}:
\begin{itemize}
    \item with the loss function being minimize over complex parameters
    \item with the \textit{non-holomorphic} complex-valued activation functions
\end{itemize}

Liouville's theorem implications were considered to be a big problem around 1990 as some researchers believed it led to the impossibility of obtaining and/or analyzing the dynamics of the \acrshort{cvnn}s \cite{hirose2013complex}.

However, Wirtinger calculus (discussed in Section \ref{chap:wirtinger-calculus}) generalizes the notion of complex derivative, making the \textit{holomorphic} function a special case only, allowing researchers to successfully implement \acrshort{cvnn}s. Under Wirtinger calculus, the gradient is defined as \cite{amin2011, LI2008GRADIENT}: 
\begin{equation}
    \nabla_z f = 2 \, \frac{\partial f}{\partial \conjugate{z}} \, .
    \label{eq:complex-gradient-def}
\end{equation}

When applying reverse-mode \acrshort{autodiff} on the complex domain, some good technical reports can be found, such as \cite{boeddeker2017computation} or \cite{hunger2007introduction}. However, it is left to be verified if \textit{Tensorflow} correctly applies the equations mentioned in these reports.
Indeed, no official documentation could be found that asserts the implementation of these equations and \textit{Wirtinger Calculus} when using \textit{Tensorflow} gradient on complex variables. This is not the case with \textit{PyTorch}, where they explicitly say that Wirtinger calculus is used when computing the derivative in the following link: \href{https://pytorch.org/docs/stable/notes/autograd.html#autograd-for-complex-numbers}{\path{pytorch.org/docs/stable/notes/autograd.html}}. In said link, they indirectly say also that "This convention matches TensorFlow’s convention for complex differentiation [...]" referencing the implementation of Equation \ref{eq:complex-gradient-def}.

However, we do know reverse-mode \acrshort{autodiff} is the method used by \textit{Tensorflow} \cite{geron2019hands}.
When reverse engineering the gradient definition of \textit{Tensorflow}, the conclusion discussed on the official \textit{Tensorflow}'s GitHub repository issue report \href{https://github.com/tensorflow/tensorflow/issues/3348}{3348} is that the gradient for $f: \mathbb{C} \to \mathbb{C}$ is computed as:
\begin{equation}
    \nabla_z f = \conjugate{\left( \frac{\partial f}{\partial z} + \frac{\partial \conjugate{f}}{\partial z} \right)} = 2 \frac{\partial \Re(f)}{\partial \conjugate{z}}\, .
\end{equation}

For application purposes, as the loss function is real-valued, we are only interested in cases where $f: \mathbb{C}^n \longrightarrow \mathbb{R} $ for what the above equation can be simplified as:
\begin{equation}
    \nabla_z f = 2 \frac{\partial f}{\partial \conjugate{z}} = \left( \frac{\partial f}{\partial x} + \im \frac{\partial f}{\partial y} \right) \, ,
\end{equation}

which indeed coincides with Wirtinger calculus definition. For this reason, it was not necessary to implement \acrshort{autodiff} from scratch, and \textit{Tensorflow}'s algorithm was used instead.

The mathematical equations on how to compute this figure are explained on Appendix \ref{sec:real-valued-backprop} and it's implementation for automatic calculation on a CPU or GPU is described in Appendix \ref{sec:forward-mode-autodiff} and \ref{sec:reverse-autodiff}.

The analysis in the complex case is analogous to that made in the real-valued backpropagation on Appendix \ref{sec:real-valued-backprop}. With the difference that now $\activation : \mathbb{C} \longrightarrow \mathbb{C}$, $e_n : \mathbb{C} \longrightarrow \mathbb{R}$ and $\omega_{ij}^{(l)}, X_n^{(l)}, V_n^{(l)}, e_n^{(l)} \in \mathbb{C}$.\\

Now, the chain rule is changed using \eqref{the:complex-chain-rule} so Equation \eqref{eq:partial-loss-function-chain-rule-definition} changes to:
\begin{equation}
    \frac{\partial e}{\partial \omega}
    = \frac{\partial e}{\partial X} \frac{\partial X}{\partial V} \frac{\partial V}{\partial \omega}
    + \frac{\partial e}{\partial X} \frac{\partial X}{\partial \conjugate{V}} \frac{\partial \conjugate{V}}{\partial \omega}
    + \frac{\partial e}{\partial \conjugate{X}} \frac{\partial \conjugate{X}}{\partial V} \frac{\partial V}{\partial \omega}
    + \frac{\partial e}{\partial \conjugate{X}} \frac{\conjugate{X}}{\partial \conjugate{V}} \frac{\partial \conjugate{V}}{\partial \omega} \, .
    \label{eq:complex-partial-loss-function-chain-rule}
\end{equation}

Note that we have used the upper line to denote the conjugate for clarity. All subindexes have been removed for clarity but they stand the same as in Equation \ref{eq:partial-loss-function-chain-rule-definition}.

As $e_n : \mathbb{C} \longrightarrow \mathbb{R}$ then using the conjugation rule (Definition \ref{def:conjugation-rule}):
\begin{equation}
    \begin{aligned}
        \frac{\partial e}{\partial \conjugate{X}} & = \conjugate{\left( \frac{\partial e}{\partial X} \right)} \, ,\\ 
        \frac{\partial \conjugate{X}}{\partial \conjugate{V}} & = \conjugate{\left( \frac{\partial X}{\partial V} \right)} \, ,\\
        \frac{\partial X}{\partial \conjugate{V}} & = \conjugate{\left( \frac{\partial \conjugate{X}}{\partial V} \right)} \, ,\\
    \end{aligned}
\end{equation}
so that not all the partial derivatives must be calculated.

Focusing our attention on the derivative $\partial V / \partial \omega$, a differentiation between layer difference of $V$ and $\omega$ will be made in this approach. The simplest is the one where the layer from $V_n^{(l)}$ is the same as the weight. Regardless of the complex domain, $V_n^{(l)}$ is still equal to $\displaystyle \sum_i^{N_{l-1}} \omega^{(l)}_{ni} X_i^{(l-1)}$ for what the value of the derivative remains unchanged. For the wights ($\omega$) of the previous layer, the derivative is as follows:
\begin{equation}
\begin{aligned}
    \frac{\partial V_n^{(l)}}{\partial \omega_{jk}^{(l-1)}} & = \frac{\partial \left( \displaystyle \sum_i^{N_{l-1}} \omega^{(l)}_{n\mathrm{j}} X_\mathrm{j}^{(l-1)} \right)}{\partial \omega_{jk}^{(l-1)}} \, ,\\
    & = \omega_{nj}^{(l)} \frac{\partial X_j^{(l-1)}}{\partial \omega_{jk}^{(l-1)}} \, ,\\
    & = \omega_{nj}^{(l)} \left[ \frac{\partial X_j^{(l-1)}}{\partial V_j^{(l-1)}} \frac{\partial V_j^{(l-1)}}{\partial \omega_{jk}^{(l-1)}} + \frac{\partial X_j^{(l-1)}}{\partial \conjugate{V}_j^{(l-1)}} \frac{\partial \conjugate{V}_j^{(l-1)}}{\partial \omega_{jk}^{(l-1)}} \right] \, .
\end{aligned}
\end{equation}

Now by definition, $V_n^{(l)}$ is analytic because of being a polynomial series and therefore is holomorphic as well. Using Theorem \ref{the:holo-df-dzc-0}, $\partial \conjugate{V}_j^{(l-1)}/ \partial\omega_{jk}^{(l-1)} = 0$. The second term could be removed. Therefore the equation is simplified to the following:
\begin{equation}
    \frac{\partial V_n^{(l)}}{\partial \omega_{jk}^{(l-1)}} = \omega_{nj}^{(l)} \frac{\partial X_j^{(l-1)}}{\partial V_j^{(l-1)}} \frac{\partial V_j^{(l-1)}}{\partial \omega_{jk}^{(l-1)}} 
    = \omega_{nj}^{(l)} \frac{\partial X_j^{(l-1)}}{\partial V_j^{(l-1)}} X_k^{(l-2)} \, .
    \label{eq:complex-partial-vn-w-l-h-1}
\end{equation}

For the rest of the cases where the layers are farther apart, the equation is as follows:
\begin{equation}
\begin{aligned}
    \frac{\partial V_n^{(l)}}{\partial \omega_{jk}^{(h)}} 
        & = \frac{\displaystyle\partial \left( \sum_i^{N_{l-1}} \omega^{(l)}_{ni} X_i^{(l-1)} \right)}{\partial \omega_{jk}^{(l-1)}} \, ,\\
        & = \sum_i^{N_{l-1}} \omega_{nj}^{(l)} \left[ \frac{\partial X_i^{(l-1)}}{\partial V_i^{(l-1)}} \frac{\partial V_i^{(l-1)}}{\partial \omega_{jk}^{(h)}} + \frac{\partial X_i^{(l-1)}}{\partial \conjugate{V}_i^{(l-1)}} \frac{\partial \conjugate{V}_i^{(l-1)}}{\partial \omega_{jk}^{(h)}} \right]\, ,
\end{aligned}
\label{eq:complex-partial-vn-w-l-h-2}
\end{equation}
where $h \leq l - 2$.
Therefore, based on Equations \ref{eq:complex-partial-vn-w-l-h-2} and \ref{eq:complex-partial-vn-w-l-h-1}, the final equation remains as follows:
\begin{equation}
    \frac{\partial V_n^{(l)}}{\partial \omega_{jk}^{(h)}} =
    \begin{cases}
        X_j^{(l-1)} & h = l \, ,\\
        \omega_{nj}^{(l)} \frac{\partial X_j^{(l-1)}}{\partial V_j^{(l-1)}} \frac{\partial V_j^{(l-1)}}{\partial \omega_{jk}^{(l-1)}} & h = l - 1 \, ,\\
        \displaystyle\sum_i^{N_{l-1}} \omega_{nj}^{(l)} \left[ \frac{\partial X_i^{(l-1)}}{\partial V_i^{(l-1)}} \frac{\partial V_i^{(l-1)}}{\partial \omega_{jk}^{(h)}} + \frac{\partial X_i^{(l-1)}}{\partial \conjugate{V}_i^{(l-1)}} \frac{\partial \conjugate{V}_i^{(l-1)}}{\partial \omega_{jk}^{(h)}} \right] & h \leq l - 2\, .
    \end{cases}
    \label{eq:complex-partial-vn-w-generic}
\end{equation}

Using the property that $\partial \conjugate{V} / \partial \omega = \conjugate{(\partial V / \partial \conjugate{\omega})}$ and the distributed properties of the conjugate, the following equation can be derived:
\begin{equation}
    \frac{\partial \conjugate{V}_n^{(l)}}{\partial \omega_{jk}^{(h)}} =
    \begin{cases}
        0 & h = l \, ,\\
        \displaystyle\conjugate{\omega}_{nj}^{(l)} \frac{\partial \conjugate{X}_j^{(l-1)}}{\partial V_j^{(l-1)}} \frac{\partial V_j^{(l-1)}}{\partial \omega_{jk}^{(l-1)}} & h = l - 1 \, ,\\
        \displaystyle\sum_i^{N_{l-1}} \conjugate{\omega}_{nj}^{(l)} \left[ \frac{\partial \conjugate{X}_i^{(l-1)}}{\partial V_i^{(l-1)}} \frac{\partial V_i^{(l-1)}}{\partial \omega_{jk}^{(h)}} + \frac{\partial \conjugate{X}_i^{(l-1)}}{\partial \conjugate{V}_i^{(l-1)}} \frac{\partial \conjugate{V}_i^{(l-1)}}{\partial \omega_{jk}^{(h)}} \right] & h \leq l - 2\, .
    \end{cases}
    \label{eq:conjugate-complex-partial-vn-w-generic}
\end{equation}

Using Equations \ref{eq:complex-partial-vn-w-generic} and \ref{eq:conjugate-complex-partial-vn-w-generic}, we can calculate all possible values of $\partial V / \partial \omega$ and $\partial \conjugate{V} / \partial \omega$. Once defined the loss and the activation function, using Equation \ref{eq:complex-partial-loss-function-chain-rule}, backpropagation can be made also in the complex plane.

\subsubsection{H\"ansch and Hellwich definition}

Ronny H\"ansch and Olaf Helllwich \cite{HANSCH_HELLWICH} made a similar approach for the general equations of complex neural networks by using the complex chain rule. Using \eqref{eq:complex-partial-loss-function-chain-rule}, they define $X$ and $V$ from the same layer as the weight instead of $e$.

By doing so, and using the fact that $\partial \conjugate{V}_j^{(l-1)}/ \partial\omega_{jk}^{(l-1)} = 0$, two terms are deleted. In conjunction with the complex equivalent of Equation \ref{eq:partial-Vn-w-with-h-l-0}, Equation \ref{eq:complex-partial-loss-function-chain-rule} is simplified to:
\begin{equation}
\begin{aligned}
    \frac{\partial e_n}{\partial \omega_{ji}^{(l)}}  
    & = \frac{\partial e_n}{\partial X_i^{(l)}} \frac{\partial X_i^{(l)}}{\partial V_i^{(l)}} \frac{\partial V_i^{(l)}}{\partial \omega_{ji}^{(l)}} + \frac{\partial e_n}{\partial \conjugate{X}_i^{(l)}} \frac{\partial \conjugate{X}_i^{(l)}}{\partial V_i^{(l)}} \frac{\partial V_i^{(l)}}{\partial \omega_{ji}^{(l)}} \, ,\\
    & = \frac{\partial e_n}{\partial X_i^{(l)}} \frac{\partial X_i^{(l)}}{\partial V_i^{(l)}} X_j^{(l-1)} + \frac{\partial e_n}{\partial \conjugate{X}_i^{(l)}} \frac{\partial \conjugate{X}_i^{(l)}}{\partial V_i^{(l)}} X_j^{(l-1)}\, .
\end{aligned}
\label{eq:complex-hansch-loss-function-chain-rule}
\end{equation}

Now instead of making the analysis for $\partial V / \partial \omega$, the equivalent analysis will be made for $\partial e / \partial X$. The case with $l = L$ is the trivial one when its value depends on the chosen error function. For $l = L-1$, the following equality applies:
\begin{equation}
    \frac{\partial e_n}{\partial X_i^{(L - 1)}} 
     = \frac{\partial e_n}{\partial V_n^{L}} \frac{\partial V_n^{L}}{\partial X_i^{(L - 1)}} + \frac{\partial e_n}{\partial \conjugate{V}_n^{L}} \frac{\partial \conjugate{V}_n^{L}}{\partial X_i^{(L - 1)}} \, .
\label{eq:complex-partial-en-x-h-l-1}
\end{equation}

However, as $V_n^{(L)} = \displaystyle\sum_i \omega_{ni}^{(L)} X_i^{(L - 1)}$, its derivatives are:
\begin{equation}
\begin{aligned}
    \frac{\partial V_n^{(l+1)}}{\partial X_i^{(l)}} &= \omega_{ni}^{(l+1)} \, ,\\
    \frac{\partial \conjugate{V}_n^{(l+1)}}{\partial X_i^{(l)}} &= 0 \,.
\end{aligned}
\label{eq:wirtinger_derivative_wrt_neurons}
\end{equation}

For that reason, the second term of Equation \ref{eq:complex-partial-en-x-h-l-1} is deleted, and using the chain rule again we have:
\begin{equation}
\begin{aligned}
    \frac{\partial e_n}{\partial X_i^{(L - 1)}} 
    & = \frac{\partial e_n}{\partial V_n^{(L)}} \frac{\partial V_n^{(L)}}{\partial X_i^{(L - 1)}} \, ,\\
    &= \frac{\partial e_n}{\partial X_n^{(L)}} \frac{\partial X_n^{(L)}}{\partial V_n^{(L)}} \frac{\partial V_n^{(L)}}{\partial X_i^{(L - 1)}} + \frac{\partial e_n}{\partial \conjugate{X}_n^{L}} \frac{\partial \conjugate{X}_n^{(L)}}{\partial V_n^{(L)}} \frac{\partial V_n^{(L)}}{\partial X_i^{(L - 1)}} \, , \\
    &= \frac{\partial e_n}{\partial X_n^{(L)}} \frac{\partial X_n^{(L)}}{\partial V_n^{(L)}} \omega_{in}^{(L)} + \frac{\partial e_n}{\partial \conjugate{X}_n^{(L)}} \frac{\partial \conjugate{X}_n^{(L)}}{\partial V_n^{(L)}} \omega_{in}^{(L)} \, ,
\end{aligned}
\end{equation}
where the partial derivatives depend on the definition of the loss and activation function.

For the general case of $l \leq L - 2$ the derivation is similar to the previous one:
\begin{equation}
\begin{aligned}
    \frac{\partial e_n}{\partial X_i^{(l)}} 
    & = \sum_k \frac{\partial e_n}{\partial V_k^{l+1}} \frac{\partial V_k^{l+1}}{\partial X_i^{(l)}} + \frac{\partial e_n}{\partial \conjugate{V}_k^{l+1}} \frac{\partial \conjugate{V}_k^{l+1}}{\partial X_i^{(l)}} \, ,\\
    & = \sum_k \frac{\partial e_n}{\partial V_k^{l+1}} \frac{\partial V_k^{l+1}}{\partial X_i^{(l)}} \, ,\\
    & = \sum_k \frac{\partial e_n}{\partial V_k^{l+1}} \omega_{ik}^{(l+1)} \, , \\
    &= \sum_k \frac{\partial e_n}{\partial X_k^{(l+1)}} \frac{\partial X_k^{(l+1)}}{\partial V_k^{(l+1)}} \omega_{ik}^{(l+1)} + \frac{\partial e_n}{\partial \conjugate{X}_k^{(l+1)}} \frac{\partial \conjugate{X}_k^{(l+1)}}{\partial V_k^{(l+1)}} \omega_{ik}^{(l+1)}\, .
\end{aligned}
\end{equation}
Therefore, $\partial e / \partial X$ can be defined as:
\begin{equation}
    \frac{\partial e_n}{\partial X_i^{(l)}} =
    \begin{cases}
        \frac{\partial e_n}{\partial X_n^{(L)}} \, ,\qquad\qquad\qquad\qquad\qquad\qquad\qquad\qquad l = L \, ,\\
        \frac{\partial e_n}{\partial X_n^{(L)}} \frac{\partial X_n^{(L)}}{\partial V_n^{(L)}} \omega_{in}^{(L)} + \frac{\partial e_n}{\partial \conjugate{X}_n^{(L)}} \frac{\partial \conjugate{X}_n^{(L)}}{\partial V_n^{(L)}} \omega_{in}^{(L)} \, , \qquad\qquad l = L - 1 \, ,\\
        \begin{aligned}
            & \displaystyle\sum_k \left( \frac{\partial e_n}{\partial X_k^{(l+1)}} \frac{\partial X_k^{(l+1)}}{\partial V_k^{(l+1)}} \omega_{ik}^{(l+1)} + \frac{\partial e_n}{\partial \conjugate{X}_k^{(l+1)}} \frac{\partial \conjugate{X}_k^{(l+1)}}{\partial V_k^{(l+1)}} \omega_{ik}^{(l+1)} \right) \, , \\ 
            & \qquad\qquad\qquad\qquad\qquad\qquad\qquad\qquad\qquad\quad l \leq L - 2 \, .
        \end{aligned}
        
    \end{cases}
\label{eq:complex-partial-en-xi-generic}
\end{equation}
As $e_n : \mathbb{C} \longrightarrow \mathbb{R}$, then applying \eqref{def:conjugation-rule}:
\begin{equation}
    \frac{\partial e_n}{\partial \conjugate{X}_i^{(l)}} = \conjugate{\left( \frac{\partial e_n}{\partial X_i^{(l)}} \right)} \, .
\end{equation}
Using this latest equality with Equations \ref{eq:complex-hansch-loss-function-chain-rule} and \ref{eq:complex-partial-en-xi-generic}, the backpropagation algorithms are fully defined.

\section{Complex-Valued layers}

\acrshort{cvnn}, as opposed to conventional \acrshort{rvnn}, possesses complex-valued input, which allows working with imaginary data without any pre-processing needed to cast its values to the real-valued domain. Each layer of the complex network operates analogously to a real-valued layer with the difference that its operations are on the complex domain (addition, multiplication, convolution, etc.) with trainable parameters being complex-valued (weights, bias, kernels, etc.). Activation functions are also defined on the complex domain so that $f: \mathbb{C} \to \mathbb{C}$ and will be described on Section \ref{sec:activation-functions}.

A wide variety of complex layers is supported by the library, and the full list can be found in \href{https://complex-valued-neural-networks.readthedocs.io/en/latest/layers.html}{\path{complex-valued-neural-networks.rtfd.io/en/latest/layers.html}}.

Some layers, such as dense layers, can be extended naturally as addition and multiplication are defined in the complex domain. Therefore, just by making the neurons complex-valued, their behavior is evident. The same analogy can be made for convolutional layers as the transformation from the complex to the real plane does not change the resolution of the image to justify increasing the kernel size or changing the stride.

Special care must be taken when implementing \lstinline{ComplexDropout} as applying it to both real and imaginary parts separately will result in unexpected behavior as the ignoring weights will not be coincident, e.g., one might mask the real part while using the imaginary part for the computation. This, however, was taken into account for the layer implementation. The usage of this layer is analogous to \textit{Tensorflow} \href{https://www.tensorflow.org/api_docs/python/tf/keras/layers/Dropout}{Dropout} layer, which also uses the boolean \lstinline{training} parameter indicating whether the layer should behave in training mode (adding dropout) or in inference mode (doing nothing). 

Other layers, such as \lstinline{ComplexFlatten} or \lstinline{ComplexInput}, needed to be implemented as \textit{Tensorflow} equivalent \href{https://www.tensorflow.org/api_docs/python/tf/keras/layers/Flatten}{\lstinline{Flatten}} and \href{https://www.tensorflow.org/api_docs/python/tf/keras/Input}{\lstinline{Input}} cast the output to float.

\subsection{Complex Pooling Layers} \label{sec:pooling-layers}

Pooling layers are not so straightforward. In the complex domain, their values are not ordered as the real values, meaning there is no sense of a maximum value, rendering it impossible to implement a Max Pooling layer directly on the input. Reference \cite{zhang2017complex} proposes to use the norm of the complex figure to make this comparison, and this method is used for the toolbox implementation of the \lstinline{ComplexMaxPooling} layer. Average Pooling opens the possibility to other interpretations as well. Even if, for computing the average, we could add the complex numbers and divide by the total number of terms as one would do with real numbers, another option arises known as circular mean. The circular or angular mean is designed for angles and similar cyclic quantities. 

\begin{figure}[ht]
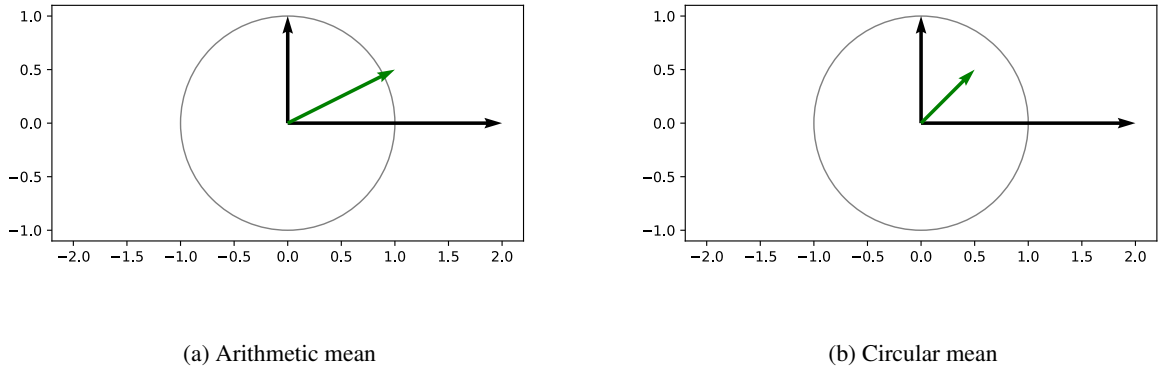

    \centering
    \begin{subfigure}[b]{0.49\textwidth}
         \centering
         \includesvg[width=\textwidth]{avg.svg}
         \caption{Arithmetic mean}
         \label{fig:conventional-mean}
    \end{subfigure}
    \hfill
    \begin{subfigure}[b]{0.49\textwidth}
         \centering
         \includesvg[width=\textwidth]{circ.svg}
         \caption{Circular mean}
         \label{fig:circular-mean}
    \end{subfigure}
    \caption{Mean example for two complex values.}
    \label{fig:mean-examples}
\end{figure}

When computing the average of $2 + 0 \,\im$ and $0 + 1 \,\im$, the conventional complex mean will yield $1 + 0.5 \,\im$ for what the angle will be $\pi / 6$ although the vectors had angles of $\pi / 2$ and $0$ (Figure \ref{fig:conventional-mean}).  The circular mean consists of normalizing the values before computing the mean, which yields $0.5 + 0.5\, \im$, having an angle of $\pi / 4$ (Figure \ref{fig:circular-mean}). The circular mean will have a norm inside the unit circle. It will be at the unit circle if all angles are equal, and it will be null if the angles are equally distributed.
Another option for computing the mean is to use the circular mean definition for the angle and then compute the arithmetic mean of the norm separately, as represented in Figure \ref{fig:mine-avg}.

\begin{figure}
    \centering
    \includesvg[width=0.5\textwidth]{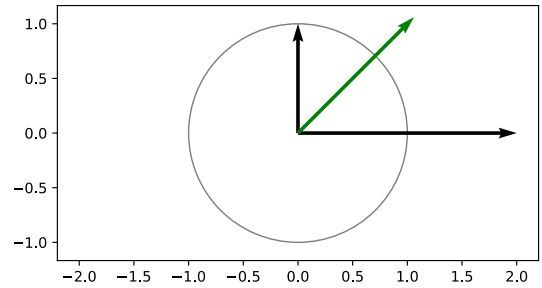}
    \caption{Circular mean with norm average.}
    \label{fig:mine-avg}
\end{figure}

All these options were used for computing the average pooling so the user could choose the case that fits their data best.

\subsection{Complex Upsampling layers} \label{sec:upsampling}

Upsampling techniques, which enable the enlargement of 2D images, were implemented. In particular, 3 techniques were applied, all of them documented in \href{https://complex-valued-neural-networks.readthedocs.io/en/latest/layers/complex_upsampling.html}{\path{complex-valued-neural-networks.readthedocs.io/en/latest/layers/complex_upsampling.html}}.

\begin{figure}
    \centering
    \includegraphics[width=0.7\textwidth]{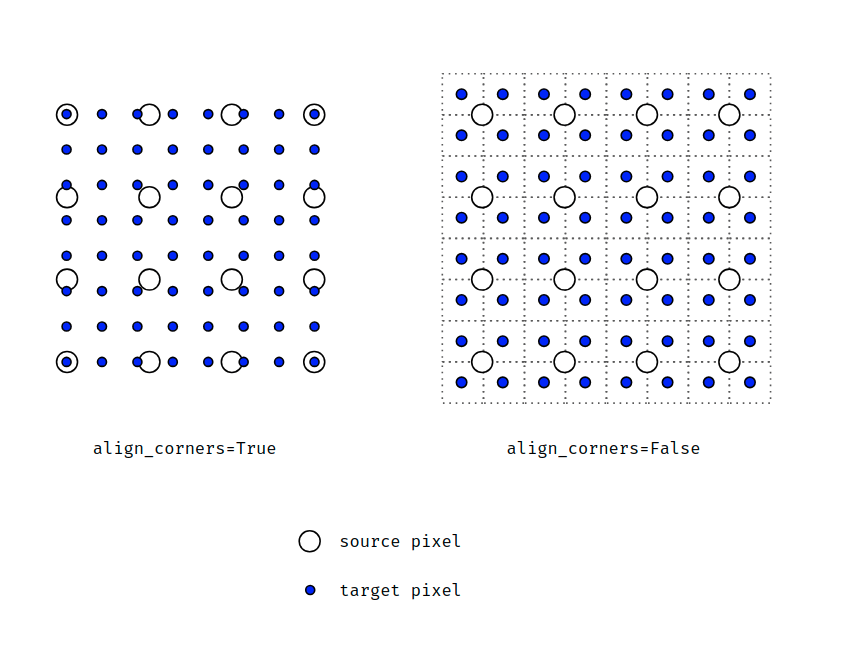}
    \caption{Upsampling alignments options extracted from \cite{forumPytorch}.}
    \label{fig:align-corners}
\end{figure}

\begin{itemize}
    \item \textbf{Complex Upsampling} Upsampling layer for 2D inputs. The upsampling can be done using the nearest neighbor or bilinear interpolation. There are at least two possible ways to implement the upsampling method depending if the corners are aligned or not (see Figure \ref{fig:align-corners}). Our implementation does not align corners.
    \item \textbf{Complex Transposed Convolution} Sometimes called \textbf{Deconvolution} although it does not compute the inverse of a convolution \cite{zeiler2010deconvolutional}.
    \item \textbf{Complex Un-Pooling} Inspired on the functioning of Max Un-pooling explained in Reference \cite{zafar2018hands}. Max un-pooling technique receives the maxed locations of a previous Max Pooling layer and then expands an image by placing the input values on those locations and filling the rest with zeros as shown in Figure \ref{fig:max-unpooling}.
\end{itemize}

\begin{figure}
    \centering
    \includegraphics[width=0.8\textwidth]{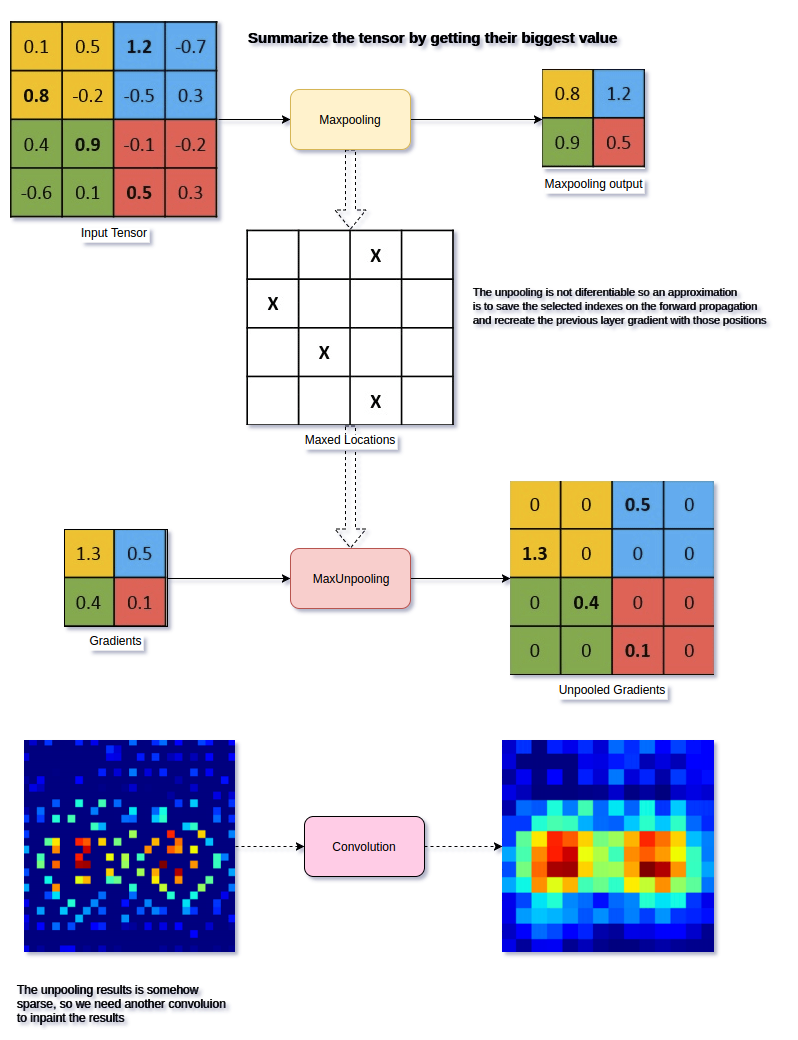}
    \caption{Max Unpooling graphical explanation extracted from \cite{zafar2018hands}.}
    \label{fig:max-unpooling}
\end{figure}

Complex un-pooling locations are not forced to be the output of a max pooling layer. However, in order to use it, we implemented as well a layer class named \lstinline{ComplexMaxPooling2DWithArgmax} which returns a tuple of tensors, the max pooling output and the maxed locations to be used as input of the un-pooling layer. 

There are two main ways to use the unpooling layer, either by using the expected output shape or using the \lstinline{upsampling_factor} parameter. 

\begin{lstlisting}[language=python]
from cvnn.layers import ComplexUnPooling2D, complex_input, ComplexMaxPooling2DWithArgmax
import tensorflow as tf
import numpy as np

x = get_img()       # Gets an image just for the example
# Removes the batch size shape
inputs = complex_input(shape=x.shape[1:])       
# Apply max-pooling and also get argmax
max_pool_o, max_arg = ComplexMaxPooling2DWithArgmax(strides=1, data_format="channels_last", name="argmax")(inputs)
# Applies the Unpooling
outputs = ComplexUnPooling2D(x.shape[1:])([max_pool_o, max_arg])

model = tf.keras.Model(inputs=inputs, outputs=outputs, name="pooling_model")
model.summary()
model(x)
\end{lstlisting}

It is possible to work with variable size images using a partially defined tensor, for example, \lstinline{shape=(None, None, 3)}. In this case, the second option (using \lstinline{upsampling_factor}) is the only way to deal with them in the following manner.

\begin{lstlisting}[language=python]
# Input is an unknown size RGB image
inputs = complex_input(shape=(None, None, 3)) 
max_pool_o, pool_argmax = ComplexMaxPooling2DWithArgmax(strides=1, data_format="channels_last", name="argmax")(inputs)
unpool = ComplexUnPooling2D(upsampling_factor=2)([max_pool_o, pool_argmax])

model = tf.keras.Model(inputs=inputs, outputs=outputs, name="pooling_model")
model.summary()
model(x)
\end{lstlisting}

All the discussed layers in this Section have a \lstinline{dtype} parameter which defaults to \lstinline{tf.complex64}, however, if \lstinline{tf.float32} or \lstinline{tf.float64} is used, the layer behaviour should be arithmetically equivalent to the corresponding \textit{Tensorflow} layer, allowing for fast test and comparison. In some cases, for example, \lstinline{ComplexFlatten}, this parameter has no effect as the layer can already deal with both complex- and real-valued input. 
Also, a method \lstinline{get_real_equivalent} is implemented which returns a new layer object with a real-valued \lstinline{dtype} and allows a \lstinline{output_multiplier} parameter in order to re-dimension the real network if needed. This is used to obtain an equivalent real-valued network as described \cite{barrachina2021about}. 

\section{Complex Activation functions} \label{sec:activation-functions}

One of the essential characteristics of \acrshort{cvnn} is its activation functions, which should be non-linear and complex-valued. 
An activation function is usually chosen to be piece-wise smooth to facilitate the computation of the gradient. 
The complex domain widens the possibilities to design an activation function, but the probable more natural way would be to extend a real-valued activation function to the complex domain. 

Our toolbox currently supports a wide range of complex-activation functions listed on \href{https://github.com/NEGU93/cvnn/blob/master/cvnn/activations.py#L534}{\lstinline{act_dispatcher}} dictionary.

\begin{lstlisting}[language=Python]
act_dispatcher = {
    'linear': linear,
    # Complex input, real output
    'cast_to_real': cast_to_real,
    'convert_to_real_with_abs': convert_to_real_with_abs,
    'sigmoid_real': sigmoid_real,
    'softmax_real_with_abs': softmax_real_with_abs,
    'softmax_real_with_avg': softmax_real_with_avg,
    'softmax_real_with_mult': softmax_real_with_mult,
    'softmax_of_softmax_real_with_mult': softmax_of_softmax_real_with_mult,
    'softmax_of_softmax_real_with_avg': softmax_of_softmax_real_with_avg,
    'softmax_real_with_polar': softmax_real_with_polar,
    # Phasor networks
    'georgiou_cdbp': georgiou_cdbp,
    'mvn_activation': mvn_activation,
    'complex_signum': complex_signum,
    # Type A (cartesian)
    'cart_sigmoid': cart_sigmoid,
    'cart_elu': cart_elu,
    'cart_exponential': cart_exponential,
    'cart_hard_sigmoid': cart_hard_sigmoid,
    'cart_relu': cart_relu,
    'cart_leaky_relu': cart_leaky_relu,
    'cart_selu': cart_selu,
    'cart_softplus': cart_softplus,
    'cart_softsign': cart_softsign,
    'cart_tanh': cart_tanh,
    'cart_softmax': cart_softmax,
    # Type B (polar)
    'pol_tanh': pol_tanh,
    'pol_sigmoid': pol_sigmoid,
    'pol_selu': pol_selu,
    # Elementary Transcendental Functions (ETF)
    'etf_circular_tan': etf_circular_tan,
    'etf_circular_sin': etf_circular_sin,
    'etf_inv_circular_atan': etf_inv_circular_atan,
    'etf_inv_circular_asin': etf_inv_circular_asin,
    'etf_inv_circular_acos': etf_inv_circular_acos,
    'etf_circular_tanh': etf_circular_tanh,
    'etf_circular_sinh': etf_circular_sinh,
    'etf_inv_circular_atanh': etf_inv_circular_atanh,
    'etf_inv_circular_asinh': etf_inv_circular_asinh,
    # ReLU
    'modrelu': modrelu,
    'crelu': crelu,
    'zrelu': zrelu,
    'complex_cardioid': complex_cardioid
}
\end{lstlisting}

Indeed, to implement an activation function, it will suffice to add it to the \lstinline{act_dispatcher} dictionary to have full functionality. 

In our published toolbox, there are two ways of using an activation function, either by using a string listed on \lstinline{act_dispatcher} like

\begin{lstlisting}[language=Python]
ComplexDense(units=x, activation='cart_sigmoid')
\end{lstlisting}

or by using the function directly

\begin{lstlisting}[language=Python]
from cvnn.activations import cart_sigmoid

ComplexDense(units=x, activation=cart_sigmoid)
\end{lstlisting}

This usage also support using \href{https://www.tensorflow.org/api_docs/python/tf/keras/layers/Activation}{\lstinline{tf.keras.layers.Activation}} to implement an activation function directly as an independent layer.

\begin{lstlisting}[language=Python]
    from cvnn.activations import cart_relu

    layer = tf.keras.layers.Activation('cart_relu')
    layer = tf.keras.layers.Activation(cart_relu)
\end{lstlisting}

Although complex activation functions used on \acrshort{cvnn} are numerous \cite{scardapane2018complex, bassey2021survey, lee2022survey}, we will mainly focus on two types of activation functions that are an extension of the real-valued functions \cite{kuroe2003activation}:
\begin{itemize}
	\item Type-A: $\activation_A(z) = \activation_{\Re}\left(\Re(z)\right) +       \im \, \activation_{\Im}\left(\Im(z) \right)$,
	\item Type-B: $\activation_B(z) = \activation_r(|z|) \, \exp{\left( \im \, \activation_\phi(\arg(z))\right)}$,
\end{itemize}
where $\activation_{\Re},\activation_{\Im},\activation_r,\activation_{\phi}$ are all real-valued  functions\footnote{Although not with the same notation, these two types of complex-valued activation functions are also discussed in Section 3.3 of \cite{hirose2012complex}}. $\Re$ and $\Im$ operators are the real and imaginary parts of the input, respectively, and the $\arg$ operator gives the phase of the input.
Note that in Type-A, the real and imaginary parts of an input go through nonlinear functions separately, and in Type-B, the magnitude and phase go through nonlinear functions separately. 

The most popular activation functions, sigmoid, \acrfull{tanh} and \acrfull{relu}, are extensible using Type-A or Type-B approach. Although \acrshort{tanh} is already defined on the complex domain for what, its transformation is probably less interesting.

Other complex-activation functions are supported by our toolbox including elementary transcentental functions (\href{https://complex-valued-neural-networks.readthedocs.io/en/latest/activations/etf.html}{\path{complex-valued-neural-networks.rtfd.io/en/latest/activations/etf.html}}) \cite{kim2001approximation, kim2001complex} or phasor activation function (\href{https://complex-valued-neural-networks.readthedocs.io/en/latest/activations/mvn_activation.html}{\path{complex-valued-neural-networks.rtfd.io/en/latest/activations/mvn_activation.html}}) such as multi-valued neuron (MVN) activation function \cite{naum1972generalization, aizenberg_multivalued_1973} or Georgiou CDBP \cite{georgiou1992complex}.

\subsection{Complex Rectified Linear Unit (ReLU)}

Normally, $\activation_\phi$ is left as a linear mapping \cite{kuroe2003activation, hirose2012complex}. Under this condition, using \acrfull{relu} activation function for $\activation_r$ has a limited interest since the latter makes $\activation_B$ converge to a linear function, limiting Type-B \acrshort{relu} usage.
Nevertheless, \acrshort{relu} has increased in popularity over the others as it has proved to learn several times faster than equivalents with saturating neurons \cite{krizhevsky2012imagenet}. 
Consequently, probably the most common complex-valued activation function is Type-A \acrshort{relu} activation function more often defined as Complex-\acrshort{relu} or $\mathbb{C}$\acrshort{relu} \cite{trabelsi2017deep, cao2019pixel}. 

However, several other \acrshort{relu} adaptations to the complex domain were defined throughout the bibliography as z\acrshort{relu} \cite{guberman2016complex}, defined as
\begin{equation}
    \textrm{zReLU}(z) = 
            \begin{cases}
                z & \textrm{if } 0 < \arg(z) < \pi / 2 \\
                0 & \textrm{if otherwise}
             \end{cases} \, ,
\end{equation}
letting the output as the input only if both real and imaginary parts are positive. 
Another popular adaptation is mod\acrshort{relu} \cite{arjovsky2016unitary}, defined as
\begin{equation}
    \textrm{modReLU}(z) = 
        \begin{cases}
                \textrm{ReLU}\left( |z| + b \right) \displaystyle\frac{z}{ |z|} 
                & \textrm{if } | z |  \geq b \\
                0 & \textrm{if otherwise}
            \end{cases} \, ,
\end{equation}
where $b$ is an adaptable parameter defining a radius along which the output of the function is 0. This function provides a point-wise non-linearity that affects only the absolute value of a complex number. 
Another extension of \acrshort{relu}, is the complex cardioid proposed by \cite{virtue2017better} 
\begin{equation}
    \activation(z) = \frac{\left(1 + \cos{(\arg(z))}\right) z}{2} \, .
\end{equation}
This function maintains the phase information while attenuating the magnitude based on the phase itself.

These last three activation functions (cardioid, z\acrshort{relu} and mod\acrshort{relu}) were analyzed and compared against each other in \cite{scardapane2018complex}. 

The discussed variants were implemented in the toolbox documented as usual in \href{https://complex-valued-neural-networks.readthedocs.io/en/latest/activations/relu.html}{\path{complex-valued-neural-networks.rtfd.io/en/latest/activations/relu.html}}.

\subsection{Output layer activation function} \label{sec:output-func}

The image domain of the output layer depends on the set of data labels. For classification tasks, real-valued integers or binary numbers are frequently used to label each class. 
For these cases, one option would be to cast the labels to the complex domain as done in \cite{zhang2017complex}, where a transformation is done to a label $c \in \mathbb{R}$ like $T: c \to c + \im \, c$. 

The second option is to use an activation function $\activation: \mathbb{C} \to \mathbb{R}$ as the output layer. 
A popular real-valued activation used for classification tasks is the \textit{softmax} function \cite{Goodfellow-et-al-2016} (normalized exponential), which maps the magnitude to $[0, 1]$, so the image domain is homogeneous to a probability.
There are several options on how to transform this function to accept complex input and still have its image $\in [0;1]$. These options include either performing an average of the magnitudes $\activation_{\Re},\activation_{\Im}$ or $\activation_r,\activation_{\phi}$, using only one of the magnitudes like $\activation_r$ or apply the real-valued \textit{softmax} to either the addition or multiplication of $\activation_{\Re},\activation_{\Im}$ or $\activation_r,\activation_{\phi}$, between other options.
Most of these variants are implemented in the library detailed in this Chapter and documented in \href{https://complex-valued-neural-networks.readthedocs.io/en/latest/activations/real_output.html}{\path{complex-valued-neural-networks.rtfd.io/en/latest/activations/real_output.html}}.

\section{Complex-compatible Loss functions} \label{sec:loss-function}

For \acrshort{cvnn}s, the loss or cost function to minimize will have a real-valued output as one can not look for the minimum of two complex numbers. 
If the application is that of classification or semantic segmentation (as it is in all the cases of study of this work), there are a few options on what to do.

Some loss functions support this naturally. Reference \cite{hansch2010complexb} compares the performance of different type of complex input compatible loss functions.
If the loss function to be used does not support complex-valued input, a popular option is to manage this through the output activation function as explained in the previous Section \ref{sec:output-func}. 
However, a second option for non-complex-compatible loss functions such as \textit{categorical cross-entropy} is to compare both the real and imaginary parts of the prediction independently with the labels and compute the loss function as an average of both results. Reference \cite{cao2019pixel}, for example, defines a loss function as the complex average cross-entropy as:
\begin{equation}
    \loss^{ACE} = \frac{1}{2} \left[ \loss^{CCE}\left( \Re{(y)}, d \right) + \loss^{CCE}\left( \Im{(y)}, d \right) \right] \, ,
\end{equation}

where $\loss^{ACE}$ is the complex average cross-entropy, $\loss^{CCE}$ is the well-known categorical cross-entropy. $y$ is the network predicted output, and $d$ is the corresponding ground truth or desired output. For real-valued output $\loss^{ACE} = \loss^{CCE}$.
This function was implemented in the published code alongside other variants, such as multiplying each class by weight for imbalanced classes or ignoring unlabeled data. All these versions are documented in \href{https://complex-valued-neural-networks.readthedocs.io/en/latest/losses.html}{\path{complex-valued-neural-networks.rtfd.io/en/latest/losses.htm}}.

When the desired output is already complex-valued (regression tasks), more natural definitions can be used, such as proposed by \cite{bassey2021survey}, where the loss is defined as 
\begin{equation}
    \loss = \frac{1}{2} \sum_k e_k \conjugate{e_k} \, ,
\end{equation}
where $e_k(y_k, d_k)$ is a complex error computation of $y_k$ and $d_k$ such as a subtraction.

\section{Complex Batch Normalization} \label{sec:bn}

The complex \acrfull{bn} was adapted from the real-valued \acrshort{bn} technique by Reference \cite{trabelsi2017deep}.
For normalizing a complex vector, we will treat the problem as a 2D vector instead of working on the complex domain so that $z = a + \im \, b \in \mathbb{C} \longrightarrow \matr{x} = (a, b) \in \mathbb{R}^{2}$. 

To normalize a complex variable, we need to compute
\begin{equation}
    \label{eq:complex-normalization}
    \matr{o} = \hat{\matr{\Sigma}}^{-\frac{1}{2}}(\matr{x} - \hat{\matr{\mu}}) \, ,
\end{equation}
where $\matr{o}$ is the normalized output, $\hat{\matr{\mu}}$ is the mean estimate of $\mathbb{E}[\matr{x}]$, and $\hat{\matr{\Sigma}} \in \mathbb{R}^{2\times2}$ is the estimated covariance matrix of $\matr{x}$ so that
\begin{equation}
    \hat{\matr{\Sigma}} = 
    \left[ 
        {\begin{array}{cc}
            \Sigma_{rr} & \Sigma_{ri} \\
            \Sigma_{ir} & \Sigma_{ii} \\
        \end{array} }
    \right] = 
    \left[ 
        {\begin{array}{cc}
            \mathrm{Cov}(\Re(x)\Re(x)) & \mathrm{Cov}(\Re(x)\Im(x)) \\
            \mathrm{Cov}(\Im(x)\Re(x)) & \mathrm{Cov}(\Im(x)\Im(x)) \\
        \end{array} }
    \right] \, .
\end{equation}

During the batch normalization layer initialization, two variables $\matr{\Sigma'} \in \mathbb{R}^{2\times2}$ (moving variance) and $\matr{\mu'} \in \mathbb{R}^{2}$ (moving mean) are initialized. 
By default, $\matr{\Sigma'} = \mathbf{I} / \sqrt{2}$ and $\matr{\mu}'$ is initialized to zero.

During the training phase, $\hat{\matr{\Sigma}}$ and $\hat{\matr{\mu}}$ are computed on the innermost dimension of the training input batch (for multi-dimensional inputs where $z \in \mathbb{C}^N \to x \in \mathbb{R}^{N \times 2}$). The output of the layer is then computed as in Equation \ref{eq:complex-normalization}. The moving variance and moving mean are iteratively updated using the following rule:
\begin{align}
    \label{eq:bn-trianing-rule}
    \matr{\mu'}_{k+1} &=  m \, \matr{\mu'}_{k} + (1 - m) \, \hat{\matr{\mu}}_{k} \\
    \matr{\Sigma'}_{k+1} &=  m \, \matr{\Sigma'}_{k} + (1 - m) \, \hat{\matr{\Sigma}}_{k} \, ,
\end{align}
where $m$ is the momentum, a constant parameter set to $0.99$ by default.

During the inference phase, that is, for example, when performing a prediction, no variance nor average is computed. The output is directly calculated using the moving variance and moving average as
\begin{equation}
    \label{eq:preditcion-out}
    \hat{\matr{x}} = \matr{\Sigma'}^{-\frac{1}{2}}(\matr{x} - \matr{\mu'}) \, .
\end{equation}

Analogously to the real-valued batch normalization, it is possible to shift and scale the output by using the trainable parameters $\matr{\beta}$ and $\boldsymbol{\Gamma}$. In this case, the output $\matr{o}$ for both the training and prediction phase will be changed to $\hat{\matr{o}} = \boldsymbol{\Gamma}\, \matr{o} + \matr{\beta}$. By default, $\matr{\beta}$ is initialized to $(0, 0)^T \in \mathbb{R}^2$ and 
\begin{equation}
    \boldsymbol{\Gamma} = \left(
        {\begin{array}{cc}
            1 / \sqrt{2} & 0 \\
            0 & 1 / \sqrt{2} \\
        \end{array} }
    \right) \, .
\end{equation}

\section{Complex Random Initialization} \label{sec:complex-init}

If we are to blindly apply any well-known random initialization algorithm to both real and imaginary parts of each trainable parameter independently, we might lose the special properties of the used initialization. This is the case, for example, for Glorot, also known as Xavier, initializer \cite{glorot2010understanding}. 

Assuming that:

\begin{itemize}
    \item The input features have the same variance $\mathrm{Var}\left[\layerout_{i}^{(0)}\right] \triangleq \mathrm{Var}\left[\layerout^{(0)}\right], \forall i \in [|1;N_0|]$ and have zero mean (can be adjusted by the bias input).
    \item All the weights are statistically centered, and there is no correlation between real and imaginary parts. 
    \item The weights at layer $l$ share the same variance $\mathrm{Var}\left[\weigth_{i,j}^{(l)}\right] \triangleq \mathrm{Var}\left[\weigth^{(l)}\right], \forall (i, j) \in [|1;N_{l+1}|] \times [|1;N_{l}|]$ and are statistically independent of the others layer weights and of inputs $\layerout^{(0)}$. 
    \item We are working on the linear part of the activation function. Therefore, $\activation(z) \approx z$, which is the same as saying that $\activation(z, \conjugate{z}) \approx z$. The partial derivatives will then be
\end{itemize}

\begin{equation}
\begin{cases}
    \displaystyle \frac{\partial \activation}{\partial z} \approx 1 \\
    \displaystyle \frac{\partial \activation}{\partial \conjugate{z}} \approx 0
\end{cases}
\end{equation}

so that $\displaystyle \frac{\partial \activation \left( \beforeact_n^{(l)} \right)}{\partial \beforeact_n^{(l)}} \approx 1$ and  $\displaystyle \frac{\partial \conjugate{\activation \left( \beforeact_n^{(l)} \right)}}{\partial \beforeact_n^{(l)}} \approx 0$, with $ \beforeact_n^{(l)}$ defined in Section \ref{sec:notation}. This is an acceptable assumption when working with logistic sigmoid or \acrshort{tanh} activation functions.

Using the notation of Section \ref{sec:notation}, for a dense feed-forward neural network with a bias initialized to zero (as is often the case), each neuron at hidden layer $l$ is expressed as
\begin{equation}
    X_n^{(l)} \triangleq \activation \left( \beforeact_n^{(l)} \right) =  \activation \left( \sum_{m=1}^{N_{l-1}} \weigth^{(l)}_{nm} X_m^{(l-1)}\right), \forall n \in [|1;N_{l}|]. \label{eq:forward-nobias}
\end{equation}

Since $\activation$ is working on the linear part, from \eqref{eq:forward-nobias} we get that
\begin{equation}
    \mathrm{Var}\left[X_n^{(l)}\right] = \mathrm{Var}\left[\sum_{m=1}^{N_{l-1}}\weigth^{(l)}_{nm} X_m^{(l-1)}\right] \, ,
\end{equation}
where $X_m^{(l-1)}$ is a combination of $\weigth^{(k)}, 1\leq k \leq l-1$ and $x^{(0)}$, so it is independent of $\weigth^{(l)}$ which leads to
\begin{equation}
    \mathrm{Var}\left[X_n^{(l)}\right] = \sum_{m=1}^{N_{l-1}}\mathrm{Var}\left[\weigth^{(l)}_{nm}\right] \mathrm{Var}\left[X_m^{(l-1)}\right] \, , 
\end{equation}
As the weights share the same variance at each layer,
\begin{align}
     \mathrm{Var}\left[X_n^{(l)}\right] & = \mathrm{Var}\left[\weigth^{(l)}\right] \sum_{m=1}^{N_{l-1}} \mathrm{Var}\left[X_m^{(l-1)}\right] \, , \nonumber \\ 
      & = \mathrm{Var}\left[\weigth^{(l)}\right] \mathrm{Var}\left[\weigth^{(l-1)}\right] \sum_{m=1}^{N_{l-1}} \sum_{p=1}^{N_{l-2}} \mathrm{Var}\left[X_p^{(l-2)}\right] \nonumber \, ,\\
      & = N_{l-1} \mathrm{Var}\left[\weigth^{(l)}\right] \mathrm{Var}\left[\weigth^{(l-1)}\right]  \sum_{p=1}^{N_{l-2}} \mathrm{Var}\left[X_p^{(l-2)}\right]\,. \label{eq:variancerecursion}
\end{align}
We can now obtain the variance of $\layerout_n^{(l)}$ as a function of $x^{(0)}$ by applying Equation \eqref{eq:variancerecursion} recursively and assuming $\layerout_{n}^{(0)}, n=1,\ldots,N_0$ sharing the same variance,
\begin{equation}
    \mathrm{Var}\left[X_n^{(l)}\right] = \mathrm{Var}\left[x^{(0)}\right]  \prod_{m = 1}^{l} N_{m-1} \mathrm{Var}\left[\weigth^{(m)}\right]\,. \label{eq:var_neurons}
\end{equation}
From a forward-propagation point of view, to keep a constant flow of information, then 
\begin{equation}
     \mathrm{Var}\left[X_n^{(l)}\right] = \mathrm{Var}\left[X_n^{(l')}\right]\,, \forall 1 \leq l < l' \leq N \label{eq:forwardflow}
\end{equation}
which implies that $\displaystyle N_{m-1} \mathrm{Var}\left[\weigth^{(m)}\right] = 1 \,, \forall 1 \leq m \leq N$.

On the other hand,
\begin{align}
    \frac{\partial \loss}{\partial \beforeact_n^{(l)}} &= \sum_{k=1}^{N_{l+1}}  \frac{\partial \loss}{\partial \beforeact_k^{(l+1)}} \frac{\partial \beforeact_k^{(l+1)}}{\partial X_n^{(l)}} \frac{\partial X_n^{(l)}}{\partial \beforeact_n^{(l)}} +  \frac{\partial \loss}{\partial \conjugate{\beforeact_k^{(l+1)}}} \frac{\partial \conjugate{\beforeact_k^{(l+1)}}}{\partial X_n^{(l)}} \frac{\partial X_n^{(l)}}{\partial \beforeact_n^{(l)}} + R\, , \nonumber \\
    &= \sum_{k=1}^{N_{l+1}}  \frac{\partial \loss}{\partial \beforeact_n^{(l+1)}}  \omega_{k,n}^{(l+1)}\,. \label{eq:lossderivative_wrtneurons}
\end{align}
where R is the remaining terms depending on $\displaystyle \frac{\partial \conjugate{X_n^{(l)}}}{\partial \beforeact_n^{(l)}}$, which is assumed to be $0$ because the activation function is working in the linear regime at the initialization, i.e.  $\displaystyle \frac{\partial X_n^{(l)}}{\partial \beforeact_n^{(l)}} \approx 1$ and  $\displaystyle \frac{\partial \conjugate{X_n^{(l)}}}{\partial \beforeact_n^{(l)}} \approx 0$. The last equality is held since $\displaystyle \frac{\partial \beforeact_k^{(l+1)}}{\partial X_n^{(l)}} = \omega_{k,n}^{(l+1)}$ and $\frac{\partial \beforeact_k^{(l+1)}}{\partial \conjugate{X_n^{(l)}}} = 0$ (see  \eqref{eq:complex-partial-en-xi-generic} in Appendix for the general case).

Assuming the loss variation w.r.t. the output neuron is statistically independent of any weights at any layers, then we can deduce recursively from \eqref{eq:lossderivative_wrtneurons},
\begin{equation}
    \mathrm{Var}\left[\frac{\partial \loss}{\partial \beforeact_n^{(l)}} \right] =  \mathrm{Var}\left[\frac{\partial \loss}{\partial \beforeact_n^{(L)}} \right] \prod_{m=l+1}^{L} N_{m} \mathrm{Var}\left[\weigth^{(m)} \right]\,.  \label{eq:lossderivative_wrtlastneurons}
\end{equation}
From a back-propagation point of view, we want to keep a constant learning flow:
\begin{equation}
    \mathrm{Var}\left[\frac{\partial \loss}{\partial \beforeact_n^{(l)}}\right] = \mathrm{Var}\left[\frac{\partial \loss}{\partial \beforeact_n^{(l')}}\right],\, \forall 1 \leq l < l' \leq N, \label{eq:backfowardflow}
\end{equation}
which implies $N_{m} \mathrm{Var}\left[\weigth^{(m)}\right] = 1, \, \forall 1 \leq m \leq N$.

Conditions \eqref{eq:forwardflow} and \eqref{eq:backfowardflow} are not possible to be satisfied at the same time (unless $N_l = N_{l+1}$, $\forall 1 \leq l <N$ , meaning all layers should have the same width) for what Reference \cite{glorot2010understanding} proposes the following trade-of
\begin{equation}
    \label{eq:glorot-compromise}
    \mathrm{Var}\left[\weigth^{(l)}\right] = \frac{2}{\hiddensize{l} + \hiddensize{l+1}}\, , \forall 1 \leq l < N\,.
\end{equation}




If the weight initialization is a uniform distribution $\sim U$, for the real-valued case, the initialization that has the variance stated on Equation \ref{eq:glorot-compromise} is:
\begin{equation}
    \label{eq:glorot-real-init}
    \weigth^{(l)} \sim U\left[ -\frac{\sqrt{6}}{\sqrt{\hiddensize{l} + \hiddensize{l+1}}}, \frac{\sqrt{6}}{\sqrt{\hiddensize{l} + \hiddensize{l+1}}} \right]\, .
\end{equation}

For a complex variable with no correlation between real and imaginary parts, the variance is defined as:
\begin{equation}
    \mathrm{Var}\left[\weigth^{(l)}\right] = \mathrm{Var}\left[\Re\left(\weigth^{(l)}\right)\right] + \mathrm{Var}\left[\Im\left(\weigth^{(l)}\right)\right]\, ,
\end{equation}
it is therefore logical to choose both variances $\mathrm{Var}\left[\Re\left(\weigth^{(l)}\right)\right]$ and $\mathrm{Var}\left[\Im\left(\weigth^{(l)}\right)\right]$ to be equal:
\begin{equation}
    \mathrm{Var}\left[\Re\left(\weigth^{(l)}\right)\right] = \mathrm{Var}\left[\Im\left(\weigth^{(l)}\right)\right] = \frac{1}{\hiddensize{l} + \hiddensize{l+1}}\, .
\end{equation}

With this definition, the complex variable could be initialized as:
\begin{equation}
    \label{eq:glorot-complex-cartesian-init}
    \Re\left(\weigth^{(l)}\right) = \Im\left(\weigth^{(l)}\right) \sim U\left[ -\frac{\sqrt{3}}{\sqrt{\hiddensize{l} + \hiddensize{l+1}}}, \frac{\sqrt{3}}{\sqrt{\hiddensize{l} + \hiddensize{l+1}}} \right]\, .
\end{equation}

By comparing \eqref{eq:glorot-real-init} with \eqref{eq:glorot-complex-cartesian-init} it is concluded that to correctly implement a Glorot initialization, one should divide the real and imaginary part of the complex weight by $\sqrt{2}$. 

It is also possible to define the initialization technique from a polar perspective. The variance definition is
\begin{equation}
    \label{eq:variance-definition}
    \mathrm{Var}\left[\weigth^{(l)}\right] = E\left[\left| \weigth^{(l)} - E\left[ \weigth^{(l)} \right]^2 \right|\right] =  E\left[\left| \weigth^{(l)}\right|^2\right] + \left|E\left[ \weigth^{(l)} \right]\right|^2\, .
\end{equation}

By choosing the phase to be a uniform distribution between $0$ and $2\pi$ and knowing the absolute is always positive, then $E\left[ \weigth^{(l)} \right] = 0$ and Equation \ref{eq:variance-definition} can be simplified to:
\begin{equation}
    \mathrm{Var}\left[\weigth^{(l)}\right] =  E\left[\left| \weigth^{(l)}\right|^2\right]\, .
\end{equation}

It will therefore suffice to choose any random initialization, such as, for example, a Rayleigh distribution \cite{rayleigh1880xii}, for $\left| \weigth^{(l)}\right| = \rho \in \mathbb{R}_0^+$ so that
\begin{equation}
    E\left[ \rho^2 \right] = \frac{2}{\hiddensize{l} + \hiddensize{l+1}}  \, .  
\end{equation}

\subsection{Impact of complex-initialization equation application}

A simulation was done for a complex multi-layer network with four hidden layers of size 128, 64, 32 and 16, respectively, with a logistic sigmoid activation function to test the impact of this constant division by $\sqrt{2}$ on a signal classification task. One-hundred and fifty epochs were done with one thousand runs of each model to obtain statistical results.

The task consisted in classifying different signals used in radar applications. Temporal and time-frequency representations of each signal are shown in Figures \ref{fig:temporal} and \ref{fig:spectrogram} respectively. The generated signals are

\begin{itemize}
    \item Sweep or chirp signal. These are signals whose frequency changes over time. These types of signals are commonly applied to radar. The chirp-generated signals were of two types, either linear chirp, where the frequency changed linearly over time or S-shaped, whose frequency variation gets faster at both the beginning and the end, forming an S-shaped spectrum as can be seen in Figure \ref{fig:spectrogram}.
    \item \acrfull{psk} modulated signals, a digital modulation process that conveys data by changing the phase of a constant frequency reference signal. These signals were BPSK (2-phase states) and QPSK (4-phase states)
    \item \acrfull{qam} signals, which are a combination of amplitude and phase modulation. These signals were 16QAM (4 phase and amplitude states) and 64QAM (8 phase and amplitude states).
\end{itemize}

A noise signal (null), without any signal of interest, was also used, making a total of 7 different classes.

\begin{figure}
    \centering
    \includegraphics[width=\textwidth]{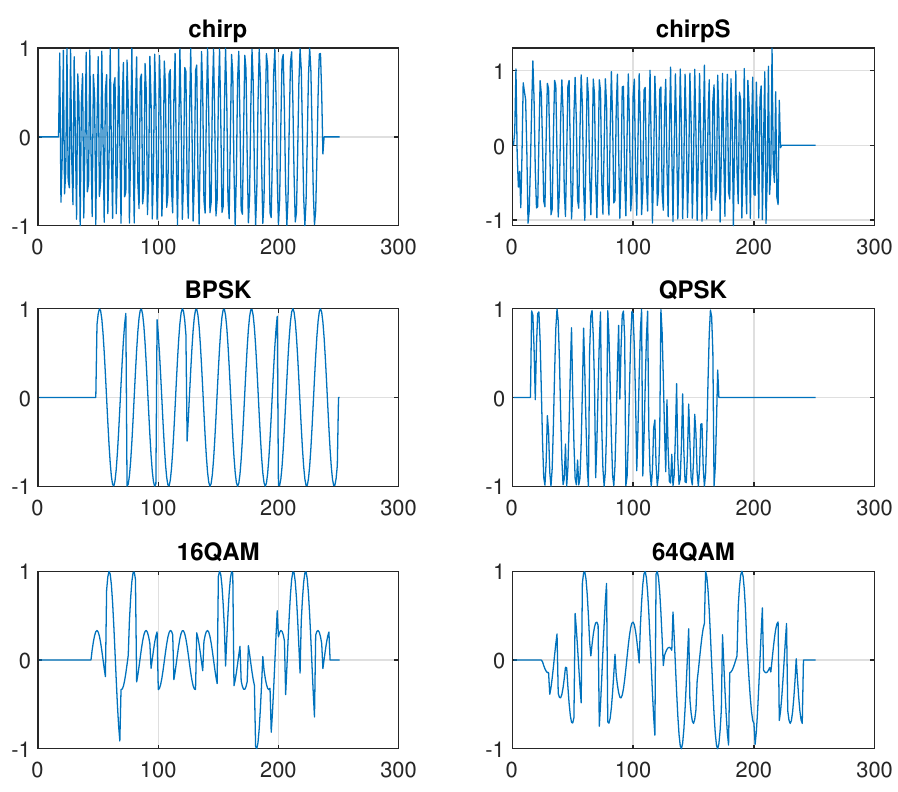}
    \caption{Temporal amplitude examples of used signals \cite{gilles2018techreport}.}
    \label{fig:temporal}
\end{figure}

\begin{figure}
    \centering
    \includegraphics[width=\textwidth]{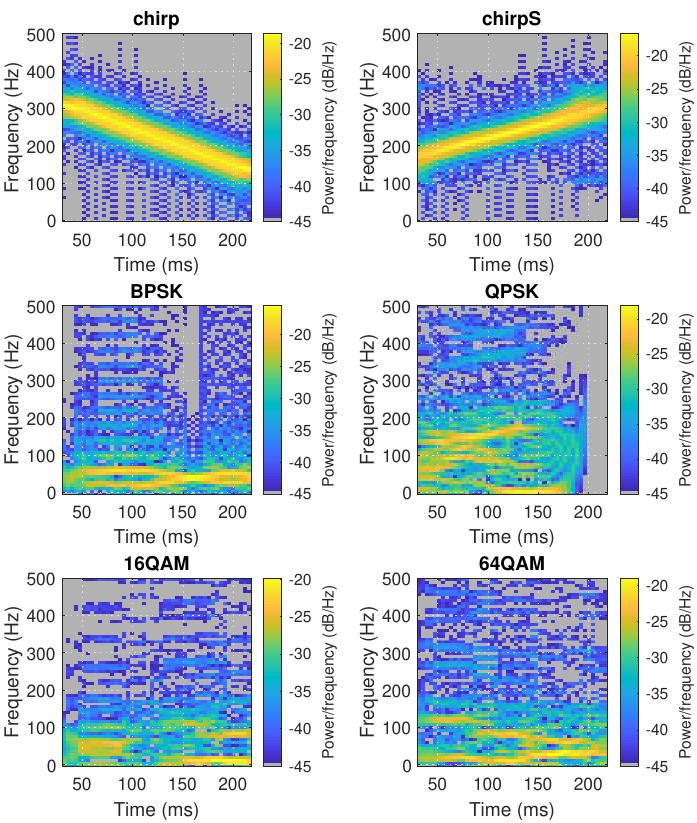}
    \caption{Spectrogram examples of used signals \cite{gilles2018techreport}.}
    \label{fig:spectrogram}
\end{figure}

Instead of using measured signals, and with the goal of facilitating the studies, the signals were randomly generated, which also allowed having the ground truth. These generated signals had the following properties:

\begin{itemize}
    \item 256 samples per signal
    \item Peak-to-peak amplitude of 1
    \item Chirp signals with frequencies from 0.05 to 0.45 times the sample frequency
    \item Number of moments between 8 and 64 for the codes BPSK, QPSK, 16QAM and 64QAM
\end{itemize}

Thermal noise was added to each signal and was transformed to the complex domain using the Hilbert Transform, a popular transformation in signal processing applications \cite{hahnhilbert}, which provides an analytic mapping of a real-valued function to the complex plane.

The Hilbert transpose has its origins in 1902 when the English mathematician Godfrey Harold Hardy (1877 – 1947) introduced a transformation that consisted in the convolution of a real function $f(s)$  \cite{hardy1902theory, hardy1909theory}, with the Cauchy kernel $\displaystyle 1/\pi(t-s)$ which, being an improper integral, must be defined in terms of its principal value (p.v.) \cite{king1997encyclopedia},
\begin{align}
    \label{eq:hilbert-transform}
    \mathcal{H}(f)(t) &= \frac{1}{\pi}\text{p.v.}\int_{-\infty}^{+\infty} \frac{f(s)}{t-s}ds
    &= \frac{1}{\pi}\lim_{\varepsilon \to 0}\int_{\varepsilon}^{+\infty} \frac{f(t-s) - f(t+s)}{s}ds \, .
\end{align}

One of the most important properties of this transformation is that its repeated application allows for the recovery of the original function, with only a change of sign, that is,
\begin{equation}
    g(t) = \mathcal{H}(f)(t) \Leftrightarrow f(t) = -\mathcal{H}(g)(t) \, .
\end{equation}

The functions $f$ and $g$ that satisfy this relation are called Hilbert transform pairs, in honor of David Hilbert, who first studied them in 1904 \cite{hilbert1912basics}. In fact, it is for this reason that in 1924 \cite{hardy1928notes, hardy1928notesb}, Hardy graciously proposed calling transformation \eqref{eq:hilbert-transform} as Hilbert Transform.

Some examples of Hilbert transform pairs are shown in Table \ref{tab:ht-examples}.

\begin{table}[ht]
    \centering
    \begin{tabular}{c c}
        $f(t)$ & $g(t) = \mathcal{H}(f)(t)$ \\
        \midrule
        $\sin(t)$ & $\cos(t)$ \\
        $1 / \left(t^2 + 1 \right)$ & $t / \left(t^2 + 1 \right)$ \\
        $\sin(t) / t$ & $\left[ 1 - \cos(t) \right] / t$ \\
        $\delta(t)$ & $1 / \pi t$ \\
        \bottomrule
    \end{tabular}
    \caption{Examples of Hilbert transform pairs}
    \label{tab:ht-examples}
\end{table}

Considering the definition of the Fourier transform of an absolutely integrable real function $f(s)$,
\begin{equation}
    \mathcal{F}(f)(t) \, .
\end{equation}
It can be shown that \cite{king1997encyclopedia}
\begin{equation}
    \mathcal{F}\left( \mathcal{H}(f)(t) \right) = - \im \,\text{sgn}(t)\mathcal{F}(f)(t) \, .
\end{equation}
This relation provides an effective way to evaluate the Hilbert transform
\begin{equation}
    \mathcal{H}(f) = - \im \,\mathcal{F}^{-1} \left[ \text{sgn} \mathcal{F}(f)(t) \right] \, ,
\end{equation}
avoiding the issue of dealing with the singular structure of the Cauchy kernel.

One of the most important properties of the Hilbert transform, at least in reference to this thesis, is that the real and imaginary parts of a function $h(z)$ that is analytic in the upper half of the complex plane are Hilbert transform pairs. That is to say that
\begin{equation}
    \Im(h) = \mathcal{H}(\Re(h)) \, .
\end{equation}

In this way, the Hilbert transform provides a simple method of performing the analytic continuation to the complex plane of a real function $f(x)$ defined on the real axis, defining $h(z) = f(z) + \im \,g(z)$ with $g(z) = \mathcal{H}(f)$. This property of the Hilbert transform was independently discovered by Ralph Kronig \cite{kronig1926theory} (1904 – 1995), and Hans Kramers \cite{kramers1927diffusion} (1894 - 1952) in 1926 in relation to the response function of physical systems, known as the Kramers-Kronig relation. At the same time, it began to be used in circuit analysis \cite{carson1926electric} in relation to the real and imaginary parts of the complex impedance. Through the work of pioneers such as the Nobel prize winner Dennis Gabor \cite{gabor1946theory} (1900 – 1979), its application in modern signal processing is wide and varied \cite{hahnhilbert}.

The real and imaginary weights where initialized as described in \eqref{eq:glorot-real-init} (The definition for real-valued weights, which is equivalent to multiplying the limits of Equation \ref{eq:glorot-complex-cartesian-init} by $\sqrt{2}$) and \eqref{eq:glorot-complex-cartesian-init} (the original case for complex-valued weights) to compare them. An initialization that divided the limits of Equation \ref{eq:glorot-complex-cartesian-init} by two was also used to assert that smaller values will not produce a superior result either. 

\begin{figure}[ht]
	\centering
\begin{tikzpicture}

\definecolor{color0}{rgb}{0.194607843137255,0.453431372549019,0.632843137254902}
\definecolor{color1}{rgb}{0.881862745098039,0.505392156862745,0.173039215686275}
\definecolor{color2}{rgb}{0.229411764705882,0.570588235294118,0.229411764705882}

\begin{axis}[
tick align=outside,
tick pos=left,
x grid style={white!69.0196078431373!black},
xmin=-0.5, xmax=2.5,
xtick style={color=black},
xtick={0,1,2},
xticklabels={$\cdot \sqrt{2}$,original, $/2$},
y grid style={white!69.0196078431373!black},
ylabel={test accuracy},
ymin=0.486162946428571, ymax=0.560265625,
ytick style={color=black},
ymajorgrids,
ytick={0.48,0.49,0.5,0.51,0.52,0.53,0.54,0.55,0.56,0.57},
yticklabels={0.48,0.49,0.50,0.51,0.52,0.53,0.54,0.55,0.56,0.57}
]
\path [draw=color0, fill=color0, opacity=0.3, semithick]
(axis cs:-0.4,0.5228125)
--(axis cs:0.4,0.5228125)
--(axis cs:0.4,0.527633928571429)
--(axis cs:-0.4,0.527633928571429)
--(axis cs:-0.4,0.5228125)
--cycle;
\path [draw=color1, fill=color1, opacity=0.3, semithick]
(axis cs:0.6,0.546216517857143)
--(axis cs:1.4,0.546216517857143)
--(axis cs:1.4,0.5519140625)
--(axis cs:0.6,0.5519140625)
--(axis cs:0.6,0.546216517857143)
--cycle;
\path [draw=color2, fill=color2, opacity=0.3, semithick]
(axis cs:1.6,0.512857142857143)
--(axis cs:2.4,0.512857142857143)
--(axis cs:2.4,0.525150669642857)
--(axis cs:1.6,0.525150669642857)
--(axis cs:1.6,0.512857142857143)
--cycle;
\addplot [semithick, color0]
table {%
0 0.5228125
0 0.518973214285714
};
\addplot [semithick, color0]
table {%
0 0.527633928571429
0 0.531919642857143
};
\addplot [semithick, color0]
table {%
-0.2 0.518973214285714
0.2 0.518973214285714
};
\addplot [semithick, color0]
table {%
-0.2 0.531919642857143
0.2 0.531919642857143
};
\addplot [semithick, color1]
table {%
1 0.546216517857143
1 0.541071428571429
};
\addplot [semithick, color1]
table {%
1 0.5519140625
1 0.556897321428571
};
\addplot [semithick, color1]
table {%
0.8 0.541071428571429
1.2 0.541071428571429
};
\addplot [semithick, color1]
table {%
0.8 0.556897321428571
1.2 0.556897321428571
};
\addplot [semithick, color2]
table {%
2 0.512857142857143
2 0.494888392857143
};
\addplot [semithick, color2]
table {%
2 0.525150669642857
2 0.53375
};
\addplot [semithick, color2]
table {%
1.8 0.494888392857143
2.2 0.494888392857143
};
\addplot [semithick, color2]
table {%
1.8 0.53375
2.2 0.53375
};
\addplot [color2, mark=diamond*, mark size=2.5, mark options={solid}, only marks]
table {%
2 0.48953125
2 0.490022321428571
2 0.492767857142857
};
\addplot [semithick, color0]
table {%
-0.4 0.524921875
0.4 0.524921875
};
\addplot [semithick, color1]
table {%
0.6 0.549341517857143
1.4 0.549341517857143
};
\addplot [semithick, color2]
table {%
1.6 0.522957589285714
2.4 0.522957589285714
};
\end{axis}

\end{tikzpicture}
	\caption{Comparison of Glorot Uniform initialization scaled by different values.}
	\label{fig:scaled-glorot-comparison}
\end{figure}
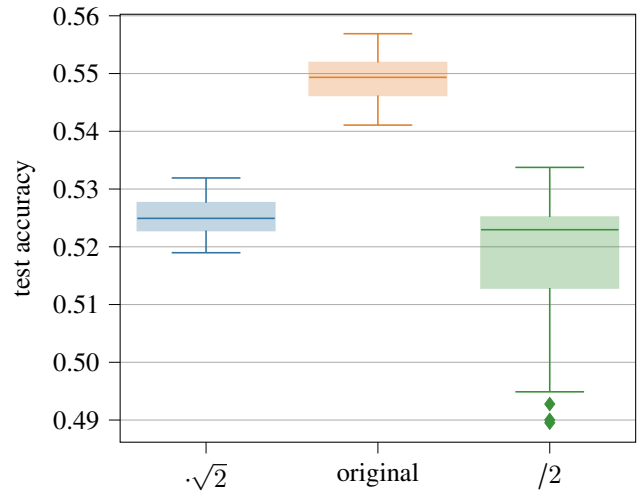

The results shown in Figure \ref{fig:scaled-glorot-comparison} prove the importance of the correct adaptation of Glorot initialization to complex numbers and how failing to do so will impact its performance negatively.

\subsection{Experiment on different trade-offs}

Complex numbers enable choosing different trade-offs than the one chosen by \cite{glorot2010understanding} (Equation \ref{eq:glorot-compromise}), for example, the following trade-off can also be chosen:
\begin{equation}
\begin{aligned}
    \label{eq:proposed-compromise}
    \mathrm{Var}\left[\Re\left(\weigth^{(l)}\right)\right] = \frac{1}{2 \hiddensize{l}} \, ,\\ \mathrm{Var}\left[\Im\left(\weigth^{(l)}\right)\right] = \frac{1}{2 \hiddensize{l+1}} \; ,
\end{aligned}
\end{equation}
between other options.

In a similar manner, as we did with Glorot (Xavier) initialization, the He weight initialization described in \cite{he2015delving} can be deduced for complex numbers. 
the same dataset as the one used in the experiment of Figure \ref{fig:scaled-glorot-comparison} was done using Glorot Uniform ($GU$), Glorot Normal ($GN$), He Normal ($HN$), He Uniform ($HU$), and Glorot Uniform using the trade-off defined in \eqref{eq:proposed-compromise} ($GU_C$). 

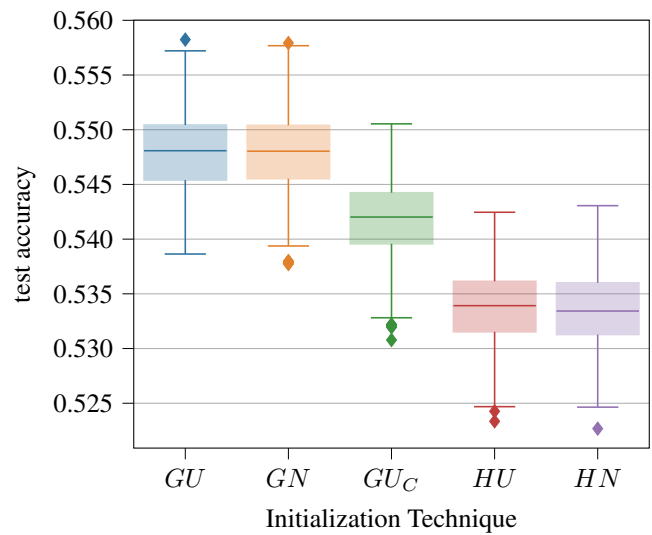
\begin{figure}[ht]
	\centering
\begin{tikzpicture}

\definecolor{color0}{rgb}{0.194607843137255,0.453431372549019,0.632843137254902}
\definecolor{color1}{rgb}{0.881862745098039,0.505392156862745,0.173039215686275}
\definecolor{color2}{rgb}{0.229411764705882,0.570588235294118,0.229411764705882}
\definecolor{color3}{rgb}{0.75343137254902,0.238725490196078,0.241666666666667}
\definecolor{color4}{rgb}{0.578431372549019,0.446078431372549,0.699019607843137}

\begin{axis}[
tick align=outside,
tick pos=left,
x grid style={white!69.0196078431373!black},
xlabel={Initialization Technique},
ymajorgrids,
xmin=-0.5, xmax=4.5,
xtick style={color=black},
xtick={0,1,2,3,4},
xticklabels={$GU$,$GN$,$GU_C$,$HU$,$HN$},
y grid style={white!69.0196078431373!black},
ylabel={test accuracy},
ymin=0.520900669642857, ymax=0.560014508928571,
ytick style={color=black},
ytick={0.52,0.525,0.53,0.535,0.54,0.545,0.55,0.555,0.56,0.565},
yticklabels={0.520,0.525,0.530,0.535,0.540,0.545,0.550,0.555,0.560,0.565}
]
\path [draw=color0, fill=color0, opacity=0.3, semithick]
(axis cs:-0.4,0.545401785714286)
--(axis cs:0.4,0.545401785714286)
--(axis cs:0.4,0.5504296875)
--(axis cs:-0.4,0.5504296875)
--(axis cs:-0.4,0.545401785714286)
--cycle;
\path [draw=color1, fill=color1, opacity=0.3, semithick]
(axis cs:0.6,0.545518973214286)
--(axis cs:1.4,0.545518973214286)
--(axis cs:1.4,0.550401785714286)
--(axis cs:0.6,0.550401785714286)
--(axis cs:0.6,0.545518973214286)
--cycle;
\path [draw=color2, fill=color2, opacity=0.3, semithick]
(axis cs:1.6,0.5395703125)
--(axis cs:2.4,0.5395703125)
--(axis cs:2.4,0.544241071428571)
--(axis cs:1.6,0.544241071428571)
--(axis cs:1.6,0.5395703125)
--cycle;
\path [draw=color3, fill=color3, opacity=0.3, semithick]
(axis cs:2.6,0.531529017857143)
--(axis cs:3.4,0.531529017857143)
--(axis cs:3.4,0.536143973214286)
--(axis cs:2.6,0.536143973214286)
--(axis cs:2.6,0.531529017857143)
--cycle;
\path [draw=color4, fill=color4, opacity=0.3, semithick]
(axis cs:3.6,0.531277901785714)
--(axis cs:4.4,0.531277901785714)
--(axis cs:4.4,0.536004464285714)
--(axis cs:3.6,0.536004464285714)
--(axis cs:3.6,0.531277901785714)
--cycle;
\addplot [semithick, color0]
table {%
0 0.545401785714286
0 0.538638392857143
};
\addplot [semithick, color0]
table {%
0 0.5504296875
0 0.557209821428571
};
\addplot [semithick, color0]
table {%
-0.2 0.538638392857143
0.2 0.538638392857143
};
\addplot [semithick, color0]
table {%
-0.2 0.557209821428571
0.2 0.557209821428571
};
\addplot [color0, mark=diamond*, mark size=2.5, mark options={solid}, only marks]
table {%
0 0.558236607142857
};
\addplot [semithick, color1]
table {%
1 0.545518973214286
1 0.539375
};
\addplot [semithick, color1]
table {%
1 0.550401785714286
1 0.557678571428571
};
\addplot [semithick, color1]
table {%
0.8 0.539375
1.2 0.539375
};
\addplot [semithick, color1]
table {%
0.8 0.557678571428571
1.2 0.557678571428571
};
\addplot [color1, mark=diamond*, mark size=2.5, mark options={solid}, only marks]
table {%
1 0.537790178571429
1 0.537901785714286
1 0.537723214285714
1 0.537991071428571
1 0.537946428571429
1 0.557924107142857
};
\addplot [semithick, color2]
table {%
2 0.5395703125
2 0.5328125
};
\addplot [semithick, color2]
table {%
2 0.544241071428571
2 0.550535714285714
};
\addplot [semithick, color2]
table {%
1.8 0.5328125
2.2 0.5328125
};
\addplot [semithick, color2]
table {%
1.8 0.550535714285714
2.2 0.550535714285714
};
\addplot [color2, mark=diamond*, mark size=2.5, mark options={solid}, only marks]
table {%
2 0.532053571428571
2 0.53078125
2 0.5321875
2 0.532165178571429
2 0.5321875
2 0.531941964285714
};
\addplot [semithick, color3]
table {%
3 0.531529017857143
3 0.5246875
};
\addplot [semithick, color3]
table {%
3 0.536143973214286
3 0.542455357142857
};
\addplot [semithick, color3]
table {%
2.8 0.5246875
3.2 0.5246875
};
\addplot [semithick, color3]
table {%
2.8 0.542455357142857
3.2 0.542455357142857
};
\addplot [color3, mark=diamond*, mark size=2.5, mark options={solid}, only marks]
table {%
3 0.523348214285714
3 0.524263392857143
};
\addplot [semithick, color4]
table {%
4 0.531277901785714
4 0.524642857142857
};
\addplot [semithick, color4]
table {%
4 0.536004464285714
4 0.543058035714286
};
\addplot [semithick, color4]
table {%
3.8 0.524642857142857
4.2 0.524642857142857
};
\addplot [semithick, color4]
table {%
3.8 0.543058035714286
4.2 0.543058035714286
};
\addplot [color4, mark=diamond*, mark size=2.5, mark options={solid}, only marks]
table {%
4 0.522678571428571
};
\addplot [semithick, color0]
table {%
-0.4 0.548080357142857
0.4 0.548080357142857
};
\addplot [semithick, color1]
table {%
0.6 0.548035714285714
1.4 0.548035714285714
};
\addplot [semithick, color2]
table {%
1.6 0.542020089285714
2.4 0.542020089285714
};
\addplot [semithick, color3]
table {%
2.6 0.533917410714286
3.4 0.533917410714286
};
\addplot [semithick, color4]
table {%
3.6 0.533426339285714
4.4 0.533426339285714
};
\end{axis}

\end{tikzpicture}
	\caption{Initialization Technique Comparison}
	\label{fig:initialization-technique-comparison}
\end{figure}

The results of Figure \ref{fig:initialization-technique-comparison} show that the difference between using a normal or uniform distribution is negligible but that Glorot clearly outperforms He initialization. This is to be expected when using sigmoid activation function because He initialization was designed for \acrshort{relu} activation functions.
On the other hand, the trade-off proposed in \eqref{eq:proposed-compromise} actually achieves lower performances than the one proposed in \cite{glorot2010understanding} for what it was not implemented in the cvnn toolbox, however, it is possible this results is only specific to this dataset and for what further research could be done with this trade-off variant. All the other initialization techniques are documented in \href{https://complex-valued-neural-networks.readthedocs.io/en/latest/initializers.html}{\path{complex-valued-neural-networks.rtfd.io/en/latest/initializers.html}} and implemented as previously described. They can be used in standalone mode as follows

\begin{lstlisting}[language=python]
import cvnn
initializer = cvnn.initializers.GlorotUniform()
values = initializer(shape=(2, 2))                  
# Returns a complex Glorot Uniform tensor of shape (2, 2)
\end{lstlisting}

or inside a layer using an initializer object like

\begin{lstlisting}[language=python]
import cvnn
initializer = cvnn.initializers.ComplexGlorotUniform()
layer = cvnn.layers.Dense(input_size=23, output_size=45, weight_initializer=initializer)
\end{lstlisting}

or as a string listed within \href{https://github.com/NEGU93/cvnn/blob/master/cvnn/initializers.py#L274}{\lstinline{init_dispatcher}} like

\begin{lstlisting}[language=python]
import cvnn

layer = cvnn.layers.Dense(input_size=23, output_size=45, weight_initializer="ComplexGlorotUniform")
\end{lstlisting}

with \lstinline{init_dispatcher} being

\begin{lstlisting}[language=python]
init_dispatcher = {
    "ComplexGlorotUniform": ComplexGlorotUniform,
    "ComplexGlorotNormal": ComplexGlorotNormal,
    "ComplexHeUniform": ComplexHeUniform,
    "ComplexHeNormal": ComplexHeNormal
}
\end{lstlisting}

\section{Performance on real-valued data}

Using the same signals of previous Sections, some experiments were perform comparing \acrfull{cv-mlp} against a \acrfull{rv-mlp}. 
16000 Chirp signals were created, 8000 linear and 8000 S-shaped. $80\%$ was used for training, $10\%$ for validation and the remaining $10\%$ for testing.
Two models (one \acrshort{cv-mlp} and one \acrshort{rv-mlp}) were designed and dimensioned as shown in Table \ref{tab:mlp-dimension}.

\begin{table}[ht]
    \centering
    \begin{tabular}{l c c}
         & \acrshort{cv-mlp} & \acrshort{rv-mlp} \\
         \midrule
         input size & 256 & 512 \\
         hidden layers activation & Type-A Selu & Selu \\
         $1^{\text{st}}$ hidden layer size & 25 & 50 \\
         $2^{\text{nd}}$ hidden layer size & 10 & 20 \\
         output activation & modulo \textit{softmax} & modulo \textit{softmax} \\
         output size & 7 & 7  \\  
         \bottomrule
    \end{tabular}
    \caption{Design of \acrshort{mlp} models}
    \label{tab:mlp-dimension}
\end{table}

We used \acrshort{sgd} as weight optimization and $50\%$ dropout for both models.
We performed 2000 iterations (1000 for each model) with 2000 epochs each.

Figure \ref{fig:mean-signal-loss} shows the mean loss per epoch of both training and validation set for \acrshort{cv-mlp} (Figure \ref{fig:cv-mlp-loss}) and \acrshort{rv-mlp} (Figure \ref{fig:rv-mlp-loss}). The Figures show that \acrshort{cv-mlp} presents less overfitting than the real-valued model.

\begin{figure}[ht]
    \centering
    \begin{subfigure}[b]{0.49\textwidth}
         \centering
         \includegraphics[width=\textwidth]{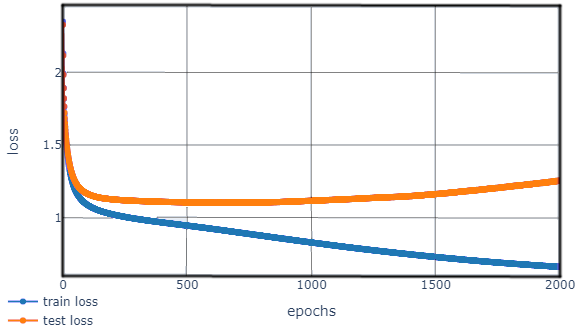}
         \caption{\acrshort{cv-mlp}}
         \label{fig:cv-mlp-loss}
    \end{subfigure}
    \hfill
    \begin{subfigure}[b]{0.49\textwidth}
         \centering
         \includegraphics[width=\textwidth]{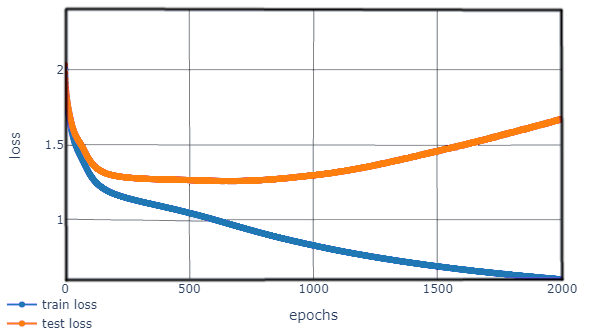}
         \caption{\acrshort{rv-mlp}}
         \label{fig:rv-mlp-loss}
    \end{subfigure}
    \caption{Mean loss evolution per epoch.}
    \label{fig:mean-signal-loss}
\end{figure}

A histogram of both models' accuracy and loss values on the test set was plotted and can be seen in Figure \ref{fig:1dt-hist}. It is clear that \acrshort{cv-mlp} outperforms \acrshort{rv-mlp} classification accuracy with around $4\%$ higher accuracy. Regarding the loss, \acrshort{rv-mlp} obtained higher variance.

\begin{figure}[ht]
    \centering
    \begin{subfigure}[b]{0.49\textwidth}
         \centering
         \includegraphics[width=\textwidth]{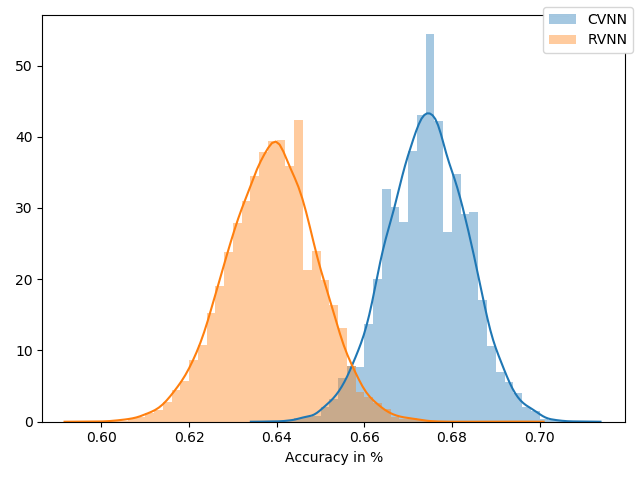}
         \caption{Accuracy}
         \label{fig:1dt-acc}
    \end{subfigure}
    \hfill
    \begin{subfigure}[b]{0.49\textwidth}
         \centering
         \includegraphics[width=\textwidth]{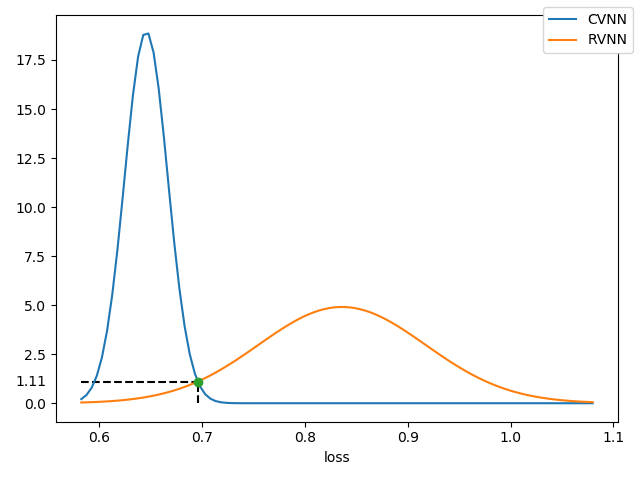}
         \caption{Loss}
         \label{fig:1dt-loss}
    \end{subfigure}
    \caption{Test set histogram metrics for binary classification on Chirp signals.}
    \label{fig:1dt-hist}
\end{figure}

Finally, the simulations were performed for all seven signals obtaining similar results as before. This time, accuracy results have a higher difference with \acrshort{cvnn} results not intersecting with the \acrshort{rvnn} results. Again, \acrshort{rv-mlp} had higher loss variance. 

\begin{figure}[ht]
    \centering
    \begin{subfigure}[b]{0.49\textwidth}
         \centering
         \includegraphics[width=\textwidth]{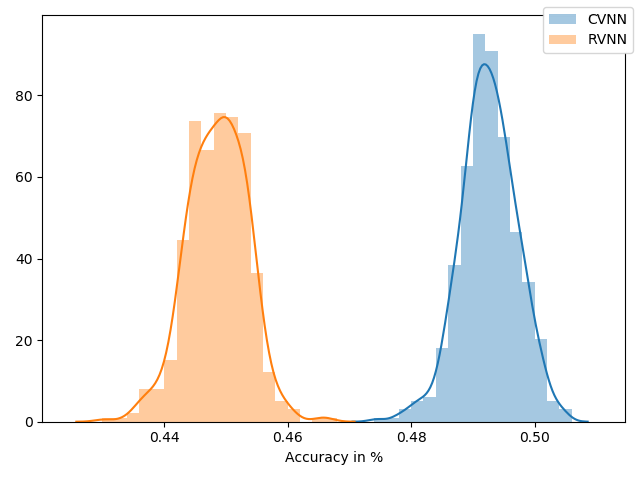}
         \caption{Accuracy}
         \label{fig:1dc-acc}
    \end{subfigure}
    \hfill
    \begin{subfigure}[b]{0.49\textwidth}
         \centering
         \includegraphics[width=\textwidth]{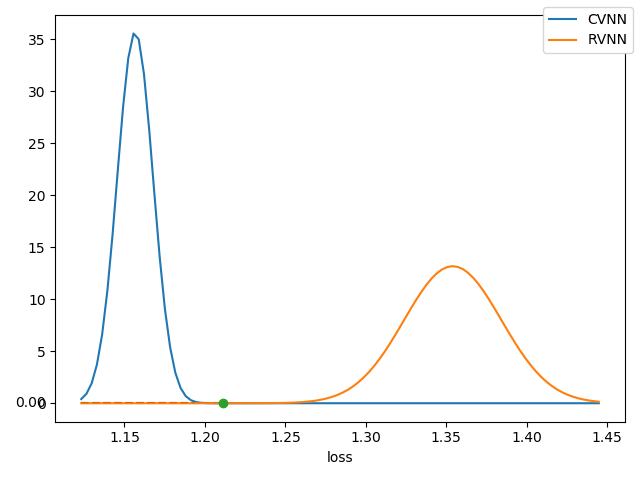}
         \caption{Loss}
         \label{fig:1dc-loss}
    \end{subfigure}
    \caption{Test set histogram metrics for multi-class classification for all 7 signals.}
    \label{fig:1dc-hist}
\end{figure}

\subsection{Discussion}

It is important to note that these networks were not optimized. The main issue is the \textit{softmax} activation function used on the absolute value of the complex output. Although it is not generally used like this, in \acrshort{cvnn}, it might be acceptable. However, for a \acrshort{rvnn}, this is unconventional and penalizes their performance greatly. Furthermore, both models are not equivalent, as described in Reference \cite{barrachina2021about}, resulting in \acrshort{rvnn} having a higher capacity, which may increase their performance but also may result in more overfitting.

For all these reasons, the simulations must be revised. However, if the general conclusions stand, this might indicate that \acrshort{cvnn} can outperform \acrshort{rvnn} even for real-valued applications when using an appropriate transformation such as the Hilbert transform.

\section{Conclusion}

In this paper, described in detail each adaptation from conventional real-valued neural networks to the complex domain in order to be able to implement \acrlong{cvnn}s. Each aspect was revised, and the appropriate mathematics was developed. With this work it should be possible to understand and implement from a basic \acrlong{cv-mlp} to a \acrlong{cv-cnn} and even \acrlong{cv-fcnn} or \acrfull{cv-unet}.

We also detail the implementation of the published library, with examples of how to use the code and with reference to the documentation to be used if needed. We showed that the library had great success in the community through its increasing popularity.

We performed simulations that verified the adaptation of the initialization technique and showed that correctly implementing this initialization is crucial for obtaining a good performance. 

Finally, we show that \acrshort{cvnn}s might be of interest even for real data by using the Hilbert transformation contrary to the work of \cite{monning2018evaluation}. The classification improved around $4\%$ when using a complex network over a real one. However, these last results should be revised as the models were not equivalent, and \acrshort{rvnn} might have been overly penalized.

\appendix


\section{Complex Identities}

For further reading and properties of complex numbers, see Reference \cite{HAYKIN2005}.

\begin{definition}[Hermitian transpose]
    Given $\matr{A} \in \mathbb{C}^{n \times m}, m,n \in \mathbb{N}^+$. The \textit{Hermitian transpose} of a matrix is defined as:
    \begin{equation}
        \matr{A}^H = \conjugate{\matr{A}^T} \, .
    \end{equation}
\end{definition}

The upper line on $\conjugate{A}$ refers to the conjugate value of a complex number, that is, a number with an equally real part and an imaginary part equal in magnitude but opposite in sign.

\begin{definition}[differential rule]
\label{def:differential-rule}
\begin{equation}
    \partial f = \frac{\partial f}{\partial z}\partial z + \frac{\partial f}{\partial \conjugate{z}}\partial \conjugate{z}\, .
\end{equation}
\end{definition}

\begin{definition}[conjugation rule]
    \label{def:conjugation-rule}
    \begin{equation}
        \begin{aligned}
        \conjugate{\left( \frac{\partial f}{\partial z}\right)} & = \frac{\partial \conjugate{f} }{\partial \conjugate{z}}\, , \\
        \conjugate{\left( \frac{\partial f}{\partial \conjugate{z}}\right)} & = \conjugate{\frac{\partial  f}{\partial z}}\, .
        \end{aligned}
    \end{equation}
    When $f : \mathbb{C} \longrightarrow \mathbb{R}$ the expression can be simplified as:
    \begin{equation}
        \begin{aligned}
        \conjugate{\left( \frac{\partial f}{\partial z}\right)} & = \frac{\partial f}{\partial \conjugate{z}} \\
        \conjugate{\left( \frac{\partial f}{\partial \conjugate{z}}\right)} & = \frac{\partial f}{\partial z} \, .
        \end{aligned}
    \end{equation}
\end{definition}

\begin{theorem}
    Given $f(z): \mathbb{C} \longrightarrow \mathbb{C}, z = x + i \,y \Rightarrow \exists u,v : \mathbb{R}^2 \longrightarrow \mathbb{R}$ such that $f(z) = u(x, y) + i \,v(x, y) $ 
    \label{the:complex-to-real-mapping}
\end{theorem}

\subsection{Complex differentiation rules}
Given $f,g : \mathbb{C} \longrightarrow \mathbb{C}$ and $c \in \mathbb{C}$:
\begin{equation}
\begin{aligned}
    (f+g)'(z_0) & = f'(z_0) + g'(z_0)\, , \\
    (f\,g)'(z_0) &= f'(z_0)\,g(z_0) + f(z_0)\,g'(z_0) \, ,\\
    \left( \frac{f}{g} \right)'(z_0) &= \frac{f'(z_0)\,g(z_0) - f(z_0)\,g'(z_0)}{g^2(z_0)}, & g(z_0) \neq 0 \, ,\\
    (c\,f)'(z_0) & = c \,f'(z_0) \, .
\end{aligned}
\end{equation}

\subsection{Properties of the conjugate}

Given $z,w \in \mathbb{C}$:
\begin{equation}
\begin{aligned}
    \conjugate{z \pm w} &= \conjugate{z} \pm \conjugate{w} \, ,\\
    \conjugate{z\,w} &= \conjugate{z} \,\conjugate{w} \, ,\\
    \conjugate{\left( \frac{z}{w} \right)} &= \frac{\conjugate{z}}{\conjugate{w}} \, ,\\
    \conjugate{z^n} &= \conjugate{z}^n, \forall n \in \mathbb{C} \, ,\\
    |z|^2 &= z \,\conjugate{z} \, ,\\
    \conjugate{z} \, z &= z \, \conjugate{z} \, ,\\
    \conjugate{\conjugate{z}} &= z \text{       (involution)}\, ,\\
    \conjugate{z} &= z \Leftrightarrow z \in \mathbb{R} \, , \\
    z^{-1} &= \frac{\conjugate{z}}{|z|^2}, \forall z \neq 0 \, .
\end{aligned}
\end{equation}

\subsection{Holomorphic Function} \label{chap:holomorphic-function}

An \textit{holomorphic} function is a complex-valued function of one or more complex variables that is, at every point of its domain, complex differentiable in a neighborhood of the point. This condition implies an \textit{holomorphic} function is Class $C^\infty$ (analytic).

\begin{definition}{}
    Given a complex function $f: \mathbb{C} \longrightarrow \mathbb{C}$ at a point $z_0$ of an open subset $\Omega \subset \mathbb{C}$ is \textit{complex-differentiable} if exists a limit such as:
    \begin{equation}
        f'(z_0) = \lim_{z \to z_0} \frac{f(z)-f(z_0)}{z-z_0}\, .
    \end{equation}
\end{definition}

As stated before, if a function is complex-differentiable at all points of $\Omega$ it is called \textit{holomorphic} \cite{monning2018evaluation}.
The relationship between real differentiability and complex differentiability is stated on the \textit{Cauchy-Riemann equations}

\begin{theorem}[Cauchy-Riemann equations]
     Given $f(x+i\,y) = u(x,y)+i\,v(x,y)$ where $u,v: \mathbb{R}^2 \longrightarrow \mathbb{R} $ real-differentiable functions, $f$ is complex-differentiable if satisfies:
    \begin{equation}
        \begin{aligned}
            \frac{\partial u}{\partial x} &= \frac{\partial v}{\partial y} \, ,\\
            - \frac{\partial u}{\partial y} &= \frac{\partial v}{\partial x} \, .
        \end{aligned}
    \end{equation}
    \label{the:cauchy-riemann}
\end{theorem}

\begin{proof}
    Given $f: \mathbb{C} \longrightarrow \mathbb{C}$: 
    \begin{equation} 
        f(x+iy) = u(x,y)+i\,v(x,y) \, ,
        \label{eq:f-defined-with-u-and-v}
    \end{equation}
    with $u,v: \mathbb{R} \longrightarrow \mathbb{R} $ real-differentiable functions. 
    
    For $f$ to be complex-differentiable, then:
    \begin{equation}
        f'(z) = \lim_{\Delta x \to 0} \frac{f(z+\Delta x)-f(z)}{\Delta x} = \lim_{i\Delta y \to 0} \frac{f(z+i\,\Delta y)-f(z)}{i\,\Delta y} \, .
        \label{eq:complex-differentiable-condition}
    \end{equation}
    
    Replacing \eqref{eq:complex-differentiable-condition} with \eqref{eq:f-defined-with-u-and-v} and doing some algebra:
     \begin{equation}
        f'(z) = \frac{\partial u}{\partial x} + i \,\frac{\partial v}{ \partial x} = \frac{\partial v}{ \partial y} - i \,\frac{\partial u}{\partial y} \, .
    \end{equation}
    
    By comparing real and imaginary parts from the latest function, the \\ Cauchy-Riemann equations are evident.
\end{proof}

\subsection{Chain Rule}

\begin{theorem}[real multivariate chain rule with complex variable] \label{the:real-multivare-chain-rule-complex-var}
    Given $f: \mathbb{R}^2 \longrightarrow \mathbb{R}, z \in \mathbb{C}$ and $x(z), y(z): \mathbb{C} \longrightarrow \mathbb{R}$ 
    \begin{equation}
        \frac{\partial f}{\partial z} = \frac{\partial f}{\partial x} \frac{\partial x}{\partial z} + \frac{\partial f}{\partial y} \frac{\partial y}{\partial z} \, .
    \end{equation}
    
    This chain rule is analogous to that of the multivariate chain rule with real values.
    The need for stating this case arises from a demonstration that will be done for Theorem \ref{the:complex-chain-rule}. 
\end{theorem}

\begin{theorem}[complex chain rule over real and imaginary part] \label{the:complex-chain-rule-real-im-part}
    Given $f: \mathbb{C} \longrightarrow \mathbb{C}, z \in \mathbb{C}$ and $x(z), y(z): \mathbb{C} \longrightarrow \mathbb{R}$ 
    \begin{equation}
        \frac{\partial f}{\partial z} = \frac{\partial f}{\partial x} \frac{\partial x}{\partial z} + \frac{\partial f}{\partial y} \frac{\partial y}{\partial z} \, .
    \end{equation}
\end{theorem}

\begin{proof}
    By using theorem \ref{the:complex-to-real-mapping}, we can write $f = u + i\,v$ so that:
    \begin{equation}
        \frac{\partial f}{\partial z} = \frac{\partial u}{\partial z} + \im \,\frac{\partial v}{\partial z}\, .
    \end{equation}
    Using theorem \ref{the:real-multivare-chain-rule-complex-var} we can apply the chain rule with $u$ and $v$:
    \begin{equation}
        \frac{\partial f}{\partial z} = \left[ \frac{\partial u}{\partial x} \frac{\partial x}{\partial z} + \frac{\partial u}{\partial y} \frac{\partial y}{\partial z} \right] + \im \left[ \frac{\partial v}{\partial x} \frac{\partial x}{\partial z} + \frac{\partial v}{\partial y} \frac{\partial y}{\partial z} \right]\, .
    \end{equation}
    Rearranging the terms leads to:
    \begin{equation}
    \begin{aligned}
        \frac{\partial f}{\partial z} & = \left[ \frac{\partial u}{\partial x} \frac{\partial x}{\partial z} + \im\, \frac{\partial v}{\partial x} \frac{\partial x}{\partial z} \right] 
        +  \left[ \frac{\partial u}{\partial y} \frac{\partial y}{\partial z} + \im\, \frac{\partial v}{\partial y} \frac{\partial y}{\partial z} \right] \, ,\\
        & = \left[ \frac{\partial u}{\partial x} + \im\, \frac{\partial v}{\partial x} \right]  \frac{\partial x}{\partial z} 
        +  \left[ \frac{\partial u}{\partial y}  + \im\, \frac{\partial v}{\partial y} \right] \frac{\partial y}{\partial z} \, ,\\
        & = \left[ \frac{\partial u + \im \,\partial v}{\partial x} \right]  \frac{\partial x}{\partial z} 
        +  \left[ \frac{\partial u + \im \,\partial v}{\partial y} \right] \frac{\partial y}{\partial z} \, ,\\
        & = \left[ \frac{\partial (u + \im \, v)}{\partial x} \right]  \frac{\partial x}{\partial z} 
        +  \left[ \frac{\partial (u + \im\, v)}{\partial y} \right] \frac{\partial y}{\partial z} \, ,\\
        & = \frac{\partial f}{\partial x} \frac{\partial x}{\partial z} + \frac{\partial f}{\partial y} \frac{\partial y}{\partial z}\, .
    \end{aligned}
    \end{equation}
\end{proof}

\begin{corollary}\label{cor:complex-chain-rule-real-variable}
   See that the proof of Theorem \ref{the:real-multivare-chain-rule-complex-var} is indistinct weather $z \in \mathbb{C}$ or $z \in \mathbb{R}$, so that both \eqref{the:real-multivare-chain-rule-complex-var} and \eqref{the:complex-chain-rule-real-im-part} are also valid for $z \in \mathbb{R}$
\end{corollary}

\begin{definition}[complex chain rule]
Given $h,g: \mathbb{C} \longrightarrow \mathbb{C}, z \in \mathbb{C}$. 
The chain rule in complex numbers is now given by:
\begin{equation}
\begin{aligned}
    \frac{\partial h(g)}{\partial z} & = \frac{\partial h}{\partial g} \frac{\partial g}{\partial z} + \frac{\partial h}{\partial \conjugate{g}} \frac{\partial \conjugate{g}}{\partial z} \\
    \frac{\partial h(g)}{\partial \conjugate{z}} & = \frac{\partial h}{\partial g} \frac{\partial g}{\partial \conjugate{z}} + \frac{\partial h}{\partial \conjugate{g}} \frac{\partial \conjugate{g}}{\partial \conjugate{z}}\, .
    \end{aligned}
    \label{the:complex-chain-rule}
\end{equation}
\end{definition}

\begin{proof}
As we make no assumption of $h$ or $g$ being \textit{holomorphic}, we will use Wirtinger calculus \eqref{eq:wirtinger-calculus} (having in mind that \textit{holomorphic} functions are special cases of the later):
    \begin{equation}
        \frac{\partial \left( f \circ g \right)}{\partial z}
             = \frac{1}{2} \left( \frac{\partial \left( f \circ g \right)}{\partial z_{\Re{}}} - \im\, \frac{\partial \left( f \circ g \right)}{\partial z_{\Im{}}} \right) \, .
    \end{equation}
    Using \eqref{cor:complex-chain-rule-real-variable}:
    \begin{align}
        \frac{\partial \left( f \circ g \right)}{\partial z}
            & = \frac{1}{2} \left(  \left(  \frac{\partial f}{\partial g_{\Re{}}} \frac{\partial g_{\Re{}}}{\partial z_{\Re{}}} 
                +
                \frac{\partial f}{\partial g_{\Im{}}} \frac{\partial g_{\Im{}}}{\partial z_{\Re{}}} \right)
            - \im \,  \left( \frac{\partial f}{\partial g_{\Re{}}} \frac{\partial g_{\Re{}}}{\partial z_{\Im{}}} 
                +
                \frac{\partial f}{\partial g_{\Im{}}} \frac{\partial g_{\Im{}}}{\partial z_{\Im{}}}  \right) \right) \, , \\
            & \begin{aligned}
                & = \frac{1}{4} 
                \left(  \frac{\partial f}{\partial g_{\Re{}}} \frac{\partial (g + \conjugate{g})}{\partial z_{\Re{}}} 
                    - \im \, 
                    \frac{\partial f}{\partial g_{\Im{}}} \frac{\partial (g - \conjugate{g})}{\partial z_{\Re{}}}  \right) \\
                &\qquad \qquad - \im \, \frac{1}{4} 
                \left( \frac{\partial f}{\partial g_{\Re{}}} \frac{\partial (g + \conjugate{g})}{\partial z_{\Im{}}} 
                    - \im \,
                    \frac{\partial f}{\partial g_{\Im{}}} \frac{\partial (g - \conjugate{g})}{\partial z_{\Im{}}}  \right)
                \end{aligned} \\
            & \begin{aligned}
                & = \frac{1}{4}
                \frac{\partial f}{\partial g_{\Re{}}} \left( \frac{\partial (g + \conjugate{g})}{\partial z_{\Re{}}} - \im \,  \frac{\partial (g + \conjugate{g})}{\partial z_{\Im{}}} \right) \\
                & \qquad \qquad - \frac{1}{4}
                \frac{\partial f}{\partial g_{\Im{}}} \left( \im\, \frac{\partial (g - \conjugate{g})}{\partial z_{\Re{}}} +  \frac{\partial (g - \conjugate{g})}{\partial z_{\Im{}}} \right)  \, .
            \end{aligned}
    \end{align}
    Using Wirtinger calculus definition again (Equation \ref{eq:wirtinger-calculus}):
    \begin{equation}
        \frac{\partial \left( f \circ g \right)}{\partial z}
             = \frac{1}{2} \left(
            \frac{\partial f}{\partial g_{\Re{}}} \frac{\partial (g + \conjugate{g})}{\partial z}
            - i
            \frac{\partial f}{\partial g_{\Im{}}} \frac{\partial (g - \conjugate{g})}{\partial z}
            \right) \, .
    \end{equation}
    
    Using \eqref{cor:partial-f-g-real}:
    \begin{equation}
    \begin{aligned}
        \frac{\partial \left( f \circ g \right)}{\partial z}
            & = \frac{1}{2} \left(
            \left( \frac{\partial f}{\partial g} + \frac{\partial f}{\partial \conjugate{g}} \right) \frac{\partial (g + \conjugate{g})}{\partial z}
            +
            \left( \frac{\partial f}{\partial g} - \frac{\partial f}{\partial \conjugate{g}} \right) \frac{\partial (g - \conjugate{g})}{\partial z}
            \right) \, , \\
            & = \frac{1}{2} \left(
            \frac{\partial f}{\partial g} \left( \frac{\partial (g + \conjugate{g})}{\partial z} + \frac{\partial (g - \conjugate{g})}{\partial z} \right) +
            \frac{\partial f}{\partial \conjugate{g}} \left( \frac{\partial (g + \conjugate{g})}{\partial z} - \frac{\partial (g - \conjugate{g})}{\partial z} \right)
            \right) \, , \\
            & = \frac{1}{2} \left(
            \frac{\partial f}{\partial g} \left( \frac{\partial (g + \conjugate{g}) + \partial (g - \conjugate{g})}{\partial z} \right) +
            \frac{\partial f}{\partial \conjugate{g}} \left( \frac{\partial (g + \conjugate{g}) - \partial (g - \conjugate{g})}{\partial z} \right)
            \right) \, ,\\
            & = \frac{1}{2} \left(
            \frac{\partial f}{\partial g} \left( \frac{\partial (g + \conjugate{g} + g - \conjugate{g})}{\partial z} \right) +
            \frac{\partial f}{\partial \conjugate{g}} \left( \frac{\partial (g + \conjugate{g} - g + \conjugate{g})}{\partial z} \right)
            \right) \, ,\\
            & = \left(
            \frac{\partial f}{\partial g}  \frac{\partial g}{\partial z}  +
            \frac{\partial f}{\partial \conjugate{g}} \frac{\partial \conjugate{g}}{\partial z} 
            \right) \, .
    \end{aligned}
    \end{equation}
\end{proof}


\section{Real Valued Backpropagation} \label{sec:real-valued-backprop}

To learn how to optimize the weight of $\omega^{(l)}_{ij}$, it is necessary to find the partial derivative of the loss function for a given weight. We will use the chain rule as follows:
\begin{equation}
    \frac{\partial \loss}{\partial \omega^{(l)}_{ij}} = \sum_{n=1}^{N_L} \frac{\partial e_n}{\partial X_n^{(L)}} \frac{\partial X_n^{(L)}}{\partial V_n^{(L)}} \frac{\partial V_n^{(L)}}{\partial \omega^{(l)}_{ij}} \, .
    \label{eq:partial-loss-function-chain-rule-definition}
\end{equation}

Given these three terms inside the addition, we could find the value we are looking for. The first two partial derivatives are trivial as $\partial e_n(d_n, y_n) / \partial y_n$ exists and is no zero and $\partial X_n^{(L)} / \partial V_n^{(L)}$ is the derivate of $\activation$.

With regard to $ \partial V_n^{(L)} / \partial \omega^{(l)}_{ij} $, when the layer is the same for both values ($l = L$), the definition is trivial and the following result is obtained:
\begin{equation}
\begin{aligned}
    \frac{\partial V_i^{(l)}}{\partial \omega^{(l)}_{ij}} 
       & = \frac{\partial \left( \displaystyle \sum_{\mathrm{j}} \omega^{(l)}_{i\mathrm{j}} X_\mathrm{j}^{(l-1)} \right)}{\partial \omega^{(l)}_{ij}} 
       = \sum_{\mathrm{j}} \frac{\partial \left( \omega^{(l)}_{i\mathrm{j}} X_\mathrm{j}^{(l-1)} \right)}{\partial \omega^{(l)}_{ij}} \, ,\\
       & = \begin{cases}
            0 & \mathrm{j} \neq j \\
            X_j^{(l-1)} & \mathrm{j} = j
       \end{cases} \, ,\\
       & =  X_j^{(l-1)} \, .
\end{aligned}
\label{eq:partial-Vn-w-with-h-l-0}
\end{equation}

Note the subtle difference between $\mathrm{j}$ and $j$. We will now define the cases where the weight and $V_n$ are not from the same layer.

For given $h,l \in [0, L]$, with $h \leq l - 2$ we can define the derivative as follows:
\begin{equation}
\begin{aligned}
    \frac{\partial V_n^{(l)}}{\partial \omega^{(h)}_{jk}} 
    & = \frac{\partial \left( \displaystyle \sum_{i} \omega_{ni}^{(l)} X_i^{(l-1)} \right)}{\partial \omega^{(h)}_{jk}} \, ,\\
    & = \sum_{i}^{N_{l-1}} \omega_{ni}^{(l)} \frac{\partial X_i^{(l-1)}}{\partial \omega^{(h)}_{jk}} \, ,\\
    & = \sum_{i}^{N_{l-1}} \omega_{ni}^{(l)} \frac{\partial X_i^{(l-1)}}{\partial V_i^{(l-1)}} \frac{\partial V_i^{(l-1)}}{\partial \omega^{(h)}_{jk}} \, .
\end{aligned}
\label{eq:partial-Vn-w-with-h-l-2}
\end{equation}

Using Equations \ref{eq:partial-Vn-w-with-h-l-0} and \ref{eq:partial-Vn-w-with-h-l-2}, we have all cases except for $h = l - 1$ which will be:
\begin{equation}
\begin{aligned}
    \frac{\partial V_n^{(l)}}{\partial \omega^{(l-1)}_{jk}}
    & = \frac{\partial \left( \displaystyle \sum_{\mathrm{j}} \omega_{n\mathrm{j}}^{(l)} X_\mathrm{j}^{(l-1)} \right)}{\partial \omega^{(l-1)}_{jk}} \, ,\\
    & = \sum_{\mathrm{j}}^{N_{l-1}} \omega_{n\mathrm{j}}^{(l)} \frac{\partial X_\mathrm{j}^{(l-1)}}{\partial \omega^{(l-1)}_{jk}} \, ,\\
    & = \omega_{nj}^{(l)} \frac{\partial X_j^{(l-1)}}{\partial V_j^{(l-1)}} \frac{\partial V_j^{(l-1)}}{\partial \omega^{(l-1)}_{jk}}\, .
\end{aligned}
\label{eq:partial-Vn-w-with-h-l-1}
\end{equation}

Using the result from \eqref{eq:partial-Vn-w-with-h-l-0} we can get a final result for \eqref{eq:partial-Vn-w-with-h-l-1}. To sum up, the derivative can be written as follows:
\begin{equation}
    \frac{\partial V_n^{(l)}}{\partial \omega^{(h)}_{jk}} =
    \begin{cases}
        X_j^{(l-1)} & h = l \, ,\\
        \omega_{nj}^{(l)} \displaystyle \frac{\partial X_j^{(l-1)}}{\partial V_j^{(l-1)}} \displaystyle \frac{\partial V_j^{(l-1)}}{\partial \omega^{(l-1)}_{jk}} & h = l - 1 \, ,\\
        \displaystyle \sum_{i}^{N_{l-1}} \omega_{ni}^{(l)} \frac{\partial X_i^{(l-1)}}{\partial V_i^{(l-1)}} \displaystyle \frac{\partial V_i^{(l-1)}}{\partial \omega^{(h)}_{jk}} & h \leq l - 2 \, .
    \end{cases}
    \label{eq:partial-Vn-w-generic}
\end{equation}

Equation \ref{eq:partial-loss-function-chain-rule-definition} can then be solved applying \eqref{eq:partial-Vn-w-generic} iteratively to reduce the exponent $L$ to the desired value $l$. Note that $l \leq L$ and $L > 0$.

\subsection{Benvenuto and Piazza definition}

In Reference \cite{BENVENUTO_PIAZZA_BACKPROPAGATION}, another recursive definition for the backpropagation algorithm is defined:
\begin{equation}
    e_n^{(l)} = 
    \begin{cases}
        e_n & l = L \, ,\\
        \displaystyle \sum_{q=1}^{N_{l+1}} \omega_{qn}^{(l+1)} \delta_q^{(l+1)} & l < L \, ,
    \end{cases}
\end{equation}
with $\delta_n^{(l)} = e_n^{(l)}\activation'(V_n^{(l)})$. Then the derivation is defined recursively as:
\begin{equation}
    \frac{\partial \loss}{\partial \omega^{(l)}_{ij}} = \sum_n^{N_L} \delta_n^{(l)} X_{m}^{(l-1)} \, .
    \label{eq:benvenuto-recursive-def-delta-w}
\end{equation}

It can be proven that \eqref{eq:partial-loss-function-chain-rule-definition} and \eqref{eq:benvenuto-recursive-def-delta-w} are equivalent.

\section{Automatic Differentiation} \label{chap:autodif}

This topic is also very well covered in Appendix D of \cite{geron2019hands}. The appendix also covers manual differentiation, symbolic differentiation, and numerical differentiation. In this Section, we will jump directly to forward-mode \acrfull{autodiff} and reverse-mode \acrshort{autodiff}.

There many approaches to explain \acrfull{autodiff} \cite{hoffmann2016hitchhiker}. Most cases tend to assume the derivative at a given point exists which is a logical assumption having in mind the algorithm is computing the derivative itself. However, we have seen that for our case we can soften this definition and only ask for the \textit{Wirtinger calculus} to exist. Furthermore, the requirement that the derivative at a point exists is a very strong condition that is to be avoided to compute complex number backpropagation. Even in the real domain, there are examples like \acrfull{relu} that are widely used for deep neural networks that have no derivative at $x = 0$, and yet the backpropagation is applied without a problem.

\begin{itemize}
    \item \cite{rall1986arithmetic} explains \acrshort{autodiff} it in a very clear and concise manner by defining the \textit{dual number} (to be explained later in this Chapter) arithmetic directly but assumes that the derivative in the point exists.
    \item \cite{hall2003calculating} is one of the few that actually talks about complex number \acrshort{autodiff} but assumes that Cauchy-Riemann equations are valid.
    \item \cite{pearlmutter2007lazy} presents \textit{Taylor series} as the main idea behind forward-mode \acrshort{autodiff}. However, \textit{Taylor series} assumes that the function is infinitely derivable at the desired point.
    \item \cite{rall1983differentiation} assumes the function is differentiable.
\end{itemize}

\subsection{Forward-mode automatic differentiation} \label{sec:forward-mode-autodiff}

In this Section, we will demonstrate the theory behind forward-mode \acrshort{autodiff} in a general way, and we will not ask for $f$ to be infinitely derivable at a point furthermore, we will not even require it to have a first derivative. We will later extend this definition to the complex domain. To the author's knowledge, this demonstration is not present in any other Reference, although the high quantity of bibliography on this subject may suggest otherwise. After defining forward-mode \acrshort{autodiff} we will proceed to explain reverse-mode \acrshort{autodiff}.

The definition of the derivative of function $f$ is given by:
\begin{equation}
    \label{eq:derivate-definition}
    f'(x) = \lim_{h \to 0} \frac{f(x+h) - f(x)}{h} \, .
\end{equation}

We will generalize this equation for only a one-sided limit. The choice of right- or left-sided limit will be indifferent to the demonstration:
\begin{equation}
    \label{eq:soften-derivate-definition}
    Df(x) = \lim_{h \to 0^\pm} \frac{f(x+h) - f(x)}{h} \, ,
\end{equation}
where $Df$ stands for this "soften" derivative definition and $\pm$ stands for either left ($-$) or right ($+$) sided limit.

In Equation \ref{eq:soften-derivate-definition}, we have "soften" the condition needed for the derivative. In cases where the derivative of $f$ exists, we will not need to worry because \eqref{eq:soften-derivate-definition} will converge to \eqref{eq:derivate-definition} hence being equivalent. However, in cases where the derivative does not exist, because the left side limit does not converge to the left side limit (like \acrshort{relu} at $x=0$), this definition will render a result.

From the above definition, we have:
\begin{equation}
\begin{aligned}
    Df(x) \, \lim_{h \to 0^\pm} h &= \lim_{h \to 0^\pm} f(x+h) - f(x) \\
    \lim_{h \to 0^\pm} f(x+h) &= f(x) + Df(x) \, \lim_{h \to 0^\pm} h \, .
\end{aligned}
\end{equation}

We can now define $\displaystyle\lim_{h \to 0} = \epsilon$. In this case, $\epsilon$ will be an infinitesimally small number. The sign of $\epsilon$ will depend on the side of the limit chosen, for example, if $\epsilon = \displaystyle\lim_{h^- \to 0}$ then $\epsilon$ can be $-0.00..001$ whereas if $\epsilon = \displaystyle\lim_{h^+ \to 0}$, $\epsilon$ can be $0.00..02$. Note that the last digit of $\epsilon$ can be anything; it can even be more than one digit as long as it is preceded by an "infinite" number of zeros.

With this definition, the above equation can now be written as:
\begin{equation}
    \label{eq:forward-equation-decimal}
    f(x+\epsilon) = f(x) + Df(x)\, \epsilon \, .
\end{equation}

Given a number $x = a + b\epsilon$, forward-mode automatic differentiation writes this number a tuple of two numbers in the form of $x = (a,b)$ called \textit{dual numbers}. \textit{Dual numbers} are represented in memory as a pair of floats. An arithmetic for this newly defined \textit{dual numbers} are described in detail in \cite{rall1986arithmetic}. We can think of dual numbers as a transformation $T_\epsilon[x]: \mathbb{R} \to \mathbb{R}^2 / T_\epsilon[a+b\,\epsilon] = (a, b)$ \cite{kedem1980automatic}. The real number system maps isomorphically into this new space by the mapping $x \mapsto (x,0), x \in \mathbb{R}$. To use the same notation as \cite{rall1986arithmetic}, we will call the \textit{dual number} space $\mathbb{D}$, although the transform definition explained here is slightly different and so, space $\mathbb{D}$ is not equivalent as space $\mathbb{D}$ from \cite{rall1986arithmetic}.
Operations in this space can be easily defined, for example:
\begin{align}
    \lambda(a, b) &= (\lambda \,a, \lambda \,b) \, ,\\
    (a, b) + (c, d) &= (a + b, c + d)      \, , \\
    (a, b) \, (c, d) &= (a\,c, (a\,d + b\,c) + b\,d\,\epsilon) = (a\,c, a\,d + b\,c)\, .
    \label{eq:mult-with-dual-numbers}
\end{align}

See that in Equation \ref{eq:mult-with-dual-numbers}, we have approximated $(a\,c, (a\,d + b\,c) + b\,d\,\epsilon) = (a\,c, a\,d + b\,c)$. This is evident because the second term of the dual number is stored in memory as a float. Being $b\,d$ a product of two bounded scalar values and being $\epsilon$, by definition, an infinitesimal number, the value yielded by the term $b\,d\,\epsilon$ will tend to zero and fall under the machine epsilon (machine precision). The existence of the additive inverse (or negative) and the identity element for addition can be easily defined. So is the case for multiplication. The multiplication and addition are commutative and associative. Basically, the space $\mathbb{D}$ is well defined \cite{rall1986arithmetic}.

The choice of $\epsilon$ will only affect how the transformation $T_\epsilon$ is applied but will not affect this demonstration nor the basic operations of the space $\mathbb{D}$. 

If we rewrite Equation \ref{eq:forward-equation-decimal} in \textit{dual number} notation we have:
\begin{equation}
    \label{eq:forward-equation-dual-number}
    f\left((a,b)\right) = \left(f(a), b \, Df(a)[\epsilon]\right) \, .
\end{equation}
Equation \ref{eq:forward-equation-dual-number} means that if we find a way to compute the function we want to derivative using dual number notation, the result will yield a dual number whose first number is the result of the function at point $a$ and the second number will be the derivative (provided $b = 1$) at that same point. Note that the above equation has widened the "soft" derivative definition to $Df(x)[\epsilon]$ which means the derivative of $f$ at point $x$ on the $\epsilon$ direction, in this case, either left or right-sided limit.

The strength of forward-mode \acrlong{autodiff} is that this result is exact \cite{hall2003calculating, hoffmann2016hitchhiker}. This is, the only inaccuracies which occur are those which appear due to rounding errors in floating-point arithmetic or due to imprecise evaluations of elementary functions. If we had infinite float number precision and we could define the exact value of $f$ on the \textit{dual number} base, we will have the exact value of the derivative.

Another virtue of forward-mode \acrshort{autodiff} is that $f$ can be any coded function whose symbolic equation is unknown. Forward-mode \acrshort{autodiff} can then compute any number of nested functions as long as the basic operations are well defined in \textit{dual number} form.
Then if $f$ does many calls to these basic operations inside loops or any conditional call (making it difficult to derive the symbolic equation), it will still be possible for forward-mode \acrshort{autodiff} to compute its derivative without any difficulty.

\subsubsection{Examples}

With the goal of helping the reader to better understand forward-mode \acrshort{autodiff}, we will see the example used in \cite{geron2019hands}, we define $f(x, y) = x^2y+y+2$. For it we compute $f(x+\epsilon,y) = 42 + 24\,\epsilon$ following the logic of Figure \ref{fig:forward-mode-autodiff}. Therefore, we conclude that $\displaystyle\frac{\partial f}{\partial x}(3,4) = 24$.

\begin{figure}[ht]
\centering
    \includegraphics[width=0.7\linewidth]{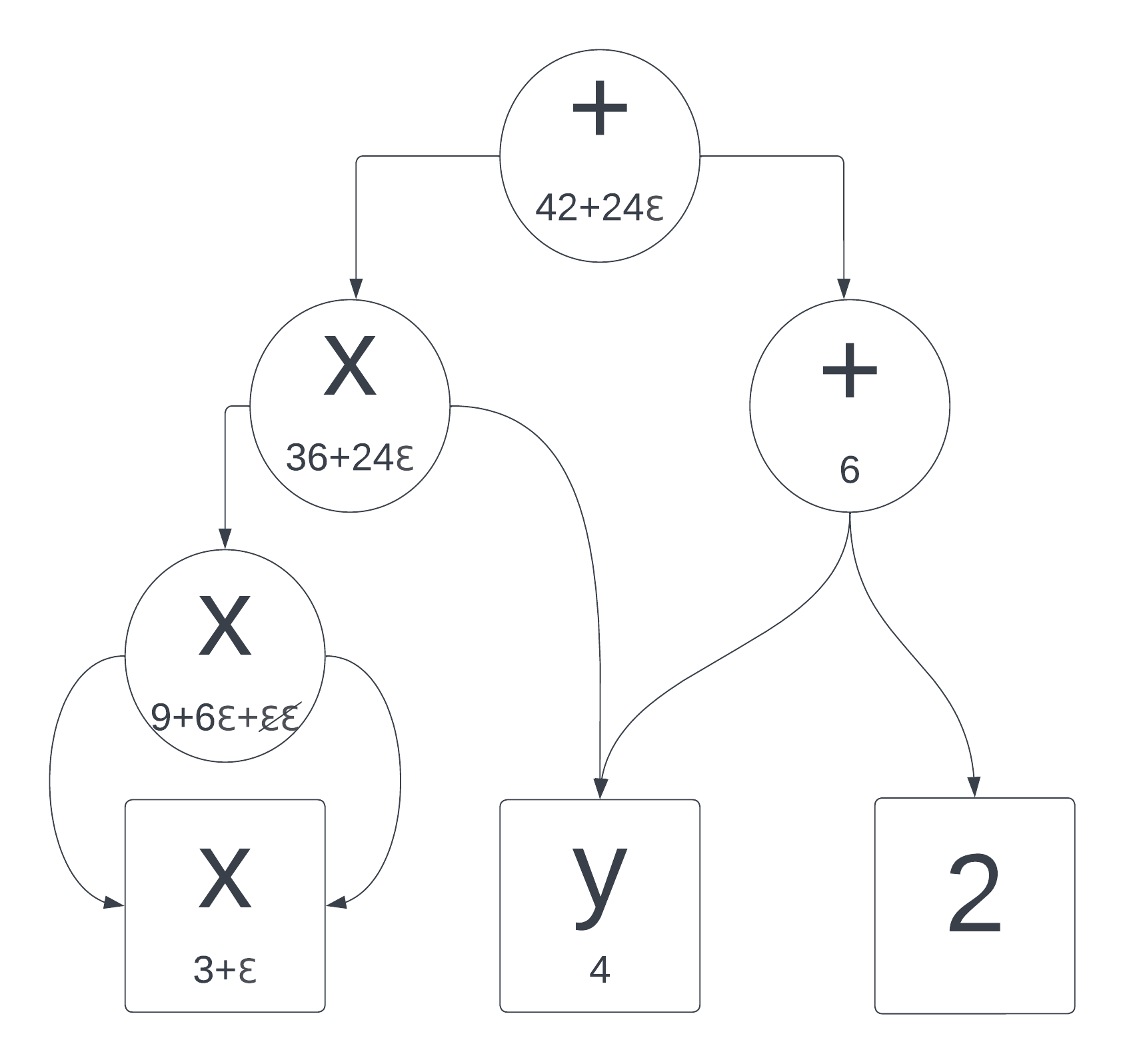}
    \caption{Forward-mode autodiff block diagram example}
    \label{fig:forward-mode-autodiff}
\end{figure}

\textbf{\acrfull{relu}}

In another example, a very well-known function called \acrfull{relu}, is one of the most used activation functions in machine learning models. We can write \acrshort{relu} as $f(x) = \max(0, x)$ for what it's derivative will be defined as:
\begin{equation}
    f'(x) = 
    \begin{cases}
        1 & x > 0 \\
        0 & x < 0 \, .
    \end{cases} 
\end{equation}

This is the example where:
\begin{equation}
    \label{eq:continuity-of-derivate}
    \lim_{h \to 0^+} \frac{f(x+h) - f(x)}{h} = 1 \neq \lim_{h \to 0^-} \frac{f(x+h) - f(x)}{h} = 0 \, .
\end{equation}

Therefore, its derivative is not defined at $x=0$. The theory of why this case doesn't pose any problem for reaching an optimal point when doing backpropagation is outside the scope of this text. However, we will see what the forward-mode \acrlong{autodiff} gives as a result. For it, we should first define, as usual, $f$ in the \textit{dual number} space:
\begin{equation}
    f\left((a,b)\right) = 
    \begin{cases}
        (a, b) & a > 0 \\
        (0, b) & a = 0 \\
        (0, 0) & a < 0 \, .
    \end{cases} 
\end{equation}

This definition is logical if we have chosen $\displaystyle\lim_{h^+ \to 0} = \epsilon$ and must be changed if the left-sided limit is chosen. We therefore compute the value of $f(0,1)$ to see the value of the derivative at point $x = 0$:
\begin{equation}
\begin{aligned}
    \label{eq:forward-mode-autodiff-relu}
    f(x + \epsilon) &= \max(0, 0.00...01) = 0.00...01 = (0, 1) = (f(0), f'(0))\, , \\
    f\left((0,1)\right) &= (0, 1)\, .
\end{aligned}
\end{equation}

We can see in Equation \ref{eq:forward-mode-autodiff-relu} that the forward-mode \acrshort{autodiff} algorithm will yield the result of $Df(0)[\epsilon] = 1$ which is an acceptable result for this case.

\subsubsection{Complex forward-mode automatic differentiation}

With the generalization to the complex case, $\epsilon$ now becomes complex, in the real case, we arbitrarily chose either the left or right-sided limit. Now, limitless directions are possible (phase), affecting, of course, the result of the derivative. Table \ref{table:directional-derivatives} shows the derivative that is calculated when different values of $\epsilon$ are chosen.

\begin{table}[ht]
\centering
\begin{tabular}{||c | c||} 
 \hline
 Definition & One possibility for $\epsilon$  \\ [0.5ex] 
 \hline\hline
 $\displaystyle\lim_{x \to 0} \frac{f(z+x)-f(z)}{x}$ & $0.00...01$  \\ 
 \hline
 $\displaystyle\lim_{y \to 0} \frac{f(z+\im \,y)-f(z)}{y}$ & $0.00...01 \,\im$  \\
 \hline
 $\displaystyle\lim_{x \to 0, y \to 0} \frac{f(z+x-\im \,y)-f(z)}{x+y}$ & $ 0.00...01\, (1-\im)$  \\
 \hline
 $\displaystyle\lim_{x \to 0, y \to 0} \frac{f(z+x+\im \,y)-f(z)}{x+y}$ & $0.00...01 \, (1+\im)$  \\ [1ex] 
 \hline
\end{tabular}
\caption{Directional derivatives with respect to $\epsilon$.}
\label{table:directional-derivatives}
\end{table}

The Table can be generalized to the following equation:
\begin{equation}
    \nabla_{\epsilon} f = \lim_{h \to 0} \frac{f(z+h\,\epsilon)-f(z)}{h\,\epsilon} \, .
\end{equation}

It is important to note that the module of $\epsilon$ is unimportant as long as it is small enough so that the analysis made in the last Section (Section \ref{sec:forward-mode-autodiff}) stands.

For functions where the complex derivative exists, all possibilities will converge to the same value.

The first and second terms on Table \ref{table:directional-derivatives} are used in the computation of \textit{Wirtinger Calculus}, and it is not necessarily equal to the third term of the Table.


\subsection{Reverse-mode automatic differentiation} \label{sec:reverse-autodiff}

Forward-mode \acrlong{autodiff} has many useful properties. First of all, as we have seen, the condition for the derivative to exist is not necessarily required for the method to yield a result that can be helpful when dealing with, in our case, non-holomorphic functions. Also, it allows finding a derivative value for any coded function, even containing loops or conditionals, as long as the primitives are defined. It is also natural for the algorithm to deal with the chain rule.

However, by taking a look at the example of Figure \ref{fig:forward-mode-autodiff}, if we now want to compute $\displaystyle\frac{\partial f}{\partial y}(3,4)$ we will need to compute $f(x,y+\epsilon)$, meaning that in order to know both $\displaystyle\frac{\partial f}{\partial x}(3,4)$ and  $\displaystyle\frac{\partial f}{\partial y}(3,4)$ we will need to compute the algorithm twice. This can be exponentially costly for neural networks where there are sometimes even millions of trainable parameters from which we need to compute the partial derivative. Here is where reverse-mode automatic differentiation comes to the rescue  enabling us to compute all partial derivatives at once.

\subsubsection{Examples}

In here we will show how reverse-mode \acrshort{autodiff} can compute the previous both $\displaystyle\frac{\partial f}{\partial x}(3,4)$ and  $\displaystyle\frac{\partial f}{\partial y}(3,4)$ of the previous example for forward-mode \acrshort{autodiff} running the code once. 

The first step is to compute $f(3, 4)$, whose intermediate values are shown on the bottom right of each node. Each node was labeled as $n_i$ for clarity with $i \in [1, 7]$. The output of $f(3, 4) = n_7 = 42$ as expected.

\begin{figure}[ht]
\centering
    \includegraphics[width=0.7\linewidth]{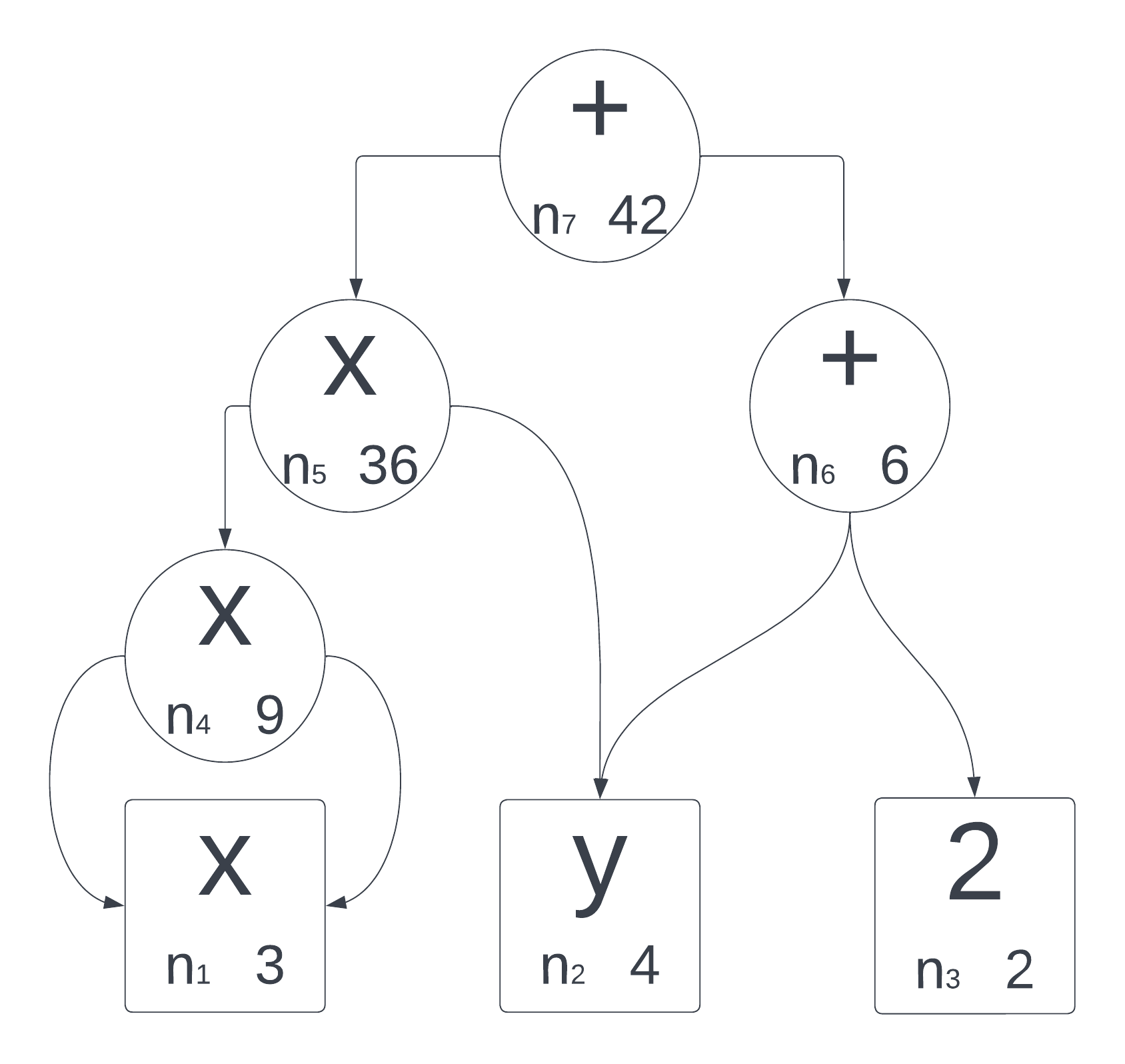}
    \caption{Reverse-mode autodiff block diagram example}
    \label{fig:reverse-mode-autodiff}
\end{figure}

Now all the partial derivatives $\displaystyle\frac{\partial f}{\partial n_i}$ are computed starting with $n_7$. Since $n_7$ is the output node, $\displaystyle\frac{\partial f}{\partial n_7} = 1$. The chain rule is then used to compute the rest of the nodes by going down the graph. For example, to compute $\displaystyle\frac{\partial f}{\partial n_5}$ we use
\begin{equation}
    \frac{\partial f}{\partial n_5} = \frac{\partial f}{\partial n_7} \frac{\partial n_7}{\partial n_5} \, ,
\end{equation}

and as we previously calculated $\displaystyle\frac{\partial f}{\partial n_7}$, we just need to compute the second term. This methodology will be repeated until all nodes' partial derivatives are computed. Here is the full list of the partial derivatives:
\begin{itemize}
    \item $\displaystyle\frac{\partial f}{\partial n_7} = 1$,
    \item $\displaystyle\frac{\partial f}{\partial n_6} = \frac{\partial f}{\partial n_7} \frac{\partial n_7}{\partial n_6} = 1$,
    \item $\displaystyle\frac{\partial f}{\partial n_5} = \frac{\partial f}{\partial n_7} \frac{\partial n_7}{\partial n_5} = 1$,
    \item $\displaystyle\frac{\partial f}{\partial n_4} = \frac{\partial f}{\partial n_5} \frac{\partial n_5}{\partial n_4} = n_2 = 4$,
    \item $\displaystyle\frac{\partial f}{\partial y} = \frac{\partial f}{\partial n_2} = 
    \frac{\partial f}{\partial n_5} \frac{\partial n_5}{\partial n_2} + 
    \frac{\partial f}{\partial n_6} \frac{\partial n_6}{\partial n_2} = n_4 + 1 = 10$
    \item $\displaystyle\frac{\partial f}{\partial x} = \frac{\partial f}{\partial n_1} = 
    \frac{\partial f}{\partial n_5} \frac{\partial n_5}{\partial n_1} + 
    \frac{\partial f}{\partial n_5} \frac{\partial n_5}{\partial n_1} = n_1 \cdot n_4 + n_1 \cdot n_4 = 3 \cdot 4 + 3 \cdot 4 = 24$.
\end{itemize}

This code has the advantage that all partial derivatives are computed at once which allows computing the values for a very high number of trainable parameters with less computational cost.

\subsubsection{Complex reverse-mode automatic differentiation}

In the following complex example, we will compute the reverse automatic differentiation on a complex multiplication operation $f = |(a+\im\,b)(c+\im\,d)| = a\,c - b\,d + a\,d + b\,c$, we know by definition that the derivative, using \textit{Wirtinger Calculus}, is:
\begin{align}
    \label{eq:partials-wirtinger-def-reverse}
    \frac{\partial f}{\partial (a + \im\,b)} &= \frac{\partial f}{\partial a} + \im\, \frac{\partial f}{\partial b} = (c + d) + \im\, (c - d) \, ,\\
    \frac{\partial f}{\partial (c + \im\,d)} &= \frac{\partial f}{\partial c} + \im\, \frac{\partial f}{\partial d} = (a + b) + \im\, (a - b) \, .
\end{align}

The following Figure \ref{fig:complex-reverse-mode-autodiff} shows the block diagram of $f$.

\begin{figure}[ht]
\centering
    \includegraphics[width=0.7\linewidth]{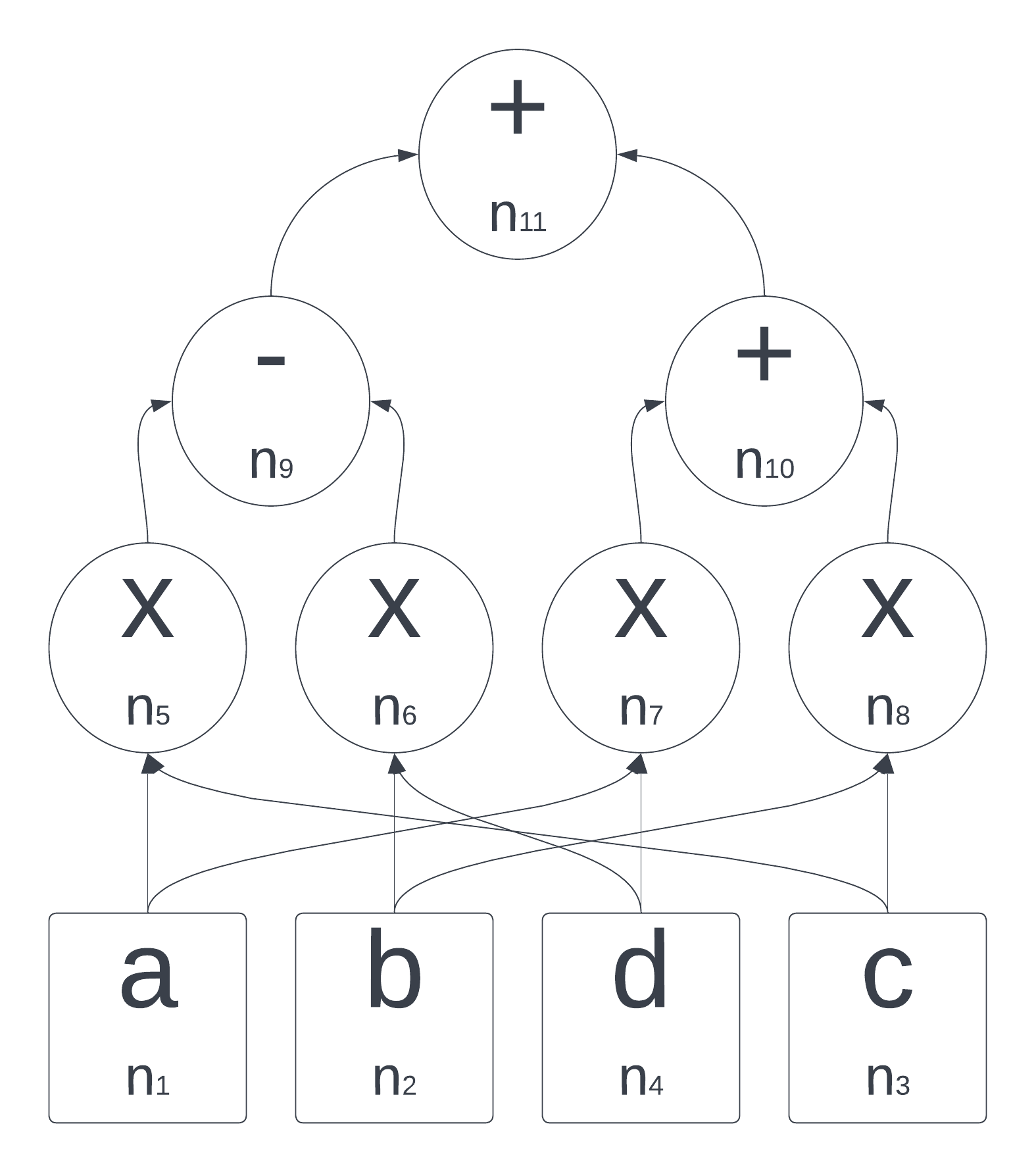}
    \caption{Complex reverse-mode autodiff block diagram example}
    \label{fig:complex-reverse-mode-autodiff}
\end{figure}

Using the block diagram of Figure \ref{fig:complex-reverse-mode-autodiff} we can compute the nodes as follows:
\begin{itemize}
    \item $n_1 = a$,
    \item $n_2 = b$,
    \item $n_3 = c$,
    \item $n_4 = d$,
    \item $n_5 = n_1 \cdot n_3 = a \cdot c$,
    \item $n_6 = n_2 \cdot n_4 = b \cdot d$,
    \item $n_7 = n_1 \cdot n_4 = a \cdot d$,
    \item $n_8 = n_2 \cdot n_3 = b \cdot c$,
    \item $n_9 = n_5 - n_6 = a\,c - b\,d$,
    \item $n_{10} = n_7 + n_8 = a\,d + b\,c$,
    \item $f = n_{11} = n_{10} + n_9 = a\,c - b\,d + a\,d + b\,c$.
\end{itemize}

We now perform the reverse-mode backpropagation starting from the last node ($n_{11}$) and go back using the previously computed values to extract the partial derivative of $f = n_{11}$ with respect to every node.
\begin{itemize}
    \item $\displaystyle\frac{\partial f}{\partial n_{11}} = \frac{\partial n_{11}}{\partial n_{11}} = 1$,
    \item $\displaystyle\frac{\partial n_{11}}{\partial n_{10}} = 1$,
    \item $\displaystyle\frac{\partial n_{11}}{\partial n_{9}} = 1$,
    \item $\displaystyle\frac{\partial n_{11}}{\partial n_{8}} = \frac{\partial n_{11}}{\partial n_{10}} \frac{\partial n_{10}}{\partial n_{8}} = 1$,
    \item $\displaystyle\frac{\partial n_{11}}{\partial n_{7}} = \frac{\partial n_{11}}{\partial n_{10}} \frac{\partial n_{10}}{\partial n_{7}} = 1$,
    \item $\displaystyle\frac{\partial n_{11}}{\partial n_{6}} = \frac{\partial n_{11}}{\partial n_{9}} \frac{\partial n_{9}}{\partial n_{6}} = -1$,
    \item $\displaystyle\frac{\partial n_{11}}{\partial n_{5}} = \frac{\partial n_{11}}{\partial n_{9}} \frac{\partial n_{9}}{\partial n_{5}} = 1$,
    \item $\displaystyle\frac{\partial f}{\partial d} = \frac{\partial n_{11}}{\partial n_{4}} = 
    \frac{\partial n_{11}}{\partial n_{7}}\frac{\partial n_{7}}{\partial n_{4}} + 
    \frac{\partial n_{11}}{\partial n_{6}}\frac{\partial n_{6}}{\partial n_{4}} = 
    n_1 - n_2 = a - b$,
    \item $\displaystyle\frac{\partial f}{\partial c} = \frac{\partial n_{11}}{\partial n_{3}} =
    \frac{\partial n_{11}}{\partial n_{8}}\frac{\partial n_{8}}{\partial n_{3}} + 
    \frac{\partial n_{11}}{\partial n_{5}}\frac{\partial n_{5}}{\partial n_{3}} =
    n_2 + n_1 = b + a$,
    \item $\displaystyle\frac{\partial f}{\partial b} = \frac{\partial n_{11}}{\partial n_{2}} = 
    \frac{\partial n_{11}}{\partial n_{6}}\frac{\partial n_{6}}{\partial n_{2}} +
    \frac{\partial n_{11}}{\partial n_{8}}\frac{\partial n_{8}}{\partial n_{2}} =
    n_3 - n_4 = c - d$,
    \item $\displaystyle\frac{\partial f}{\partial a} = \frac{\partial n_{11}}{\partial n_{1}} = 
    \frac{\partial n_{11}}{\partial n_{5}}\frac{\partial n_{5}}{\partial n_{1}} +
    \frac{\partial n_{7}}{\partial n_{1}}\frac{\partial n_{7}}{\partial n_{1}} =
    n_3 + n_4 = c + d$.
\end{itemize}

Now, if we consider $\partial f / \partial (a + \im\,b) = \partial f / \partial a + \im\, \partial f / \partial b$ and replace the values we obtained in the previous list, we get the same result as in Equation \ref{eq:partials-wirtinger-def-reverse}.

\bibliographystyle{unsrt}  
\bibliography{references}

\end{document}